\documentclass{article}

\usepackage{mathtools,latexsym}
\usepackage{amsthm,thmtools,thm-restate}

\PassOptionsToPackage{sort&compress}{natbib}
\usepackage{iclr2021_conference}

\usepackage[utf8]{inputenc} 
\usepackage[T1]{fontenc}    

\usepackage{amssymb,amsfonts,bbm}
\usepackage{times}
\usepackage{microtype}

\PassOptionsToPackage{hyphens}{url}
\usepackage{hyperref}       
\usepackage{url}            
\usepackage{booktabs}       
\usepackage{tablefootnote}
\usepackage{nicefrac}       

\usepackage[capitalize]{cleveref}
\usepackage{csquotes}

\usepackage{enumitem}

\usepackage[export]{adjustbox}
\usepackage{caption}
\usepackage{subcaption}

\declaretheoremstyle[
  notefont = \bfseries,
  bodyfont=\normalfont\itshape,
  spaceabove = \parskip,
  spacebelow = 0pt
  ]{theoremstyle}
\declaretheorem[style=theoremstyle]{theorem}
\declaretheorem[style=theoremstyle]{lemma}

\declaretheorem[style=theoremstyle]{corollary}

\declaretheoremstyle[
  postheadspace = \newline,
  notefont = \bfseries,
  headpunct = {},
  bodyfont = \normalfont,
  spaceabove = \parskip,
  spacebelow = 0pt
  ]{examplestyle}

\declaretheoremstyle[
  notefont = \bfseries,
  bodyfont = \normalfont,
  spaceabove = \parskip,
  spacebelow = 0pt
  ]{definitionstyle}
\declaretheorem[style=definitionstyle]{definition}
\declaretheorem[
  style=definitionstyle,
  shaded={rulecolor=black, rulewidth=1pt, bgcolor=white},
  name=Definition,
  refname={Definition,Definitions},
  Refname={Definition,Definitions},
  sibling=definition
  ]{definitionboxed}

\declaretheoremstyle[
  headfont = \normalfont\itshape,
  notefont = \normalfont\itshape,
  bodyfont = \normalfont,
  spaceabove = \parskip,
  spacebelow = 0pt
  ]{remarkstyle}
\declaretheorem[style=remarkstyle]{remark}

\DeclareMathOperator{\Prob}{\mathbb{P}}
\DeclareMathOperator{\Exp}{\mathbb{E}}
\DeclareMathOperator{\Var}{\mathbb{V}}

\DeclareMathOperator{\Ber}{Ber}
\DeclareMathOperator*{\esssup}{ess\,sup}
\DeclareMathOperator{\argmax}{arg\,max}

\title{Calibration tests beyond classification}

\author{%
  David Widmann\\
  Department of Information Technology\\
  Uppsala University, Sweden\\
  \texttt{david.widmann@it.uu.se} \\
  \And
  Fredrik Lindsten\\
  Division of Statistics and Machine Learning\\
  Linköping University, Sweden\\
  \texttt{fredrik.lindsten@liu.se} \\
  \And
  Dave Zachariah\\
  Department of Information Technology\\
  Uppsala University, Sweden\\
  \texttt{dave.zachariah@it.uu.se}\\
}

\iclrfinalcopy 

\begin{document}

\maketitle

\begin{abstract}
Most supervised machine learning tasks are subject to irreducible prediction
errors. Probabilistic predictive models address this limitation by providing
probability distributions that represent a belief over plausible targets,
rather than point estimates. Such models can be a valuable tool in
decision-making under uncertainty, provided that the model output is
meaningful and interpretable. Calibrated models guarantee that the probabilistic
predictions are neither over- nor under-confident. In the machine learning literature,
different measures and statistical tests have been proposed and studied
for evaluating the calibration of classification models. For
regression problems, however, research has been focused on a weaker
condition of calibration based on predicted quantiles for real-valued targets.
In this paper, we propose the first framework that unifies calibration evaluation and
tests for general probabilistic predictive models. It applies to any such model, including
classification and regression models of arbitrary dimension. Furthermore,
the framework generalizes existing measures and provides a more intuitive
reformulation of a recently proposed framework for calibration in
multi-class classification. In particular, we reformulate and generalize the
kernel calibration error, its estimators, and hypothesis tests using scalar-valued
kernels, and evaluate the calibration of real-valued regression
problems.\footnote{The source code of the experiments is available at
\url{https://github.com/devmotion/Calibration_ICLR2021}.}
\end{abstract}

\section{Introduction}

We consider the general problem of modelling the relationship
between a feature $X$ and a target $Y$ in a probabilistic setting, i.e., we
focus on models that approximate the conditional probability
distribution $\Prob(Y | X)$ of target $Y$ for given feature $X$. The use of
probabilistic models that output a probability distribution instead
of a point estimate demands guarantees on the predictions beyond accuracy,
enabling meaningful and interpretable predicted uncertainties. One such statistical
guarantee is calibration, which has been studied extensively in metereological
and statistical literature~\citep{DeGroot1983,Murphy1977}.

A calibrated model ensures that almost every prediction matches the
conditional distribution of targets given this prediction.
Loosely speaking, in a classification setting a predicted
distribution of the model is called calibrated (or reliable), if
the empirically observed frequencies of the different classes
match the predictions in the long run, if the same
class probabilities would be predicted repeatedly.
A classical example is a weather forecaster who predicts each day if it is going to
rain on the next day. If she predicts rain with probability 60\% for
a long series of days, her forecasting model is calibrated \emph{for predictions
of 60\%} if it actually rains on 60\% of these days.

If this property holds for almost every probability distribution that
the model outputs, then the model is considered to be calibrated.
Calibration is an appealing property of a probabilistic model since it provides
safety guarantees on the predicted distributions even in the common case
when the model does not predict the true distributions $\Prob(Y | X)$.
Calibration, however, does not guarantee accuracy (or refinement)---a model
that always predicts the marginal probabilities of each class
is calibrated but probably inaccurate and of limited use. On the other
hand, accuracy does not imply calibration either since the predictions
of an accurate model can be too over-confident and hence miscalibrated,
as observed, e.g., for deep neural networks~\citep{Guo2017}.

In the field of machine learning, calibration has been
studied mainly for classification problems~\citep{Guo2017,Widmann2019,Vaicenavicius2019,Platt2000,Zadrozny2002,Broecker2009,Kull2017,Kumar2018,Kull2019}
and for quantiles and confidence intervals of models for regression problems
with real-valued targets \citep{Ho2005,Fasiolo2020,Rueda2006,Taillardat2016,Kuleshov2018}.
In our work, however, we do not restrict ourselves to these problem
settings but instead consider calibration for arbitrary predictive
models.
Thus, we generalize the common notion of calibration as:

\begin{definition}\label{def:calibration}
    Consider a model $P_X \coloneqq P(Y|X)$ of a conditional probability distribution
    $\Prob(Y|X)$. Then model $P$ is said to be calibrated
    if and only if
    \begin{equation}\label{eq:calibration_definition}
        \Prob(Y|P_X) = P_X \qquad \text{almost surely}.
    \end{equation}
\end{definition}

If $P$ is a classification model, \cref{def:calibration} coincides
with the notion of (multi-class) calibration by
\citet{Broecker2009,Vaicenavicius2019,Kull2019}. Alternatively, in classification
some authors~\citep{Naeini2015,Guo2017,Kumar2018} study the
strictly weaker property of confidence calibration~\citep{Kull2019},
which only requires
\begin{equation}\label{eq:calibration_confidence}
    \Prob{(Y = \argmax{P_X} | \max P_X)} = \max P_X \qquad \text{almost surely}.
\end{equation}
This notion of calibration corresponds to calibration according to \cref{def:calibration}
for a reduced problem with binary targets
$\widetilde{Y} \coloneqq \mathbbm{1}(Y = \argmax{P_X})$ and Bernoulli
distributions $\widetilde{P}_X \coloneqq \Ber(\max P_X)$ as probabilistic models.

For real-valued targets, \cref{def:calibration} coincides with the
so-called distribution-level calibration by \citet{Song2019}.
Distribution-level calibration
implies that the predicted quantiles are calibrated, i.e., the outcomes for all
real-valued predictions of the, e.g., 75\% quantile are actually below the
predicted quantile with 75\% probability~\citep[Theorem~1]{Song2019}. Conversely,
although quantile-based calibration is a
common approach for real-valued regression
problems~\citep{Ho2005,Fasiolo2020,Rueda2006,Taillardat2016,Kuleshov2018},
it provides weaker guarantees on the predictions. For instance, the
linear regression model in \cref{fig:ols_motivation} empirically shows quantiles that
appear close to being calibrated albeit being uncalibrated according to \cref{def:calibration}.

\begin{figure}[!htb]
    \centering
    \begin{subfigure}[t]{0.62\textwidth}
        \centering
        \textbf{a}
        \includegraphics[valign=t,scale=0.23]{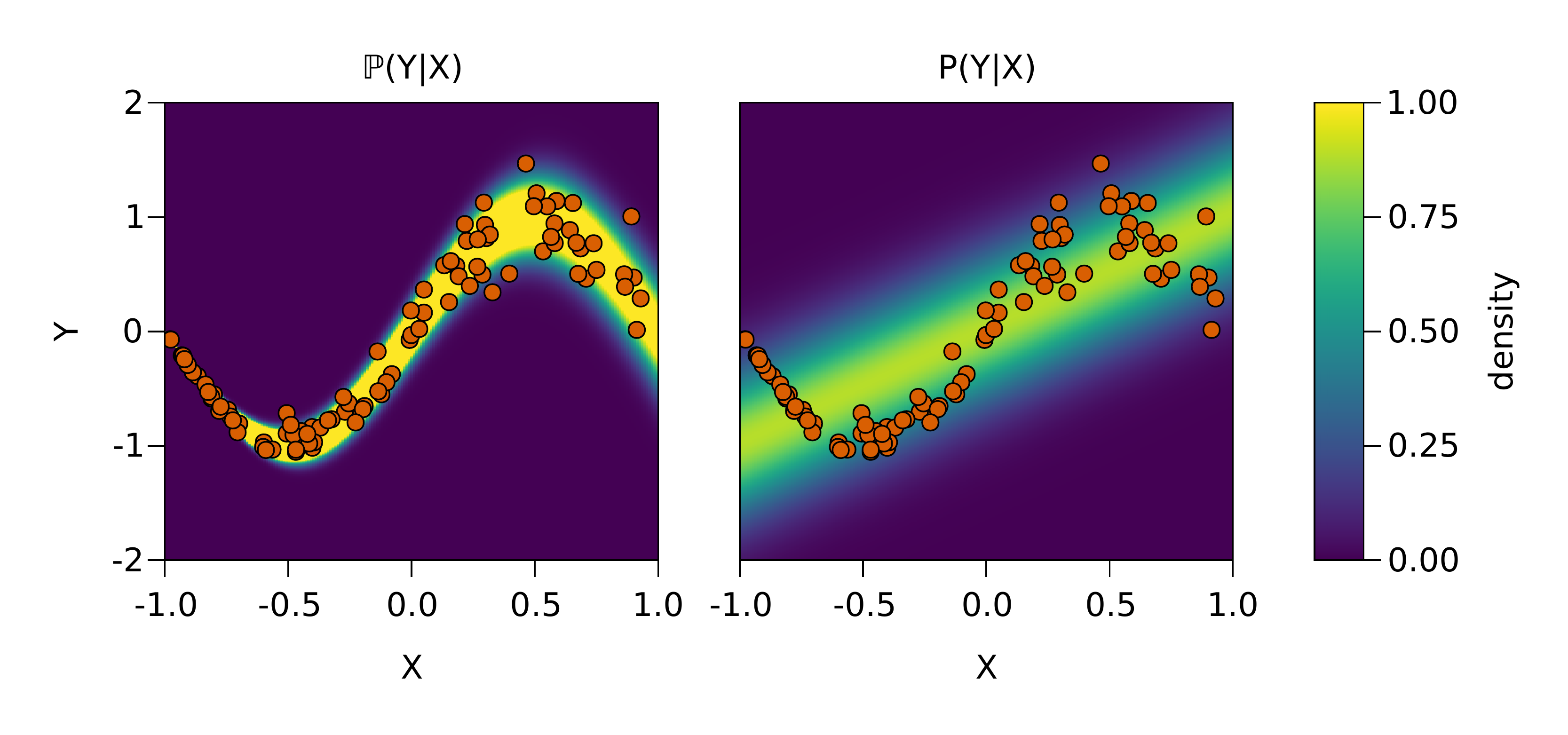}
    \end{subfigure}%
    \hfill%
    \begin{subfigure}[t]{0.38\textwidth}
        \centering
        \textbf{b}
        \includegraphics[valign=t,scale=0.23]{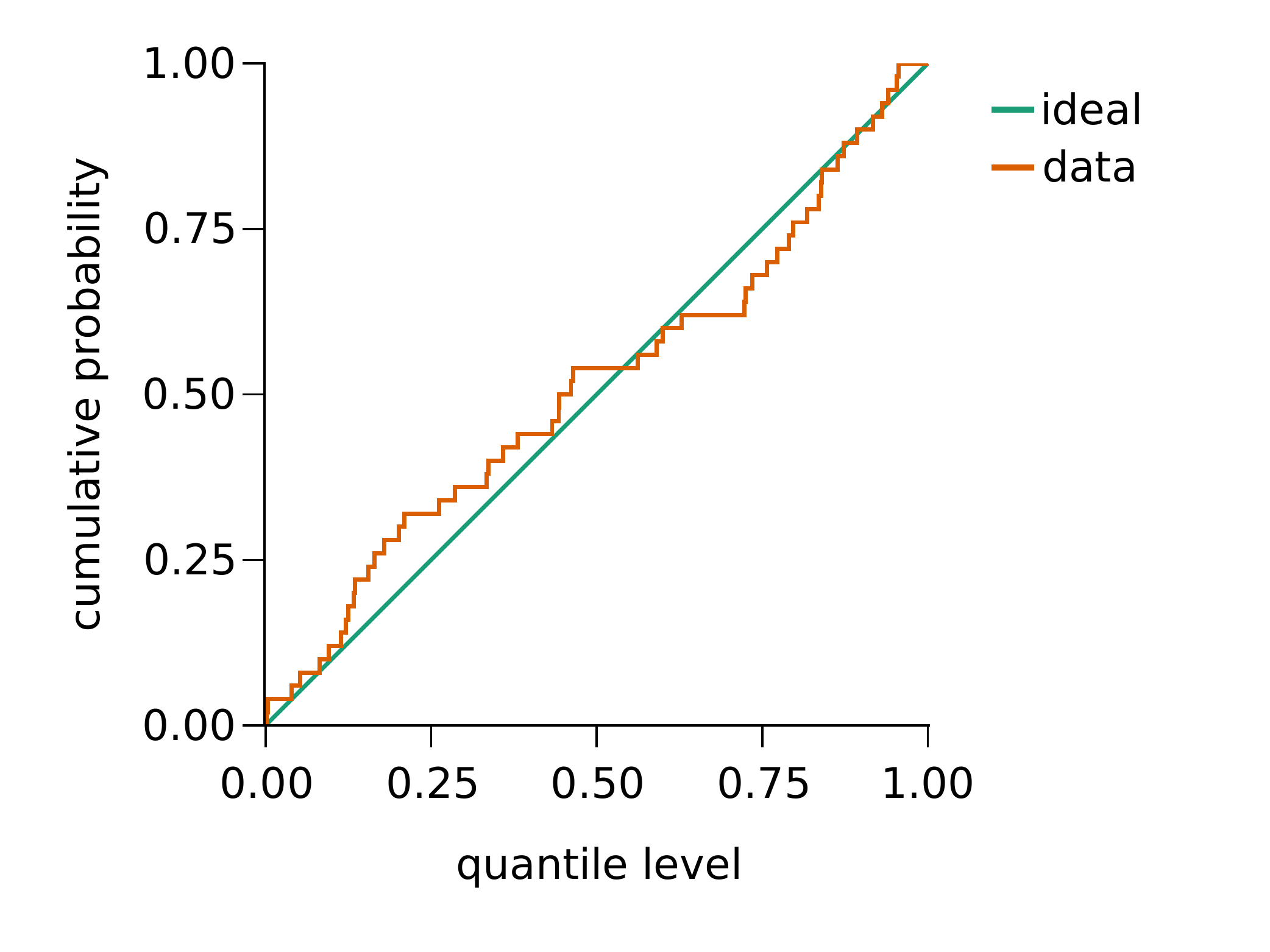}
    \end{subfigure}
    \caption{Illustration of a conditional distribution $\Prob(Y|X)$ with scalar
    feature and target. We consider a Gaussian predictive model $P$, obtained by
    ordinary least squares regression with 100 training data points (orange dots).
    Empirically the predicted quantiles on 50 validation data points appear close
    to being calibrated, although model $P$ is uncalibrated according to
    \cref{def:calibration}. Using the framework in this
    paper, on the same validation data a statistical test allows us to reject
    the null hypothesis that model $P$ is calibrated at a significance level of
    $\alpha = 0.05$ ($p < 0.05$). See~\cref{app:ols} for details.}%
    \label{fig:ols_motivation}
\end{figure}

\Cref{fig:ols_motivation} also raises the question of how to assess
calibration for general target spaces in the sense of \cref{def:calibration},
without having to rely on visual inspection. In classification, measures of
calibration such as the commonly used expected calibration error~(ECE)~\citep{Naeini2015,Guo2017,Vaicenavicius2019,Kull2019}
and the maximum calibration error~(MCE)~\citep{Naeini2015} try to capture
the average and maximal discrepancy between the distributions on the
left hand side and the right hand side of \cref{eq:calibration_definition}
or \cref{eq:calibration_confidence}, respectively.
These measures can be generalized to other target spaces (see~\cref{def:ece_mce}),
but unfortunately estimating these calibration errors from observations
of features and corresponding targets is problematic. Typically, the
predictions are different for (almost) all observations, and hence
estimation of the conditional probability $\Prob{(Y|P_X)}$, which is 
needed in the estimation of ECE and MCE, is challenging even for low-dimensional
target spaces and usually leads to biased and inconsistent
estimators~\citep{Vaicenavicius2019}.

Kernel-based calibration errors such as the maximum mean calibration
error~(MMCE)~\citep{Kumar2018} and the kernel calibration
error~(KCE)~\citep{Widmann2019} for confidence and multi-class
calibration, respectively, can be estimated without first estimating
the conditional probability and hence avoid this issue. They are
defined as the expected value of a weighted sum of the differences of the
left and right hand side of \cref{eq:calibration_definition} for each
class, where the weights are given as a function of the predictions
(of all classes) and chosen such that the calibration error is maximized.
A reformulation with matrix-valued kernels~\citep{Widmann2019} yields
unbiased and differentiable estimators without explicit dependence on $\Prob(Y|P_X)$,
which simplifies the estimation and allows to explicitly account
for calibration in the training objective~\citep{Kumar2018}. Additionally, the
kernel-based framework allows the derivation of reliable statistical hypothesis
tests for calibration in multi-class classification~\citep{Widmann2019}.

However, both the construction as a weighted difference of the class-wise
distributions in \cref{eq:calibration_definition} and the reformulation with
matrix-valued kernels require finite target spaces and hence cannot be applied
to regression problems. To be able to deal with general target spaces, we present
a new and more general framework of calibration errors without these limitations.

Our framework can be used to reason about and test for calibration of \emph{any
probabilistic predictive model}. As explained above, this is in stark contrast with
existing methods
that are restricted to simple output distributions, such as classification
and \emph{scalar-valued} regression problems. 
A \emph{key contribution} of this paper is a new framework that
is applicable to \emph{multivariate} regression, as well as situations when the output is of a
different (e.g., discrete ordinal) or more complex (e.g., graph-structured) type, with clear practical implications.

Within this framework a KCE for general target spaces is obtained. We want to highlight
that for multi-class classification problems its formulation is more intuitive and
simpler to use than the measure proposed by \citet{Widmann2019} based on matrix-valued
kernels.
To ease the application of the KCE we derive several
estimators of the KCE with subquadratic sample complexity and their asymptotic properties
in tests for calibrated models, which improve on existing estimators and tests in the
two-sample test literature by exploiting the special structure of the calibration
framework. Using the proposed framework, we numerically evaluate the calibration of neural
network models and ensembles of such models.

\section{Calibration error: A general framework}

In classification, the distributions on the left and right hand side of
\cref{eq:calibration_definition} can be interpreted as vectors in the probability
simplex. Hence ultimately the distance measure for ECE and MCE
(see~\cref{def:ece_mce}) can be chosen as a distance measure of real-valued vectors.
The total variation, Euclidean, and squared Euclidean distances
are common choices~\citep{Guo2017,Kull2019,Vaicenavicius2019}.
However, in a general setting measuring the discrepancy between
$\Prob(Y|P_X)$ and $P_X$ cannot necessarily be reduced to measuring
distances between vectors. The conditional distribution $\Prob(Y|P_X)$ can
be arbitrarily complex, even if the predicted distributions are restricted to a
simple class of distributions that can be represented as real-valued
vectors. Hence in general we have to resort to dedicated distance measures
of probability distributions.

Additionally, the estimation of 
conditional distributions $\Prob(Y|P_X)$ is challenging, even more so than in the
restricted case of classification, since in general these distributions can be
arbitrarily complex. To circumvent this problem, we propose to use the
following construction:
We define a random variable $Z_X \sim P_X$ obtained from the predictive model and study
 the discrepancy between the \emph{joint}
distributions of the two pairs of random variables $(P_X, Y)$ and $(P_X, Z_X)$, respectively,
instead of the discrepancy between the \emph{conditional} distributions
$\Prob(Y|P_X)$ and $P_X$. Since
\begin{equation*}
    (P_X, Y) \stackrel{d}{=} (P_X, Z_X) \qquad \text{if and only if}
    \qquad \Prob(Y|P_X) = P_X \quad \text{almost surely},
\end{equation*}
model $P$ is calibrated if and only if the distributions of $(P_X, Y)$
and $(P_X, Z_X)$ are equal.

The random variable pairs $(P_X, Y)$ and $(P_X, Z_X)$
take values in the product space $\mathcal{P} \times \mathcal{Y}$, where $\mathcal{P}$
is the space of predicted distributions $P_X$ and $\mathcal{Y}$ is the space of
targets $Y$. For instance, in classification, $\mathcal{P}$ could be the probability
simplex and $\mathcal{Y}$ the set of all class labels, whereas in the case of
Gaussian predictive models for scalar targets $\mathcal{P}$ could be the space of
normal distributions and $\mathcal{Y}$ be $\mathbb{R}$.

The study of the joint distributions of $(P_X, Y)$ and $(P_X, Z_X)$ motivates the
definition of a generally applicable calibration error as an integral probability
metric~\citep{Sriperumbudur2009,Sriperumbudur2012,Mueller1997} between these
distributions. In contrast to common $f$-divergences such as the Kullback-Leibler
divergence, integral probability metrics do not require that one distribution is
absolutely continuous with respect to the other, which cannot be guaranteed in general.

\begin{definitionboxed}\label{def:ce}
    Let $\mathcal{Y}$ denote the space of targets $Y$, and $\mathcal{P}$ the
    space of predicted distributions $P_X$.
    We define the calibration error with respect to a space of functions
    $\mathcal{F}$ of the form $f \colon \mathcal{P} \times \mathcal{Y} \to \mathbb{R}$
    as
    \begin{equation}\label{eq:calibrationerror}
        \mathrm{CE}_\mathcal{F} \coloneqq \sup_{f \in \mathcal{F}} \big| \mathbb{E}_{P_X, Y}{f(P_X, Y)} - \Exp_{P_X, Z_X} f(P_X, Z_X) \big|.
    \end{equation}
\end{definitionboxed}

By construction, if model $P$ is calibrated, then $\mathrm{CE}_{\mathcal{F}} = 0$
regardless of the choice of $\mathcal{F}$. However, the converse statement is not true
for arbitrary function spaces $\mathcal{F}$. From the theory of integral
probability metrics~\citep[see, e.g.,][]{Mueller1997,Sriperumbudur2009,Sriperumbudur2012},
we know that for certain choices of $\mathcal{F}$ the calibration error in
\cref{eq:calibrationerror} is a well-known metric on the product space
$\mathcal{P} \times \mathcal{Y}$, which implies that $\mathrm{CE}_\mathcal{F} = 0$
if and only if model $P$ is calibrated. Prominent examples include the maximum mean
discrepancy\footnote{As we discuss in \cref{sec:kce}, the MMD is a metric if and only if
the employed kernel is characteristic.}~(MMD)~\citep{Gretton2007},
the total variation distance, the Kantorovich distance,
and the Dudley metric~\citep[p.~310]{Dudley1989}.

As pointed out above, \cref{def:ce} is a generalization of the definition
for multi-class classification proposed by \citet{Widmann2019}---which is based on vector-valued functions and only applicable to finite
target spaces---to \emph{any probabilistic predictive model}.
In \cref{app:classification} we show this explicitly and discuss
the special case of classification problems in more detail. Previous results~\citep{Widmann2019}
imply that in classification $\mathrm{MMCE}$ and, for common distance measures
$d(\cdot, \cdot)$ such as the total variation and squared Euclidean distance,
$\mathrm{ECE}_d$ and $\mathrm{MCE}_d$ are special cases of $\mathrm{CE}_{\mathcal{F}}$.
In \cref{app:ece_infinite} we show that our framework also covers
natural extensions of $\mathrm{ECE}_d$ and $\mathrm{MCE}_d$ to countably
infinite discrete target spaces, which to our knowledge have not been studied
before and occur, e.g., in Poisson regression.

The literature of integral
probability metrics suggests that we can resort to estimating
$\mathrm{CE}_{\mathcal{F}}$ from i.i.d.\ samples from the distributions
of $(P_X, Y)$ and $(P_X, Z_X)$. For the MMD, the Kantorovich distance, and
the Dudley metric tractable strongly consistent empirical estimators
exist~\citep{Sriperumbudur2012}. Here the empirical
estimator for the MMD is particularly appealing since compared with the
other estimators \textquote[\cite{Sriperumbudur2012}]{it is
computationally cheaper, the empirical estimate converges at a faster
rate to the population value, and the rate of convergence is independent
of the dimension $d$ of the space (for $S = \mathbb{R}^d$)}.

Our specific design of $(P_X, Z_X)$ can be exploited to improve
on these estimators. If $\mathbb{E}_{Z_x \sim P_x} f(P_x, Z_x)$
can be evaluated analytically for a fixed prediction $P_x$, then
$\mathrm{CE}_{\mathcal{F}}$ can be estimated empirically with reduced
variance by marginalizing out $Z_X$. Otherwise
$\mathbb{E}_{Z_x \sim P_x} f(P_x, Z_x)$ has to be estimated, but in
contrast to the common estimators of the integral probability metrics
discussed above 
the artificial construction of $Z_X$ allows us to
approximate it by numerical integration methods such as (quasi) Monte
Carlo integration or quadrature rules with arbitrarily small error and
variance. Monte Carlo integration preserves statistical properties of
the estimators such as unbiasedness and consistency.

\section{Kernel calibration error}
\label{sec:kce}

For the remaining parts of the paper we focus on the MMD formulation of
$\mathrm{CE}_{\mathcal{F}}$ due to the appealing properties of the common
empirical estimator mentioned above. We derive calibration-specific
analogues of results for the MMD that exploit the special structure of the
distribution of $(P_X, Z_X)$ to improve on existing estimators and tests
in the MMD literature. To the best of our knowledge these variance-reduced
estimators and tests have not been discussed in the MMD literature.

Let $k \colon (\mathcal{P} \times \mathcal{Y}) \times (\mathcal{P} \times \mathcal{Y}) \to \mathbb{R}$
be a measurable kernel with corresponding reproducing kernel Hilbert space (RKHS)
$\mathcal{H}$, and assume
that
\begin{equation*}
\Exp_{P_X,Y} k^{1/2}\big((P_X, Y), (P_X, Y)\big) < \infty
\quad \text{and} \quad
\Exp_{P_X,Z_X} k^{1/2}\big((P_X, Z_X), (P_X, Z_X)\big) < \infty.
\end{equation*}
We discuss how such kernels can be constructed in a generic way in \cref{sec:kernel_choice}
below.

\begin{definition}\label{def:kce}
Let $\mathcal{F}_k$ denote the unit ball in $\mathcal{H}$, i.e.,
$\mathcal{F} \coloneqq \{f \in \mathcal{H} | \|f\|_{\mathcal{H}} \leq 1 \}$.
Then the kernel calibration error~(KCE) with respect to kernel $k$ is
defined as
\begin{equation*}
    \mathrm{KCE}_k \coloneqq \mathrm{CE}_{\mathcal{F}_k} = \sup_{f \in \mathcal{F}_k} \big| \mathbb{E}_{P_X, Y}{f(P_X, Y)} - \Exp_{P_X, Z_X} f(P_X, Z_X) \big|.
\end{equation*}
\end{definition}

As known from the MMD literature, a more explicit formulation can be given for
the squared kernel calibration error $\mathrm{SKCE}_k \coloneqq \mathrm{KCE}^2_k$ (see \cref{lemma:skce}).
A similar explicit expression for $\mathrm{SKCE}_k$ was obtained by \citet{Widmann2019} for
the special case of classification problems. However, their expression relies on 
$\mathcal{Y}$ being finite and is based on matrix-valued kernels over the
finite-dimensional probability simplex $\mathcal{P}$. A key difference to the
expression in \cref{lemma:skce} is that we instead propose to use real-valued kernels defined
on the product space of predictions and targets. This construction is applicable to
arbitrary target spaces and does not require $\mathcal{Y}$ to be finite.

\subsection{Choice of kernel}
\label{sec:kernel_choice}

The construction of the product space $\mathcal{P} \times \mathcal{Y}$ suggests
the use of tensor product kernels $k = k_{\mathcal{P}} \otimes k_{\mathcal{Y}}$, where
$k_{\mathcal{P}} \colon \mathcal{P} \times \mathcal{P} \to \mathbb{R}$ and
$k_{\mathcal{Y}} \colon \mathcal{Y} \times \mathcal{Y} \to \mathbb{R}$ are
kernels on the spaces of predicted distributions and targets, respectively.\footnote{As
mentioned above, our framework rephrases and generalizes the construction used
by \citet{Widmann2019}. The matrix-valued kernels that they employ can be
recovered by setting $k_{\mathcal{P}}$ to a Laplacian kernel on the probability simplex and
$k_{\mathcal{Y}}(y, y') = \delta_{y,y'}$.}

So-called characteristic kernels guarantee that $\mathrm{KCE} = 0$ if and only if the distributions of $(P_X,Y)$ and $(P_X,Z_X)$ are equal, i.e., if model $P$ is calibrated~\citep{Fukumizu2004,Fukumizu2008}.
Many common kernels such as the Gaussian and Laplacian kernel on $\mathbb{R}^d$ are characteristic~\citep{Fukumizu2008}.%
\footnote{For a general discussion about characteristic kernels and their relation to universal kernels we refer to the paper by \citet{Sriperumbudur2011}.}
For kernels $k_{\mathcal{P}}$ and $k_{\mathcal{Y}}$ their characteristic property is necessary but generally not sufficient for the tensor product kernel $k_{\mathcal{P}} \otimes k_{\mathcal{Y}}$ to be characteristic~\citep[Example~1]{Szabo2018}.
However, if $k_{\mathcal{P}}$ and $k_{\mathcal{Y}}$ are characteristic, continuous, bounded, and translation-invariant kernels on $\mathbb{R}^{d_i}$ ($i=1,2$), then $k_{\mathcal{P}} \otimes k_{\mathcal{Y}}$ is characteristic~\citep[Theorem~4]{Szabo2018}.
We use this property to construct characteristic tensor product kernels for regression problems.
More generally, if $k_{\mathcal{P}}$ and $k_{\mathcal{Y}}$ are universal kernels on locally compact Polish spaces, then $k_{\mathcal{P}} \otimes k_{\mathcal{Y}}$ is universal and hence characteristic~\citep[Theorem~5]{Szabo2018}.
In classification even the reverse implication holds, and $k_{\mathcal{P}} \otimes k_{\mathcal{Y}}$ is characteristic if and only if $k_{\mathcal{P}}$ and $k_{\mathcal{Y}}$ are universal~\citep[Corollary~3.15]{Steinwart2021}.%

It is suggestive to construct kernels on general spaces of predicted distributions as
\begin{equation}\label{eq:kernelp}
k_{\mathcal{P}}(p, p') = \exp{\big(- \lambda d^\nu_{\mathcal{P}}(p, p') \big)},
\end{equation}
where $d_{\mathcal{P}}(\cdot, \cdot)$ is a metric on $\mathcal{P}$ and $\nu, \lambda > 0$ 
are kernel hyperparameters. The Wasserstein distance is a widely used metric for
distributions from optimal transport theory that allows to lift a ground
metric on the target space and possesses many important
properties~\citep[see, e.g.,][Chapter~2.4]{Peyre2018}. In general, however,
it does not lead to valid kernels $k_{\mathcal{P}}$, apart from the notable exception
of elliptically contoured distributions such as normal and Laplace
distributions~\citep[Chapter~8.3]{Peyre2018}.

In machine learning, common
probabilistic predictive models output parameters of distributions such as
mean and variance of normal distributions. Naturally these
parameterizations give rise to injective mappings
$\phi \colon \mathcal{P} \to \mathbb{R}^d$ that can be used to define a
Hilbertian metric
\begin{equation*}\label{eq:hilbertian}
    d_{\mathcal{P}}(p, p') = {\|\phi(p) - \phi(p')\|}_2.
\end{equation*}
For such metrics, $k_{\mathcal{P}}$ in \cref{eq:kernelp}
is a valid kernel for all $\lambda > 0$ and
$\nu \in (0, 2]$~\citep[Corollary~3.3.3, Proposition~3.2.7]{Berg1984}.
In~\cref{app:mixture} we show that for many mixture models, and hence model ensembles,
Hilbertian metrics between model components
can be lifted to Hilbertian metrics between mixture models. This construction
is a generalization of the Wasserstein-like distance for Gaussian mixture models proposed
by \citet{Delon2019,Chen2019,Chen2020}.

\subsection{Estimation}

Let $(X_1, Y_1), \ldots, (X_n, Y_n)$ be a data set of features and targets which are i.i.d.\
according to the law of $(X, Y)$. Moreover, for notational brevity, for
$(p, y), (p', y') \in \mathcal{P} \times \mathcal{Y}$ we let
\begin{multline*}
    h\big((p, y), (p', y')\big) \coloneqq k\big((p, y), (p', y')\big) - \Exp_{Z \sim p} k\big((p, Z), (p', y')\big) \\
    - \Exp_{Z' \sim p'} k\big((p, y), (p', Z')\big) + \Exp_{Z \sim p, Z' \sim p'} k\big((p, Z), (p', Z')\big).
\end{multline*}
Note that in contrast to the regular MMD we marginalize out $Z$ and $Z'$. Similar to the MMD, there exist
consistent estimators of the SKCE, both biased and unbiased.

\begin{restatable}{lemma}{lemmaskceb}\label{lemma:skceb}
The plug-in estimator of $\mathrm{SKCE}_k$ is non-negatively biased. It is given by
\begin{equation*}
    \widehat{\mathrm{SKCE}}_k =
    \frac{1}{n^2} \sum_{i,j=1}^n h\big((P_{X_i}, Y_i), (P_{X_j}, Y_j)\big).
\end{equation*}
\end{restatable}

Inspired by the block tests for the regular MMD~\citep{Zaremba2013}, we define the
following class of unbiased estimators. Note that in contrast to $\widehat{\mathrm{SKCE}}_k$
they do not include terms of the form $h\big((P_{X_i}, Y_i), (P_{X_i}, Y_i)\big)$.

\begin{restatable}{lemma}{lemmaskceblock}\label{lemma:skceblock}
The block estimator of $\mathrm{SKCE}_k$ with
block size $B \in \{2,\ldots,n\}$, given by
\begin{equation*}
    \widehat{\mathrm{SKCE}}_{k,B} \coloneqq \bigg\lfloor \frac{n}{B} \bigg\rfloor^{-1} \sum_{b=1}^{\lfloor n / B \rfloor} \binom{B}{2}^{-1} \sum_{(b - 1) B < i < j \leq bB} h\big((P_{X_{i}}, Y_i), (P_{X_j}, Y_j)\big),
\end{equation*}
is an unbiased estimator of $\mathrm{SKCE}_k$.
\end{restatable}

The extremal estimator with $B = n$ is a so-called U-statistic of
$\mathrm{SKCE}_k$~\citep{Hoeffding1948,Vaart1998}, and hence it is
the minimum variance unbiased estimator. All presented estimators are
consistent, i.e., they converge to $\mathrm{SKCE}_k$ almost surely as
the number $n$ of data points goes to infinity. The sample complexity
of $\widehat{\mathrm{SKCE}}_k$ and $\widehat{\mathrm{SKCE}}_{k,B}$ is
$O(n^2)$ and $O(Bn)$, respectively.

\subsection{Calibration tests}

A fundamental issue with calibration errors in general, including ECE,
is that their empirical estimates do not provide an answer to the
question if a model is actually calibrated. Even if the measure is
guaranteed to be zero if and only if the model is calibrated,
usually the estimates of calibrated models are non-zero due to
randomness in the data and (possibly) the estimation procedure.
In classification, statistical hypothesis tests of the null hypothesis
\begin{equation*}
    H_0 \colon \text{model } P \text{ is calibrated},
\end{equation*}
so-called calibration tests, have been proposed as a tool for checking
rigorously if $P$ is calibrated~\citep{Broecker2007,Vaicenavicius2019,Widmann2019}.
For multi-class classification, \citet{Widmann2019} suggested 
calibration tests based on the asymptotic distributions of estimators
of the previously formulated KCE. Although for finite data sets the
asymptotic distributions are only approximations of the actual
distributions of these estimators, in their experiments with 10
classes the resulting $p$-value approximations seemed reliable whereas
$p$-values obtained by so-called consistency
resampling~\citep{Broecker2007,Vaicenavicius2019} underestimated
the $p$-value and hence rejected the null hypothesis too
often~\citep{Widmann2019}.

For fixed block sizes
$    \sqrt{\lfloor n / B \rfloor} \big(\widehat{\mathrm{SKCE}}_{k,B} - \mathrm{SKCE}_k\big) \xrightarrow{d} \mathcal{N}\big(0, \sigma^2_B\big)$ as $n \to \infty$,
and, under $H_0$,
$n \widehat{\mathrm{SKCE}}_{k,n} \xrightarrow{d} \sum_{i=1}^\infty \lambda_i (Z_i - 1)$ as $n \to \infty$,
where $Z_i$ are independent $\chi_1^2$ distributed random variables. See~\cref{app:theory} for details and definitions of the involved constants.
From these results one can derive
calibration tests that extend and generalize the existing tests
for classification problems, as explained in
\cref{remark:skceb_fixed,remark:skceb_quadratic}.
Our formulation illustrates also the close connection of these tests
to different two-sample tests~\citep{Gretton2007,Zaremba2013}.

\section{Alternative approaches}

For two-sample tests, \citet{Chwialkowski2015} suggested the use
of the so-called unnormalized mean embedding~(UME) to overcome
the quadratic sample complexity of the minimum variance unbiased
estimator and its intractable asymptotic distribution. As we show
in \cref{app:ucme}, there exists an analogous measure of calibration,
termed unnormalized calibration mean embedding~(UCME), with a corresponding
calibration mean embedding~(CME) test.

As an alternative to our construction based on the
joint distributions of $(P_X, Y)$ and $(P_X, Z_X)$, one could try
to directly compare the conditional distributions $\Prob(Y|P_X)$
and $\Prob(Z_X|P_X) = P_X$. For instance, \citet{Ren2016} proposed the
conditional MMD based on the so-called conditional kernel mean
embedding~\citep{Song2009,Song2013}. However, as noted by
\citet{Park2020}, its common definition as operator
between two RKHS is based on very restrictive assumptions, which are violated
in many situations~\citep[see, e.g.,][Footnote~4]{Fukumizu2013}
and typically require regularized estimates. Hence, even theoretically,
often the conditional MMD is
\textquote[\cite{Park2020}]{not an exact measure of discrepancy between conditional distributions}.
In contrast, the maximum conditional mean discrepancy~(MCMD) proposed in a
concurrent work by \citet{Park2020} is a random variable derived from much
weaker measure-theoretical assumptions.
The MCMD provides a local discrepancy conditional on random predictions
whereas KCE is a global real-valued summary of these local
discrepancies.\footnote{In our calibration setting, the $\mathrm{MCMD}$ is
almost surely equal to
$
    \sup_{f \in \mathcal{F}_\mathcal{Y}}
    \big| \Exp_{Y|P_X}\big(f(Y)|P_X\big) - \Exp_{Z_X|P_X}\big(f(Z_X)|P_X\big) \big|,
$
where $\mathcal{F}_{\mathcal{Y}} \coloneqq \{ f \colon \mathcal{Y} \to \mathbb{R} | \|f\|_{\mathcal{H}_{\mathcal{Y}}} \leq 1 \}$
for an RKHS $\mathcal{H}_{\mathcal{Y}}$ with kernel
$k_{\mathcal{Y}} \colon \mathcal{Y} \times \mathcal{Y} \to \mathbb{R}$.
If kernel $k_{\mathcal{Y}}$ is characteristic, $\mathrm{MCMD} = 0$ almost
surely if and only if model $P$ is calibrated~\citep[Theorem~3.7]{Park2020}. Although the definition of MCMD
only requires
a kernel $k_{\mathcal{Y}}$ on the target space, a kernel $k_{\mathcal{P}}$ on the space of predictions
has to be specified for the evaluation of its regularized estimates.}

\section{Experiments}

In our experiments we evaluate the computational efficiency and empirical properties
of the proposed calibration error estimators and calibration tests on both calibrated
and uncalibrated models. By means of a classic regression problem from statistics
literature, we demonstrate that the estimators and tests can be used for the
evaluation of calibration of neural network models and ensembles of such models. This
section contains only an high-level overview of these experiments to conserve space but
all experimental details are provided in \cref{app:experiments}.

\subsection{Empirical properties and computational efficiency}

We evaluate error, variance, and computation time of calibration error
estimators for calibrated and uncalibrated Gaussian predictive models in synthetic
regression problems. The results empirically confirm the
consistency of the estimators and the computational efficiency of the estimator with
block size $B = 2$ which, however, comes at the cost of increased error and variance.

Additionally, we evaluate empirical test errors of calibration tests
at a fixed significance level $\alpha = 0.05$. The evaluations, visualized in
\cref{fig:synthetic_tests_10_main} for models with ten-dimensional targets, demonstrate
empirically that the percentage of incorrect rejections of $H_0$ converges to the set
significance level as the number of samples increases.
Moreover, the results highlight the computational burden of the calibration test
that estimates quantiles of the intractable asymptotic distribution of
$n\widehat{\mathrm{SKCE}}_{k,n}$ by bootstrapping. As expected, due to the larger
variance of $\widehat{\mathrm{SKCE}}_{k,2}$ the test with fixed block size $B = 2$
shows a decreased test power although being computationally much more efficient.

\begin{figure}[!hbtp]
    \centering
    \includegraphics[scale=0.35]{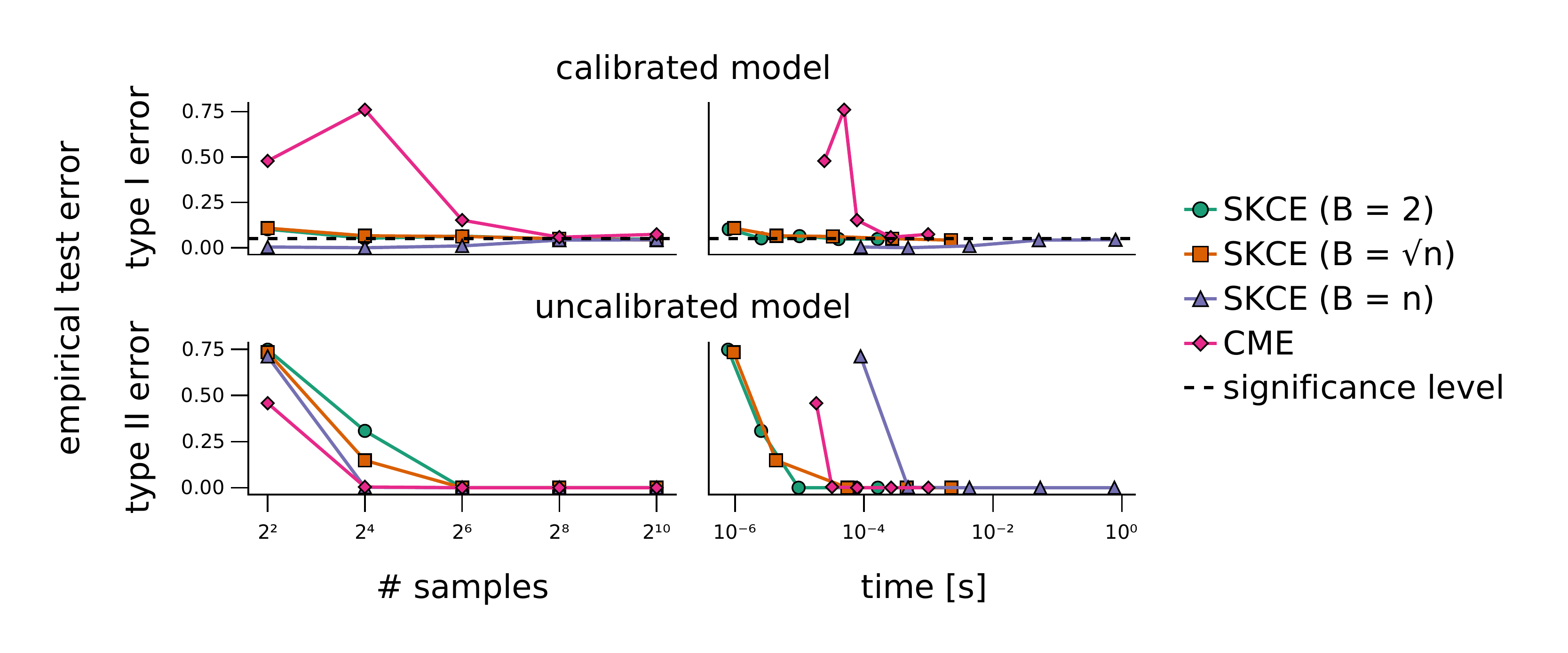}
    \caption{Empirical test errors for 500 data sets of $n \in \{4, 16, 64, 256, 1024\}$
    samples from models with targets of dimension $d = 10$. The dashed black line
    indicates the set signficance level $\alpha = 0.05$.}%
    \label{fig:synthetic_tests_10_main}
\end{figure}

\subsection{Friedman 1 regression problem}

The Friedman~1 regression problem~\citep{Friedman1979,Friedman1983,Friedman1991}
is a classic non-linear regression problem with ten-dimensional features and
real-valued targets with Gaussian noise. We train a Gaussian predictive model
whose mean is modelled by a shallow neural network and
a single scalar
variance
parameter (consistent with the data-generating model) ten times
with different initial parameters. \Cref{fig:friedman1_zoom} shows estimates of the
mean squared error (MSE),
the average negative log-likelihood (NLL), $\mathrm{SKCE}_k$, and a $p$-value
approximation for these models and their ensemble on the training and a separate
test data set.
All estimates indicate consistently
that the models are overfit after 1500 training iterations. The estimations of
$\mathrm{SKCE}_k$ and the $p$-values allow to focus on calibration specifically,
whereas MSE indicates accuracy only and NLL, as any proper scoring rule~\citep{Broecker2009}, provides a summary of calibration and accuracy. The estimation
of $\mathrm{SKCE}_k$ in addition to NLL could serve as another source of information
for early stopping and model selection.

\begin{figure}[!hbtp]
    \centering
    \includegraphics[scale=0.35]{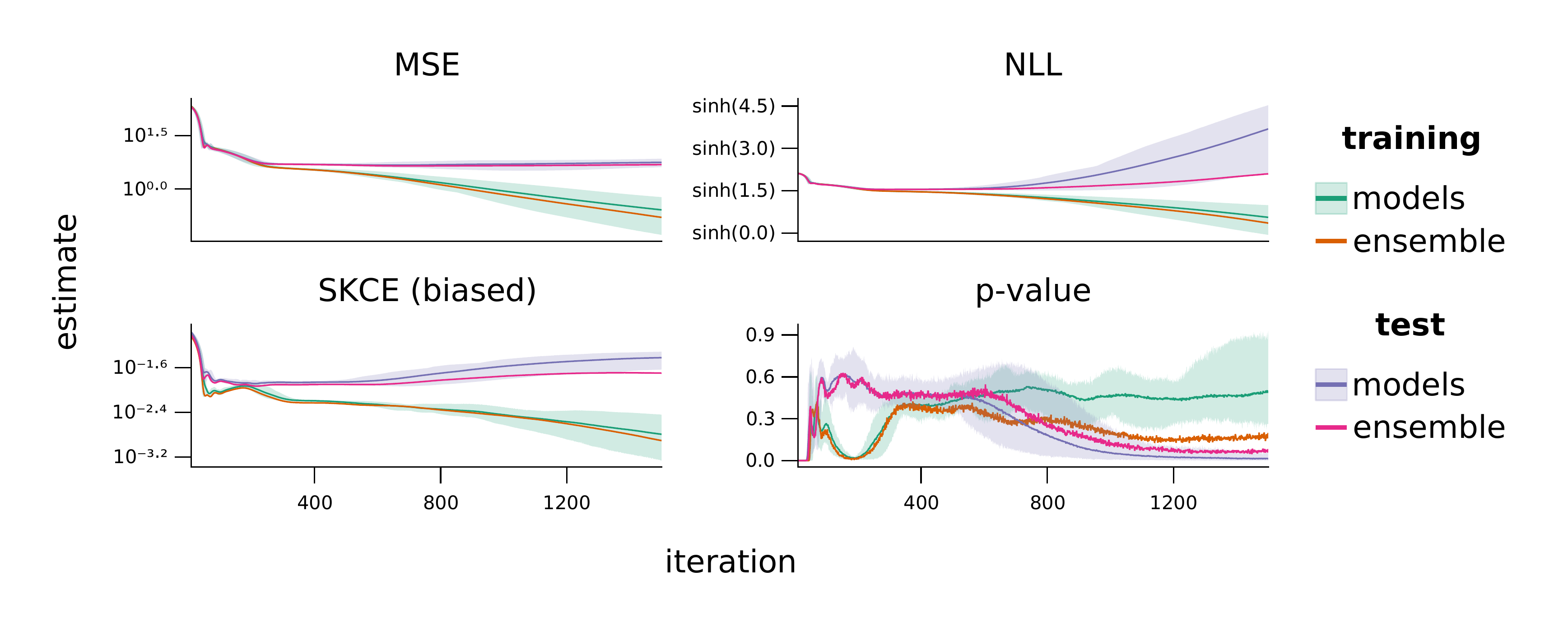}
    \caption{Mean squared error (MSE), average negative log-likelihood (NLL),
    $\widehat{\mathrm{SKCE}}_k$ (SKCE (biased)), and $p$-value approximation ($p$-value) of
    ten Gaussian predictive models for the Friedman~1 regression problem versus
    the number of training iterations.
    Evaluations on the training data set (100 samples) are displayed in green
    and orange, and on the test data set (50 samples) in blue and purple.
    The green and blue line and their surrounding bands represent
    the mean and the range of the evaluations of the ten models. The orange and
    purple lines visualize the evaluations of their ensemble.}%
    \label{fig:friedman1_zoom}
\end{figure}

\section{Conclusion}

We presented a framework of calibration estimators and tests for any
probabilistic model that captures both classification and regression
problems of arbitrary dimension
as well as other predictive models.
We successfully applied it
for measuring calibration of (ensembles of) neural network models.

Our framework highlights connections of calibration to two-sample
tests and optimal transport theory which we expect
to be fruitful for future research. For instance, the power of calibration
tests could be improved by heuristics and theoretical results about
suitable kernel choices or hyperparameters~\citep[cf.][]{Jitkrittum2016}.
It would also be interesting to investigate alternatives to $\mathrm{KCE}$
captured by our framework, e.g., by exploiting recent advances in
optimal transport theory~\citep[cf.][]{Genevay2016}.

Since the presented estimators of $\mathrm{SKCE}_k$ are differentiable, we
imagine that our framework could be helpful for improving calibration of
predictive models, during training~\citep[cf.][]{Kumar2018} or post-hoc.
Currently, many calibration
methods~\citep[see, e.g.,][]{Guo2017,Kull2019,Song2019} are based on
optimizing the log-likelihood since it is a strictly proper scoring rule and
thus encourages \emph{both} accurate and reliable predictions. However, as
for any proper scoring rule, \textquote[\cite{Broecker2009}]{Per se, it is impossible to say
how the score will rank unreliable forecast schemes \textelp{}.
The lack of reliability of one forecast scheme might be outbalanced
by the lack of resolution of the other}. In other words, if one does not
use a calibration method such as temperature scaling~\citep{Guo2017} that
keeps accuracy invariant\footnote{Temperature scaling can be defined
and applied for general probabilistic predictive models, see~\cref{app:temperature}.},
it is unclear if the resulting model is
trading off calibration for accuracy when using log-likelihood for
re-calibration. Thus hypothetically flexible calibration methods might
benefit from using the presented calibration error estimators.

\subsubsection*{Acknowledgments}

We thank the reviewers for all the constructive feedback on our paper.
This research is financially supported by the Swedish Research Council via the
projects \emph{Learning of Large-Scale Probabilistic Dynamical Models} (contract
number: 2016-04278), \emph{Counterfactual Prediction Methods for Heterogeneous
Populations} (contract number: 2018-05040), and
\emph{Handling Uncertainty in Machine Learning Systems} (contract number: 2020-04122),
by the Swedish Foundation for Strategic Research via the project
\emph{Probabilistic Modeling and Inference for Machine Learning} (contract number: ICA16-0015),
by the Wallenberg AI, Autonomous Systems and Software Program (WASP) funded by the Knut and Alice
Wallenberg Foundation, and by ELLIIT.

\bibliography{references}
\bibliographystyle{iclr2021_conference}

\clearpage

\appendix
\numberwithin{equation}{section}
\renewcommand*{\thetheorem}{\thesection.\arabic{theorem}}
\renewcommand*{\thelemma}{\thesection.\arabic{lemma}}
\renewcommand*{\theproposition}{\thesection.\arabic{proposition}}
\renewcommand*{\thecorollary}{\thesection.\arabic{corollary}}
\renewcommand*{\thedefinition}{\thesection.\arabic{definition}}
\renewcommand*{\theexample}{\thesection.\arabic{example}}
\renewcommand*{\theremark}{\thesection.\arabic{remark}}
\makeatletter
\@addtoreset{theorem}{section}
\@addtoreset{lemma}{section}
\@addtoreset{proposition}{section}
\@addtoreset{corollary}{section}
\@addtoreset{definition}{section}
\@addtoreset{example}{section}
\@addtoreset{remark}{section}
\makeatother

\section{Experiments}
\label{app:experiments}

The source code of the experiments and instructions for reproducing the
results are available at
\url{https://github.com/devmotion/Calibration_ICLR2021}. Additional
material such as automatically generated HTML output and Jupyter
notebooks is available at
\url{https://devmotion.github.io/Calibration_ICLR2021/}.

\subsection{Ordinary least squares}\label{app:ols}

We consider a regression problem with scalar feature $X$ and scalar target $Y$ with input-dependent Gaussian noise
that is inspired by a problem by \citet{Gustafsson2019}. Feature $X$ is
distributed uniformly at random in $[-1, 1]$, and target $Y$ is distributed according to
\begin{equation*}
    Y \sim \sin(\pi X) + | 1 + X | \epsilon,
\end{equation*}
where $\epsilon \sim \mathcal{N}(0, 0.15^2)$. We train a linear regression model $P$ with
homoscedastic variance using ordinary least squares and a data set of 100 i.i.d.\
pairs of feature $X$ and target $Y$ (see~\cref{fig:ols_heatmap}).

\begin{figure}[hpt]
    \begin{center}
        \includegraphics[scale=0.35]{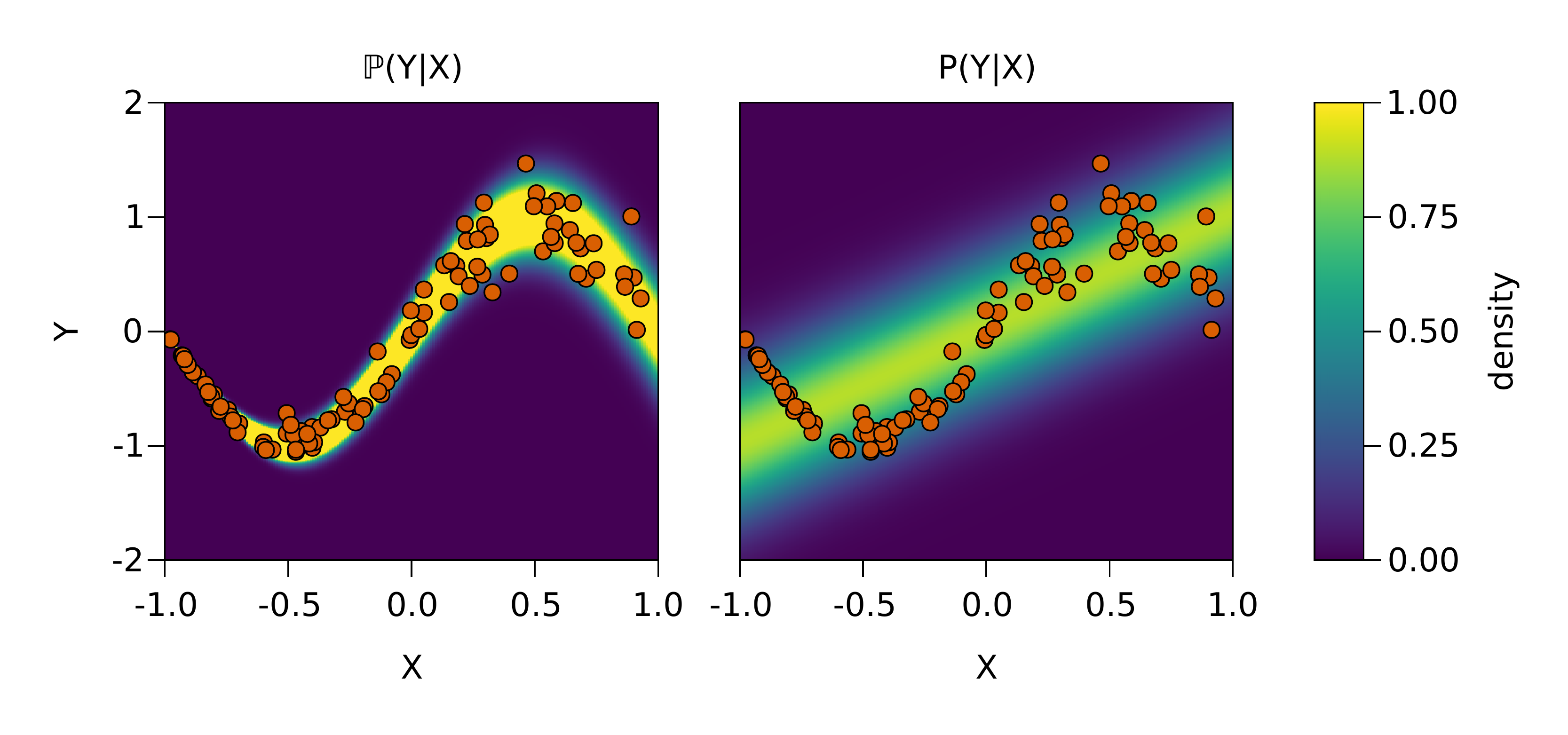}
        \caption{Data generating distribution $\Prob(Y|X)$ and predicted distribution $P(Y|X)$ of the linear regression model.
        Training data is indicated by orange dots.}
        \label{fig:ols_heatmap}
    \end{center}
\end{figure}

A validation data set of $n = 50$ i.i.d.\ pairs of $X$ and $Y$ is used to
evaluate the empirical cumulative probability
\begin{equation*}
    n^{-1} \sum_{i=1}^n \mathbbm{1}_{[0,\tau]}\big(P(Y \leq Y_i | X = X_i)\big)
\end{equation*}
of model $P$ for quantile levels $\tau \in [0, 1]$. Model $P$ would be quantile
calibrated~\citep{Song2019} if
\begin{equation*}
    \tau = \Prob_{X',Y'}\big(P(Y \leq Y' | X = X') \leq \tau \big)
\end{equation*}
for all $\tau \in [0,1]$, where $(X, Y)$ and $(X', Y')$ are independent identically
distributed pairs of random variables (see~\cref{fig:ols_quantiles}).

\begin{figure}[hpt]
    \begin{center}
        \includegraphics[scale=0.35]{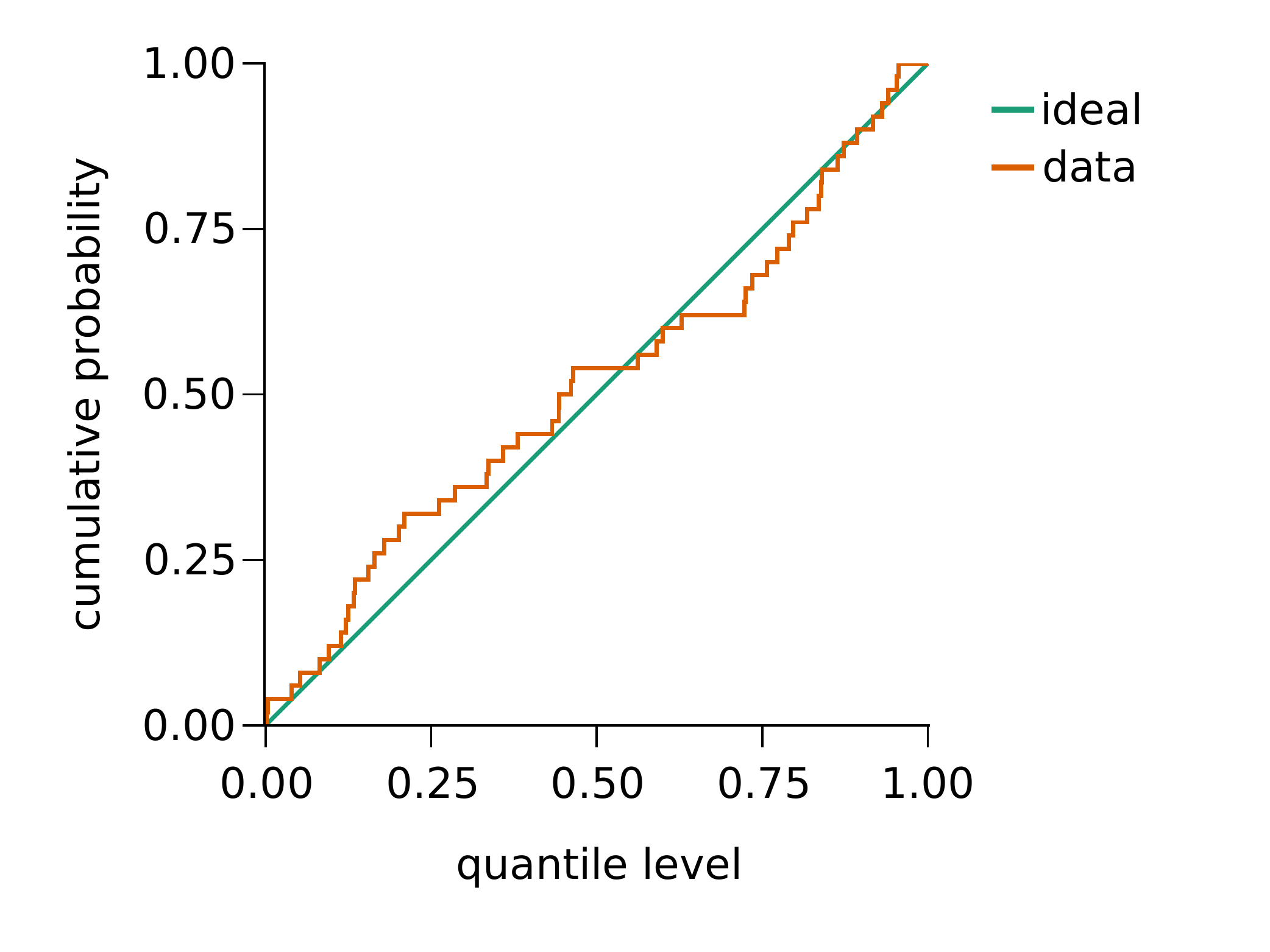}
        \caption{Cumulative probability versus quantile level for the linear regression
        model on the validation data (orange curve). The green curve indicates the theoretical ideal
        for a quantile-calibrated model.}
        \label{fig:ols_quantiles}
    \end{center}
\end{figure}

Additionally, we compute a $p$-value estimate of the null hypothesis $H_0$ that model $P$
is calibrated using an estimation of the quantile of the asymptotic distribution
of $n \widehat{\mathrm{SKCE}}_{k,n}$ with 100000 bootstrap samples
on the validation data set (see~\cref{remark:skceb_quadratic}). Kernel $k$ is chosen
as the tensor product kernel
\begin{equation*}
    \begin{split}
        k\big((p, y), (p', y')\big) &= \exp{\big(-W_2(p, p')\big)} \exp{\big(-(y - y')^2/2 \big)} \\
        &= \exp{\bigg(-\sqrt{(m_p - m_{p'})^2 + (\sigma_p - \sigma_{p'})^2}\bigg)} \exp{\big(- (y - y')^2/2\big)},
    \end{split}
\end{equation*}
where $W_2$ is the 2-Wasserstein distance and $m_p, m_{p'}$ and $\sigma_p, \sigma_{p'}$ denote
the mean and the standard deviation of the normal distributions $p$ and $p'$ (see~\cref{app:normal}).
We obtain $p < 0.05$ in our experiment, and hence the calibration test rejects $H_0$ at the
significance level $\alpha = 0.05$.

\subsection{Empirical properties and computational efficiency}\label{app:efficiency}

We study two setups with $d$-dimensional targets $Y$ and normal distributions $P_X$
of the form $\mathcal{N}(c \mathbf{1}_d, 0.1^2 \mathbf{I}_d)$ as predictions,
where $c \sim \mathrm{U}(0, 1)$.
Since calibration analysis is only based on the targets and predicted distributions,
we neglect features $X$ in these experiments and specify only the distributions of
$Y$ and $P_X$.

In the first setup we simulate a calibrated model. We achieve this by sampling
targets from the predicted distributions, i.e., by defining the conditional
distribution of $Y$ given $P_X$ as
\begin{equation*}
    Y \,|\, P_X = \mathcal{N}(\mu, \Sigma) \sim \mathcal{N}(\mu, \Sigma).
\end{equation*}

In the second setup we simulate an uncalibrated model of the form
\begin{equation*}
    Y \,|\, P_X = \mathcal{N}(\mu, \Sigma) \sim \mathcal{N}([0.1, \mu_2, \ldots, \mu_d]^{\mathsf{T}}, \Sigma).
\end{equation*}

We perform an evaluation of the convergence and computation time of the biased estimator
$\widehat{\mathrm{SKCE}}_k$ and the unbiased estimator $\widehat{\mathrm{SKCE}}_{k,B}$ with
blocks of size $B \in \{2, \sqrt{n}, n\}$. We use the tensor product kernel
\begin{equation*}
\begin{split}
    k\big((p, y), (p', y')\big) &= \exp{\big(- W_2(p, p')\big)} \exp{\big(-(y - y')^2/2\big)} \\
    &= \exp{\bigg(-\sqrt{(m_p - m_{p'})^2 + (\sigma_p - \sigma_{p'})^2}\bigg)} \exp{\big( - (y - y')^2/2\big)},
\end{split}
\end{equation*}
where $W_2$ is the 2-Wasserstein distance and $m_p, m_{p'}$ and $\sigma_p, \sigma_{p'}$
denote the mean and the standard deviation of the normal distributions $p$ and $p'$.

\Cref{fig:synthetic_estimators_calibrated_1,fig:synthetic_estimators_calibrated_10,fig:synthetic_estimators_uncalibrated_1,fig:synthetic_estimators_uncalibrated_10} visualize the mean absolute error and the variance
of the resulting estimates for the calibrated and the uncalibrated model with dimensions
$d = 1$ and $d = 10$ for $500$ independently drawn data sets of $n \in \{4, 16, 64, 256, 1024\}$
samples of $(P_X, Y)$. Computation time indicates the minimum time in the $500$ evaluations
on a computer with a 3.6\,GHz processor. The ground truth values of the uncalibrated models were
estimated by averaging the estimates of $\widehat{\mathrm{SKCE}}_{k,1000}$
for $1000$ independently drawn data sets of $1000$ samples of $(P_X, Y)$ (independent from
the data sets used for the evaluation of the estimates). \Cref{fig:synthetic_estimators_calibrated_1,fig:synthetic_estimators_calibrated_10} illustrate that
the computational efficiency of $\widehat{\mathrm{SKCE}}_{k,2}$ in comparison with the
other estimators comes at the cost of increased error and variance for the calibrated models
for fixed numbers of samples.

\begin{figure}[hpt]
    \begin{center}
        \includegraphics[scale=0.35]{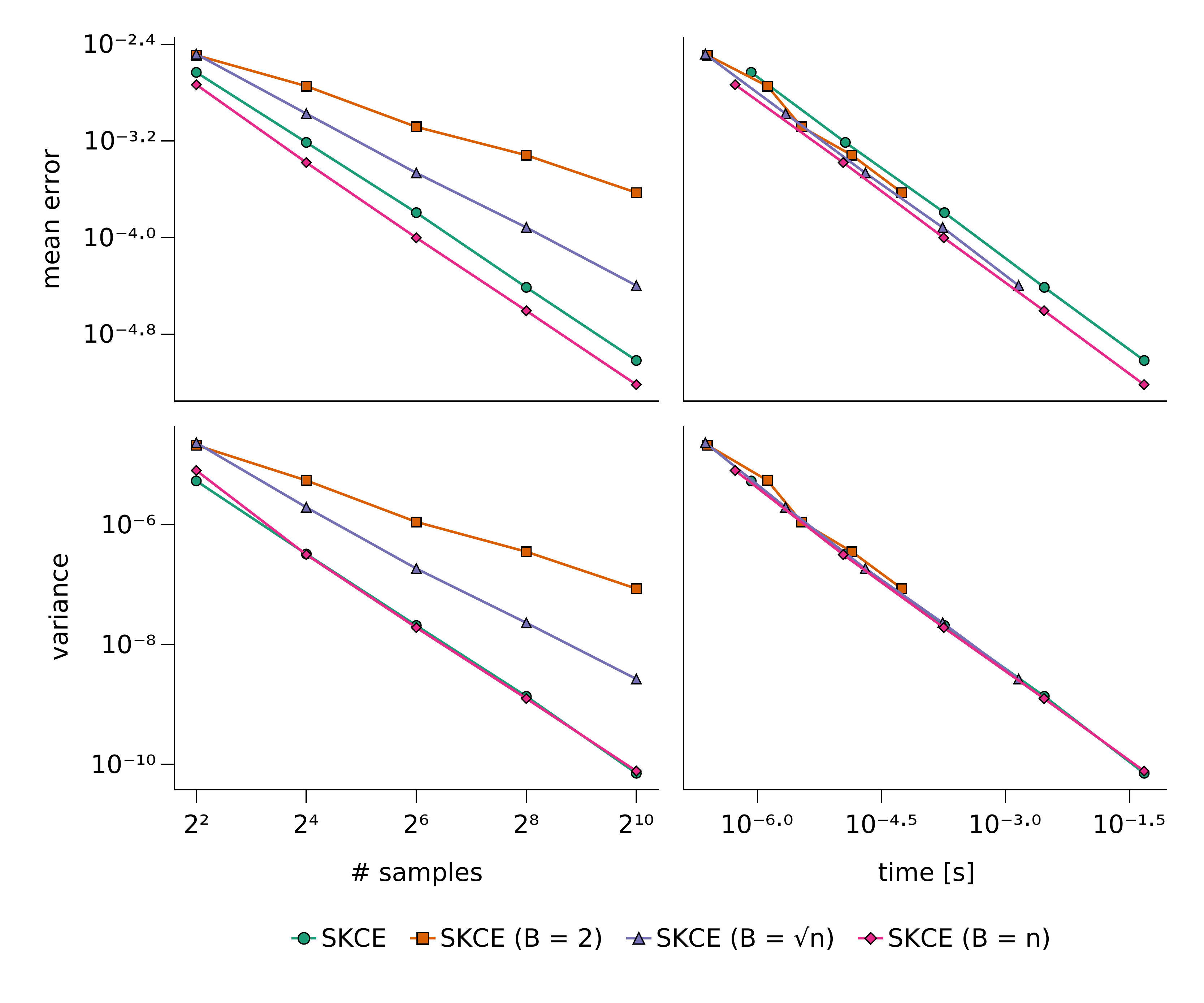}
        \caption{Mean absolute error and variance of 500 calibration error estimates for
        data sets of $n \in \{4, 16, 64, 256, 1024\}$ samples from the calibrated model
        of dimension $d = 1$.}
        \label{fig:synthetic_estimators_calibrated_1}
    \end{center}
\end{figure}

\begin{figure}[hpt]
    \begin{center}
        \includegraphics[scale=0.35]{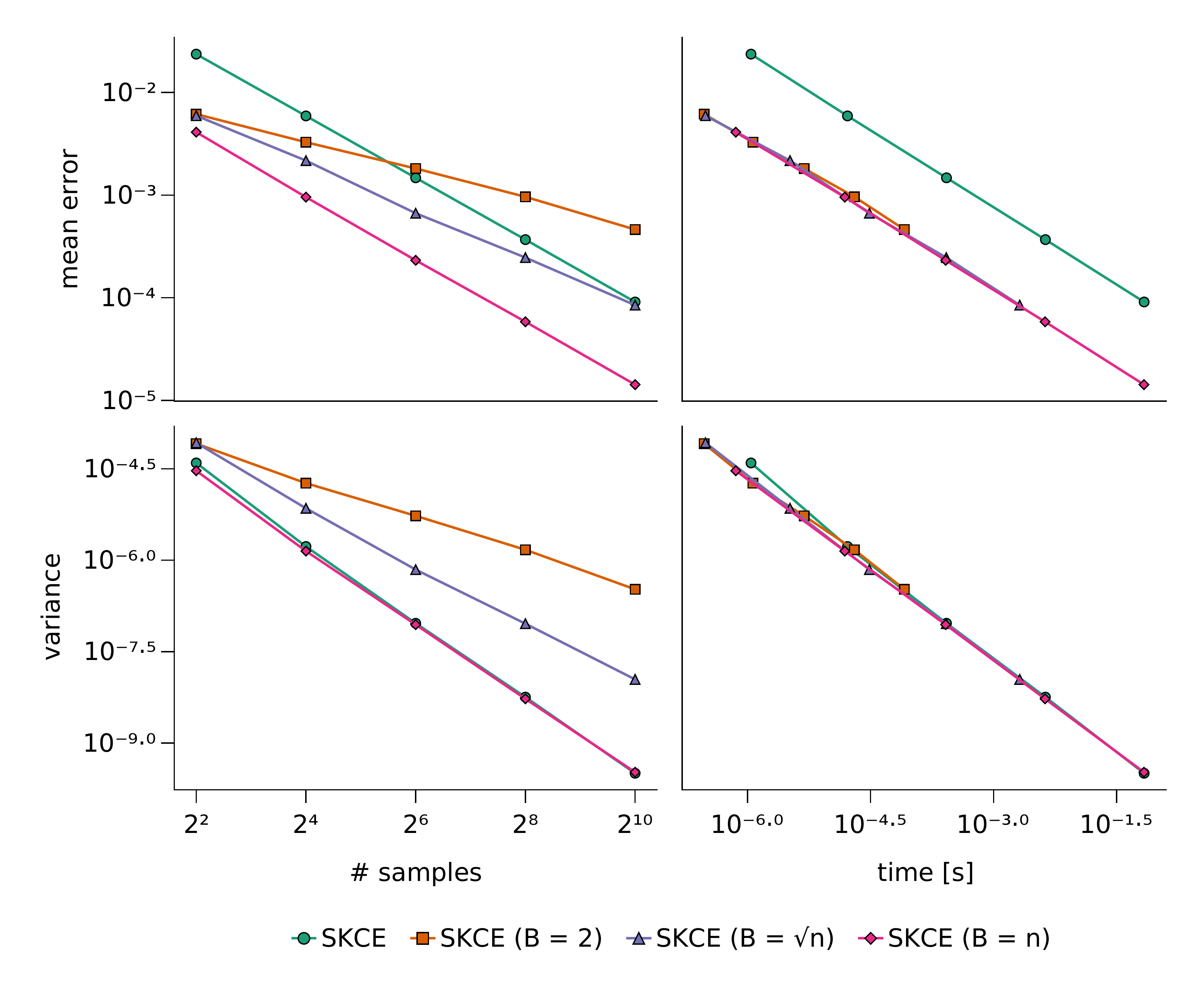}
        \caption{Mean absolute error and variance of 500 calibration error estimates for
        data sets of $n \in \{4, 16, 64, 256, 1024\}$ samples from the calibrated model
        of dimension $d = 10$.}
        \label{fig:synthetic_estimators_calibrated_10}
    \end{center}
\end{figure}

\begin{figure}[hpt]
    \begin{center}
        \includegraphics[scale=0.35]{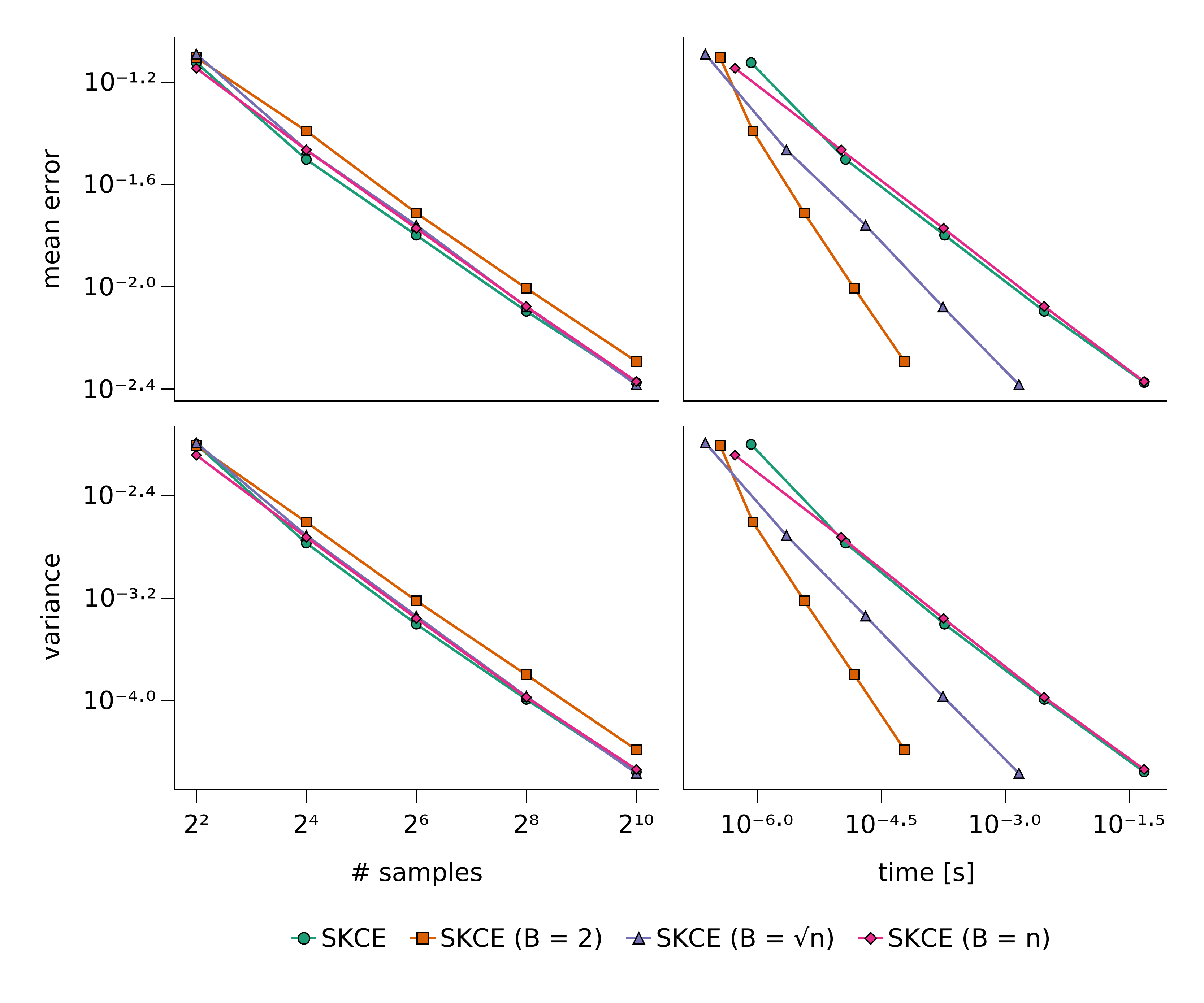}
        \caption{Mean absolute error and variance of 500 calibration error estimates for
        data sets of $n \in \{4, 16, 64, 256, 1024\}$ samples from the uncalibrated model
        of dimension $d = 1$.}
        \label{fig:synthetic_estimators_uncalibrated_1}
    \end{center}
\end{figure}

\begin{figure}[hpt]
    \begin{center}
        \includegraphics[scale=0.35]{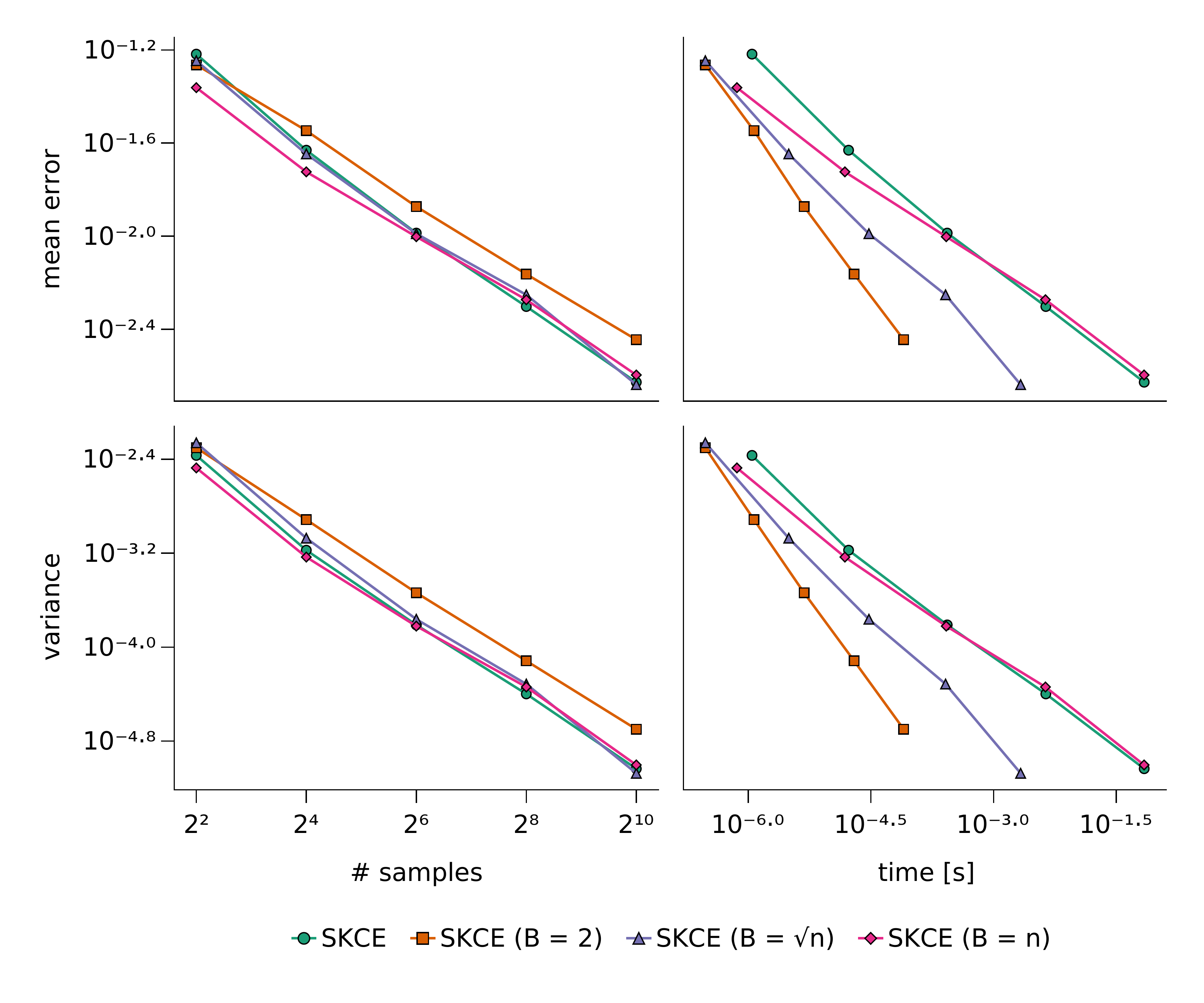}
        \caption{Mean absolute error and variance of 500 calibration error estimates for
        data sets of $n \in \{4, 16, 64, 256, 1024\}$ samples from the uncalibrated model
        of dimension $d = 10$.}
        \label{fig:synthetic_estimators_uncalibrated_10}
    \end{center}
\end{figure}

We compare calibration tests based on the (tractable) asymptotic distribution of
$\sqrt{\lfloor n / B \rfloor} \widehat{\mathrm{SKCE}}_{k,B}$ with fixed block
size $B \in \{2, \sqrt{n}\}$
(see~\cref{remark:skceb_fixed}), the (intractable) asymptotic distribution of
$n\widehat{\mathrm{SKCE}}_{k,n}$ which is approximated with 1000 bootstrap
samples (see~\cref{remark:skceb_quadratic}), and a Hotelling's $T^2$-statistic for
$\mathrm{UCME}_{k,10}$ with 10 test locations (see~\cref{app:ucme}). We compute
the empirical test errors (percentage of false rejections of the null hypothesis $H_0$
that model $P$ is calibrated if $P$ is calibrated, and percentage of false non-rejections
of $H_0$ if $P$ is not calibrated) at a fixed significance level $\alpha = 0.05$ and
the minimal computation time for the calibrated and the uncalibrated model with
dimensions $d = 1$ and $d = 10$ for $500$ independently drawn data sets of
$n \in \{4, 16, 64, 256, 1024\}$ samples of $(P_X, Y)$. The 10 test
predictions of the $\mathrm{CME}$ test are of the form $\mathcal{N}(m, 0.1^2 \mathbf{I}_d)$
where $m$ is distributed uniformly at random in the $d$-dimensional unit hypercube $[0,1]^d$,
the corresponding 10 test targets are i.i.d.\ according to
$\mathcal{N}(\mathbf{0}, 0.1^2 \mathbf{I}_d)$.

\Cref{fig:synthetic_tests_1,fig:synthetic_tests_10} show that all tests adhere to the set
significance level asymptotically as the number of samples increases. The convergence of
the $\mathrm{CME}$ test with 10 test locations is found to be much slower than the convergence
of all other tests.
The tests based on the tractable asymptotic distribution of
$\sqrt{\lfloor n / B \rfloor} \widehat{\mathrm{SKCE}}_{k,B}$ for
fixed block size $B$ are orders of magnitudes faster than the test based on the intractable
asymptotic distribution of $n \widehat{\mathrm{SKCE}}_{k,n}$, approximated with 1000 bootstrap
samples. We see that the efficiency gain comes at the cost of decreased test power for smaller
number of samples, explained by the increasing variance of $\widehat{\mathrm{SKCE}}_{k,B}$ for
decreasing block sizes $B$.
However, in our examples the test based on $\widehat{\mathrm{SKCE}}_{k,\sqrt{n}}$
still achieves good test power for reasonably large number of samples (> 30).

\begin{figure}[hpt]
    \begin{center}
        \includegraphics[scale=0.35]{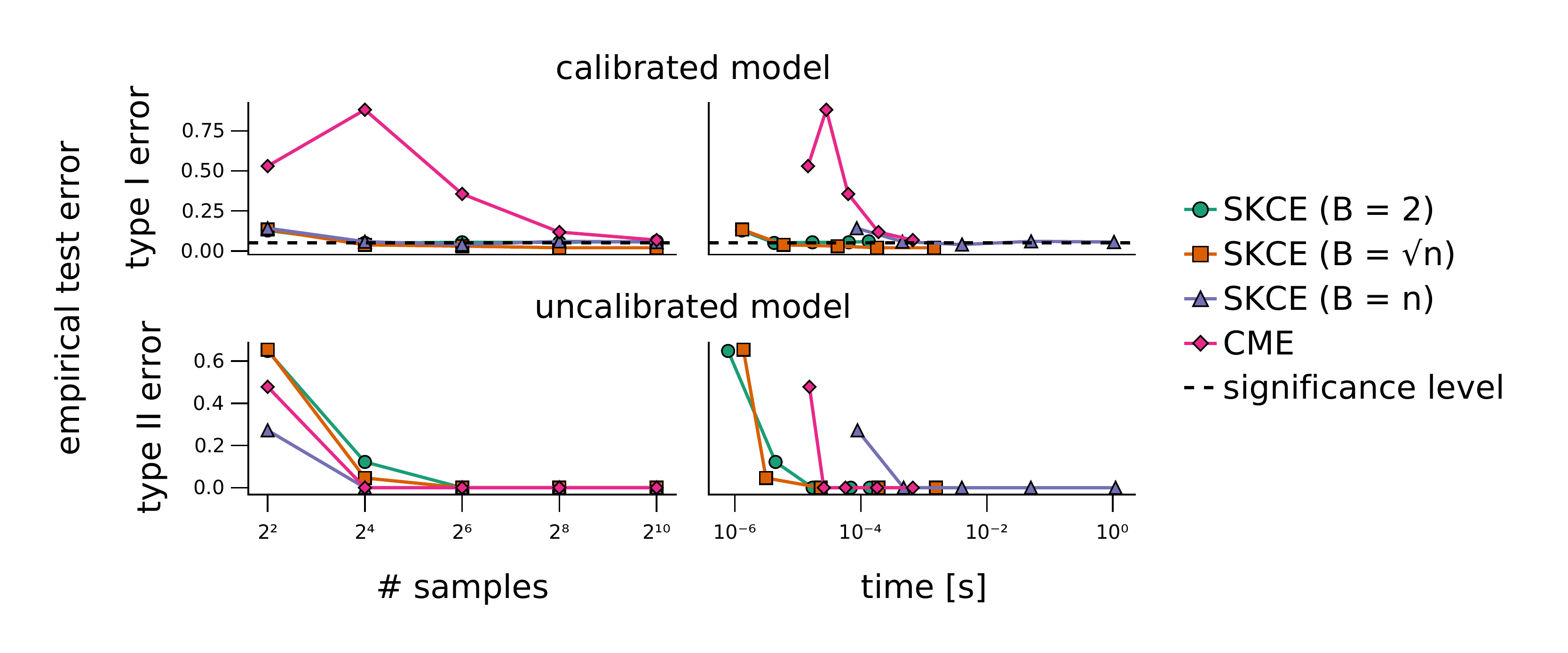}
        \caption{Empirical test errors for 500 data sets of $n \in \{4, 16, 64, 256, 1024\}$
        samples from models with targets of dimension $d = 1$. The dashed black line
        indicates the set signficance level $\alpha = 0.05$.}
        \label{fig:synthetic_tests_1}
    \end{center}
\end{figure}

\begin{figure}[hpt]
    \begin{center}
        \includegraphics[scale=0.35]{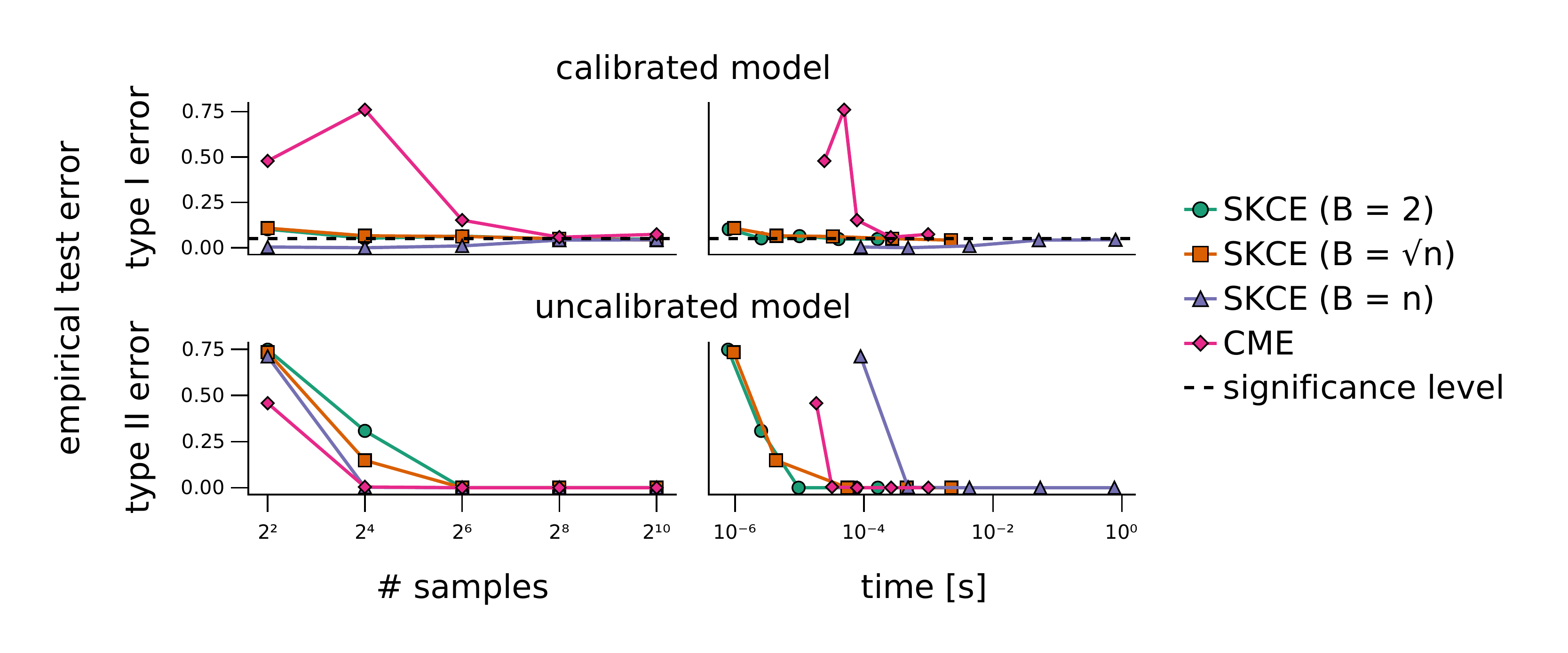}
        \caption{Empirical test errors for 500 data sets of $n \in \{4, 16, 64, 256, 1024\}$
        samples from models with targets of dimension $d = 10$. The dashed black line
        indicates the set signficance level $\alpha = 0.05$.}
        \label{fig:synthetic_tests_10}
    \end{center}
\end{figure}

\clearpage

\subsection{Friedman 1 regression problem}\label{app:friedman}

We study the so-called Friedman~1 regression problem, which was initially
described for 200 inputs in the six-dimensional unit
hypercube~\citep{Friedman1979,Friedman1983} and later modified to 100 inputs in
the 10-dimensional unit hypercube~\citep{Friedman1991}. In this regression
problem real-valued target $Y$ depends on input $X$ via
\begin{equation*}
    Y = 10 \sin{(\pi X_1 X_2)} + 20{(X_3 - 0.5)}^2 + 10 X_4 + 5 X_5 + \epsilon,
\end{equation*}
where noise $\epsilon$ is typically chosen to be independently standard
normally distributed. We generate a training data set of 100 inputs distributed
uniformly at random in the 10-dimensional unit hypercube and corresponding targets
with identically and independently distributed noise following a standard normal
distribution.

We consider models $P^{(\theta,\sigma^2)}$ of normal distributions with fixed variance
$\sigma^2$
\begin{equation*}
    P^{(\theta,\sigma^2)}_x = \mathcal{N}(f_{\theta}(x), \sigma^2),
\end{equation*}
where $f_{\theta}(x)$, the model of the mean of the distribution $\Prob(Y|X = x)$,
is given by a fully connected neural network with two hidden layers with 200 and 50
hidden units and ReLU activation functions. The parameters of the neural network
are denoted by $\theta$.

We use a maximum likelihood approach and train the parameters $\theta$ of the model
for 5000 iterations by minimizing the mean squared error on the training data set
using ADAM~\citep{Kingma2015} (default settings in the machine learning
framework Flux.jl~\citep{Innes2018a,Innes2018b}). In each iteration, the
variance $\sigma^2$ is set to the maximizer of the likelihood of the training data
set.

We train 10 models with different initializations of parameters $\theta$. The initial values
of the weight matrices of the neural networks are sampled from the uniform Glorot
initialization~\citep{Glorot2010} and the offset vectors are initialized with zeros.
In \cref{fig:friedman1}, we visualize estimates of accuracy and calibration measures
on the training and test data set with 100 and 50 samples, respectively, for
5000 training iterations. The pinball loss is a common measure and
training objective for calibration of quantiles~\citep{Song2019}. It is defined
as
\begin{equation*}
    \Exp_{X,Y} L_{\tau}\big(Y, \mathrm{quantile}(P_X, \tau)\big),
\end{equation*}
where $L_{\tau}(y, \tilde{y}) = (1 - \tau) (\tilde{y} - y)_{+} + \tau (y - \tilde{y})_{+}$
and $\mathrm{quantile}(P_x, \tau) = \inf_{y} \{P_x(Y \leq y) \geq \tau\}$ for
quantile level $\tau \in [0, 1]$.
In \cref{fig:friedman1} we plot the average pinball loss (pinball) for
quantile levels $\tau \in \{0.05, 0.1, \ldots, 0.95\}$. We evaluate
$\widehat{\mathrm{SKCE}}_{k,n}$ (SKCE (unbiased)) and $\widehat{\mathrm{SKCE}}_{k}$
(SKCE (biased)) for the tensor product kernel
\begin{equation*}
    \begin{split}
        k\big((p, y), (p', y')\big) &= \exp{\big(-W_2(p, p')\big)} \exp{\big(-(y - y')^2/2 \big)} \\
        &= \exp{\bigg(-\sqrt{(m_p - m_{p'})^2 + (\sigma_p - \sigma_{p'})^2}\bigg)} \exp{\big(- (y - y')^2/2\big)},
    \end{split}
\end{equation*}
where $W_2$ is the 2-Wasserstein distance and $m_p, m_{p'}$ and $\sigma_p, \sigma_{p'}$ denote
the mean and the standard deviation of the normal distributions $p$ and $p'$ (see~\cref{app:normal}).
The $p$-value estimate ($p$-value) is computed by estimating the quantile of the asymptotic distribution
of $n \widehat{\mathrm{SKCE}}_{k,n}$ with 1000 bootstrap samples
(see~\cref{remark:skceb_quadratic}). The estimates of the mean squared error
and the average negative log-likelihood are denoted by MSE and NLL.
All estimators indicate consistently that the trained models suffer from overfitting
after around 1000 training iterations.

Additionally, we form ensembles of the ten individual models at every training
iteration. The evaluations for the ensembles are visualized in \cref{fig:friedman1} as
well. Apart from the unbiased estimates of $\mathrm{SKCE}_k$, the estimates of the
ensembles are consistently better than the average estimates of the ensemble
members. For the mean squared error and the negative log-likelihood this behaviour
is guaranteed theoretically by the generalized mean inequality.

\begin{figure}[htp]
    \begin{center}
        \includegraphics[scale=0.35]{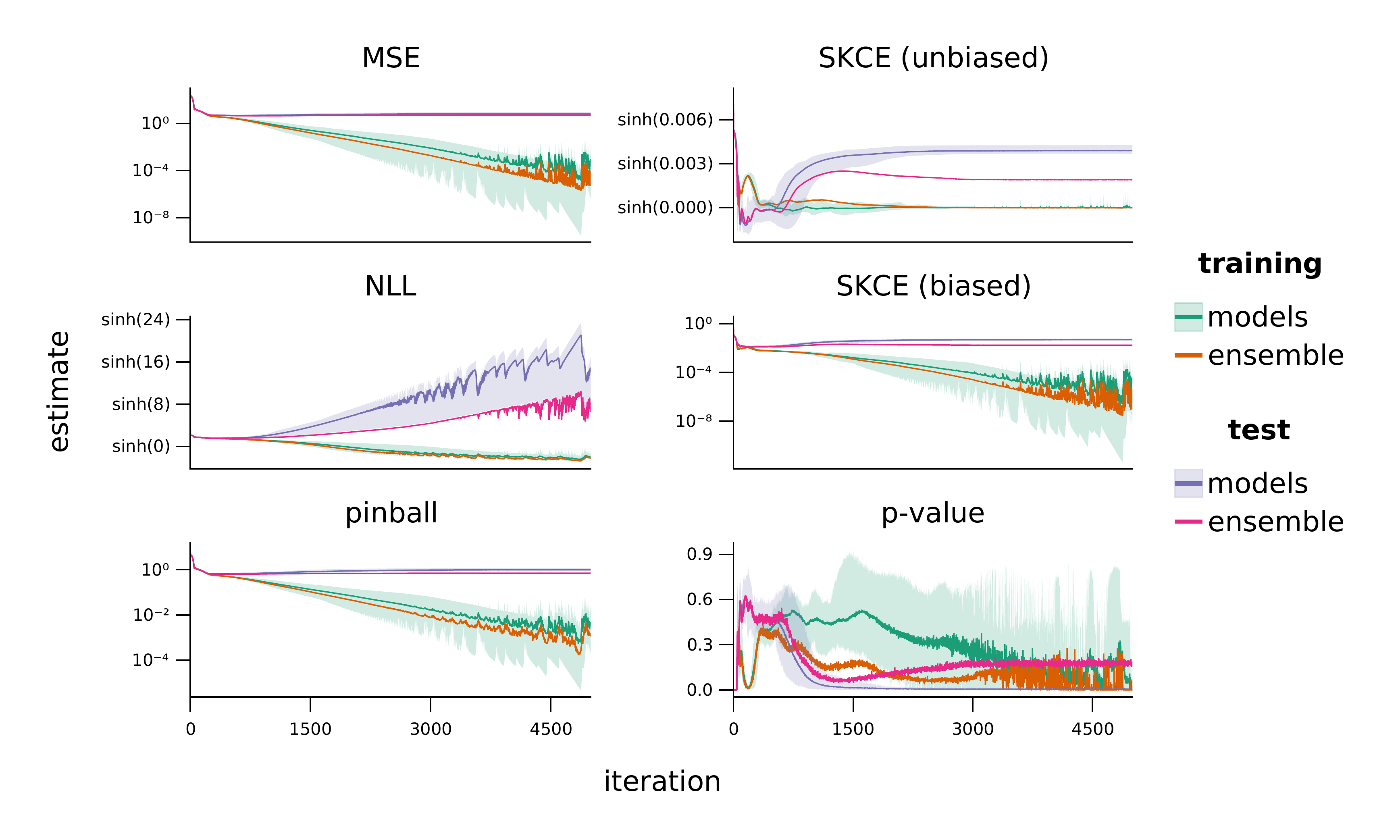}
        \caption{Estimates of different accuracy and calibration measures of ten
        Gaussian predictive models for the Friedman~1 regression problem versus
        the number of training iterations.
        Evaluations on the training data set (100 samples) are displayed in green
        and orange, and on the test data set (50 samples) in blue and purple.
        The green and blue line and their surrounding bands represent
        the mean and the range of the evaluations of the ten models. The orange and
        purple lines visualize the evaluations of their ensemble.}
        \label{fig:friedman1}
    \end{center}
\end{figure}

\section{Theory}
\label{app:theory}

\subsection{General setting}

Let $(\Omega, \mathcal{A}, \Prob)$ be a probability space. Define the random
variables $X \colon (\Omega, \mathcal{A}) \to (\mathcal{X}, \Sigma_X)$
and $Y \colon (\Omega, \mathcal{A}) \to (\mathcal{Y}, \Sigma_Y)$
such that $\Sigma_X$ contains all singletons, and denote a version of the regular
conditional distribution of $Y$ given $X = x$ by $\Prob(Y|X = x)$ for all
$x \in \mathcal{X}$.

Let $P \colon (\mathcal{X}, \Sigma_X) \to \big(\mathcal{P}, \mathcal{B}(\mathcal{P})\big)$
be a measurable function that maps features in $\mathcal{X}$ to
probability measures in $\mathcal{P}$ on the target space $\mathcal{Y}$.
We call $P$ a probabilistic model, and denote by $P_x \coloneqq P(x)$ its
output for feature $x \in \mathcal{X}$. This gives rise to the random
variable $P_X \colon (\Omega, \mathcal{A}) \to \big(\mathcal{P}, \mathcal{B}(\mathcal{P})\big)$
as $P_X \coloneqq P(X)$. We denote a version of the regular conditional distribution
of $Y$ given $P_X = P_x$ by $\Prob(Y| P_X = P_x)$ for all $P_x \in \mathcal{P}$.

\subsection{Expected and maximum calibration error}
\label{app:ece_mce}

The common definition of the expected and maximum calibration error~\citep{Naeini2015,Guo2017,Vaicenavicius2019,Kull2019}
for classification models can be generalized to arbitrary predictive
models.

\begin{definition}\label{def:ece_mce}
Let $d(\cdot, \cdot)$ be a distance measure of probability distributions
of target $Y$, and let $\mu$ be the law of $P_X$. Then we call
\begin{equation*}
    \mathrm{ECE}_d = \Exp d\big(\Prob(Y|P_X), P_X\big) \qquad \text{and} \qquad
    \mathrm{MCE}_d = \mu\text{-}\esssup d\big(\Prob(Y|P_X), P_X\big)
\end{equation*}
the expected calibration error~(ECE) and the maximum calibration error~(MCE)
of model $P$ with respect to measure $d$, respectively.
\end{definition}

\subsection{Kernel calibration error}

Recall the general notation: Let
$k \colon (\mathcal{P} \times \mathcal{Y}) \times (\mathcal{P} \times \mathcal{Y}) \to \mathbb{R}$
be a kernel, amd denote its corresponding RKHS by $\mathcal{H}$.

If not stated otherwise, we assume that
\begin{itemize}
    \item[(K1)] $k(\cdot, \cdot)$ is Borel-measurable.
    \item[(K2)] $k$ is integrable with respect to the distributions
    of $(P_X,Y)$ and $(P_X,Z_X)$, i.e.,
    \begin{equation*}
    \Exp_{P_X,Y} k^{1/2}\big((P_X, Y), (P_X, Y)\big) < \infty
    \end{equation*}
    and
    \begin{equation*}
    \Exp_{P_X,Z_X} k^{1/2}\big((P_X, Z_X), (P_X, Z_X)\big) < \infty.
    \end{equation*}
\end{itemize}

\begin{lemma}\label{lemma:mean_embedding}
There exist kernel mean embeddings
$\mu_{P_X Y}, \mu_{P_X Z_X} \in \mathcal{H}$
such that for all $f \in \mathcal{H}$
\begin{equation*}
    \langle f, \mu_{P_X Y} \rangle_{\mathcal{H}} = \Exp_{P_X,Y} f(P_X, Y) \qquad \text{and} \qquad    
    \langle f, \mu_{P_X Z_X} \rangle_{\mathcal{H}} = \Exp_{P_X,Z_X} f(P_X, Z_X).
\end{equation*}
This implies that
\begin{equation*}
    \mu_{P_X Y} = \Exp_{P_X,Y} k(\cdot, (P_X, Y)) \qquad \text{and} \qquad
    \mu_{P_X Z_X} = \Exp_{P_X,Z_X} k(\cdot, (P_X, Z_X)).
\end{equation*}
\end{lemma}

\begin{proof}
The linear operators $T_{P_X Y} f \coloneqq \Exp_{P_X, Y} f(P_X, Y)$
and $T_{P_X Z_X} f \coloneqq \Exp_{P_X, Z_X} f(P_X, Z_X)$
for all $f \in \mathcal{H}$ are bounded since
\begin{equation*}
    \begin{split}
        |T_{P_X Y} f| &= |\Exp_{P_X,Y} f(P_X, Y) | \leq \Exp_{P_X,Y} |f(P_X, Y)| = \Exp_{P_X,Y} |\langle k((P_X, Y), \cdot), f\rangle_{\mathcal{H}}| \\
        &\leq \Exp_{P_X,Y} \| k((P_X, Y), \cdot) \|_{\mathcal{H}} \|f\|_{\mathcal{H}}] = \|f\|_{\mathcal{H}} \Exp_{P_X,Y} k^{1/2}((P_X, Y), (P_X, Y))
    \end{split}
\end{equation*}
and similarly
\begin{equation*}
     |T_{P_X Z_X} f| \leq \|f\|_{\mathcal{H}} \Exp_{P_X,Z_X} k^{1/2}((P_X, Z_X), (P_X, Z_X)).
\end{equation*}
Thus Riesz representation theorem implies that there exist
$\mu_{P_X Y}, \mu_{P_X Z_X} \in \mathcal{H}$ such that
$T_{P_X Y} f = \langle f,  \mu_{P_X Y} \rangle_{\mathcal{H}}$ and
$T_{P_X Z_X} f = \langle f,  \mu_{P_X Z_X} \rangle_{\mathcal{H}}$. The reproducing
property of $\mathcal{H}$ implies
\begin{equation*}
    \mu_{P_X Y}(p, y) = \langle k((p, y), \cdot), \mu_{P_X Y} \rangle_{\mathcal{H}} = \Exp_{P_X,Y} k((p, y), (P_X, Y))
\end{equation*}
for all $(p, y) \in \mathcal{P} \times \mathcal{Y}$, and similarly
$ \mu_{P_X Z_X}(p, y) = \Exp_{P_X,Z_X} k((p, y), (P_X, Z_X))$.
\end{proof}

\begin{lemma}\label{lemma:skce}
The squared kernel calibration error~(SKCE) with respect to kernel $k$, defined
as $\mathrm{SKCE}_k \coloneqq \mathrm{KCE}_k^2$, is given by
\begin{equation*}
    \begin{split}
        \mathrm{SKCE}_k ={}& \Exp_{P_X,Y,P_{X'},Y'} k\big((P_X, Y), (P_{X'}, Y')\big)
        - 2 \Exp_{P_X,Y,P_{X'},Z_{X'}} k\big((P_X,Y),(P_{X'},Z_{X'})\big) \\
        &+ \Exp_{P_X,Z_X,P_{X'},Z_{X'}} k\big((P_X, Z_X), (P_{X'}, Z_{X'})\big),
    \end{split}
\end{equation*}
where $(P_{X'}, Y', Z_{X'})$ is independently distributed according to the law
of $(P_X, Y, Z_X)$
\end{lemma}

\begin{proof}
From \cref{lemma:mean_embedding} we know that there exist
kernel mean embeddings
$ \mu_{P_X Y},  \mu_{P_X Z_X} \in \mathcal{H}$ that satisfy
\begin{equation*}
    \begin{split}
        \langle f,  \mu_{P_X Y} -  \mu_{P_X Z_X} \rangle_{\mathcal{H}}
        &= \langle f,  \mu_{P_X Y} \rangle_{\mathcal{H}} - \langle f,  \mu_{P_X Z_X} \rangle_{\mathcal{H}} \\
        &= \Exp_{P_X,Y} f(P_X, Y) - \Exp_{P_X,Z_X} f(P_X, Z_X)
    \end{split}
\end{equation*}
for all $f \in \mathcal{H}$. Hence by the definition of the dual norm
\begin{equation*}
    \begin{split}
        \mathrm{CE}_{\mathcal{F}_k} &= \sup_{f \in \mathcal{F}_k} \big|\Exp_{P_X,Y} f(P_X, Y) - \Exp_{P_X,Z_X} f(P_X, Z_X) \big| \\
        &= \sup_{f \in \mathcal{F}_k} \big|\langle f,  \mu_{P_X,Y} -  \mu_{P_X,Z_X} \rangle_{\mathcal{H}}\big|
        = \|\mu_{P_X,Y} -  \mu_{P_X,Z_X}\|_{\mathcal{H}},
    \end{split}
\end{equation*}
which implies
\begin{equation*}
    \mathrm{SKCE}_k = \langle \mu_{P_X Y} -  \mu_{P_X Z_X}, \mu_{P_X Y} -  \mu_{P_X Z_X} \rangle_{\mathcal{H}}.
\end{equation*}
From \cref{lemma:mean_embedding} we obtain
\begin{equation*}
    \begin{split}
        \mathrm{SKCE}_k ={}& \Exp_{P_X,Y,P_{X'},Y'} k\big((P_X,Y), (P_{X'}, Y')\big) - 2 \Exp_{P_X,Y,P_{X'},Z_{X'}} k\big((P_X, Y), (P_{X'}, Z'_X)\big) \\
        &+ \Exp_{P_X,Z_X,P_{X'},Z'_X} k\big((P_X, Z_X), (P_{X'}, Z'_X)\big),
    \end{split}
\end{equation*}
which yields the desired result.
\end{proof}

Recall that $(P_{X_1}, Y_1), \ldots, (P_{X_n}, Y_n)$ is
a validation data set that is sampled i.i.d.\ according to
the law of $(P_X, Y)$ and that for all
$(p, y), (p', y') \in \mathcal{P} \times \mathcal{Y}$
\begin{multline*}
    h((p, y), (p', y')) \coloneqq k((p, y), (p', y')) - \Exp_{Z \sim p} k((p, Z), (p', y')) \\
    - \Exp_{Z' \sim p'} k((p, y), (p', Z')) + \Exp_{Z \sim p, Z' \sim p'} k((p, Z), (p', Z')).
\end{multline*}

\begin{lemma}\label{lemma:hbound}
    For all $i, j = 1, \ldots, n$,
    \begin{equation*}
        \big|h\big((P_{X_i}, Y_i), (P_{X_j}, Y_j)\big)\big| < \infty
    \end{equation*}
    almost surely.
\end{lemma}

\begin{proof}
Let $i, j \in \{1, \ldots, n\}$. By assumption (K2) we know
that
\begin{equation*}
    \big| k\big((P_{X_i}, Y_i), (P_{X_j}, Y_j)\big) \big| \leq
    k^{1/2}\big((P_{X_i}, Y_i), (P_{X_i}, Y_i)\big) k^{1/2}\big((P_{X_j}, Y_j), (P_{X_j}, Y_j)\big) < \infty
\end{equation*}
almost surely. Moreover,
\begin{multline*}
    \big|\Exp_{Z_{X_i}} k\big((P_{X_i}, Z_{X_i}), (P_{X_j}, Y_j)\big) \big|
    \leq \Exp_{Z_{X_i}} \big| k\big((P_{X_i}, Z_{X_i}), (P_{X_j}, Y_j)\big) \big| \\
    \leq \Exp_{Z_{X_i}} \bigg(k^{1/2}\big((P_{X_i}, Z_{X_i}), (P_{X_i}, Z_{X_i})\big) k^{1/2}\big((P_{X_j}, Y_j), (P_{X_j}, Y_j)\big) \bigg) < \infty
\end{multline*}
almost surely, and similarly $\big|\Exp_{Z_{X_i}, Z_{X_j}} k\big((P_{X_i}, Z_{X_i}), (P_{X_j}, Z_{X_j})\big) \big| < \infty$
almost surely. Thus
\begin{multline*}
    \big|h\big((P_{X_i}, Y_i), (P_{X_j}, Y_j)\big)\big| \leq \big| k\big((P_{X_i}, Y_i), (P_{X_j}, Y_j)\big) \big| + \big|\Exp_{Z_{X_i}} k\big((P_{X_i}, Z_{X_i}), (P_{X_j}, Y_j)\big) \big| \\
    + \big|\Exp_{Z_{X_j}} k\big((P_{X_i}, Y_i), (P_{X_j}, Z_{X_j})\big) \big| + \big|\Exp_{Z_{X_i}, Z_{X_j}} k\big((P_{X_i}, Z_{X_i}), (P_{X_j}, Z_{X_j})\big) \big| < \infty
\end{multline*}
almost surely.
\end{proof}

\lemmaskceb*

\begin{proof}
From \cref{lemma:skce} we know that $\mathrm{KCE}_k < \infty$,
and \cref{lemma:hbound} implies that
$\widehat{\mathrm{SKCE}}_k < \infty$ almost surely.

For $i = 1,\ldots, n$, the linear operators
$T_i f \coloneqq \Exp_{Z_{X_i}} f(P_{X_i}, Z_{X_i})$ for
$f \in \mathcal{H}$ are bounded almost surely since
\begin{equation*}
    \begin{split}
        |T_i f| &= \big|\Exp_{Z_{X_i}} f(P_{X_i}, Z_{X_i}) \big| \leq \Exp_{Z_{X_i}} \big|f(P_{X_i}, Z_{X_i})\big| = \Exp_{Z_{X_i}} \big| \langle k\big((P_{X_i}, Z_{X_i}), \cdot\big), f \rangle_{\mathcal{H}}\big| \\
        &\leq \Exp_{Z_{X_i}} \bigg(\big\|k\big((P_{X_i}, Z_{X_i}), \cdot\big) \big\|_{\mathcal{H}} \|f\|_{\mathcal{H}}\bigg) = \|f\|_{\mathcal{H}} \Exp_{Z_{X_i}} k^{1/2}\big((P_{X_i}, Z_{X_i}), (P_{X_i}, Z_{X_i})\big).
    \end{split}
\end{equation*}
Hence Riesz representation theorem implies that there exist
$\rho_i \in \mathcal{H}$ such that $T_i f = \langle f, \rho_i \rangle_{\mathcal{H}}$
almost surely. From the reproducing property of $\mathcal{H}$
we deduce that
$\rho_i(p, y) = \langle k\big((p, y), \cdot\big), \rho_i \rangle_{\mathcal{H}} = \Exp_{Z_{X_i}} k\big((p, y), (P_{X_i}, Z_{X_i})\big)$ for all $(p, y) \in \mathcal{P} \times \mathcal{Y}$
almost surely.

Thus by the definition of the dual norm the plug-in estimator
$\widehat{\mathrm{KCE}}_k$ satisfies
\begin{equation*}
    \begin{split}
        \widehat{\mathrm{KCE}}_k &= \sup_{f \in \mathcal{F}_k} \frac{1}{n} \Bigg| \sum_{i=1}^n \big(f(P_{X_i}, Y_i) - \Exp_{Z_{X_i}} f(P_{X_i}, Z_{X_i}) \big) \Bigg| \\
        &= \sup_{f \in \mathcal{F}_k} \frac{1}{n} \Bigg| \sum_{i=1}^n \big\langle k\big((P_{X_i}, Y_i), \cdot\big) - \rho_i, f \big\rangle_{\mathcal{H}} \Bigg| \\
        &= \sup_{f \in \mathcal{F}_k} \frac{1}{n} \Bigg| \bigg\langle \sum_{i=1}^n \big(k\big((P_{X_i}, Y_i), \cdot\big) - \rho_i\big), f \bigg\rangle_{\mathcal{H}} \Bigg|\\
        &= \frac{1}{n} \Bigg\| \sum_{i=1}^n \bigg(k\big((G_i, Y_i), \cdot\big) - \rho_i\bigg) \Bigg\|_{\mathcal{H}} \\
        &= \frac{1}{n} \Bigg(\bigg\langle \sum_{i=1}^n k\big((P_{X_i}, Y_i), \cdot\big) - \rho_i, \sum_{i=1}^n k\big((P_{X_i}, Y_i), \cdot\big) - \rho_i \bigg\rangle_{\mathcal{H}}\Bigg)^{1/2} \\
        &= \frac{1}{n} \Bigg(\sum_{i,j=1}^n h\big((P_{X_i}, Y_i), (P_{X_j}, Y_j)\big)\Bigg)^{1/2}
        = \widehat{\mathrm{SKCE}}^{1/2}_k < \infty
    \end{split}
\end{equation*}
almost surely, and hence indeed
$\widehat{\mathrm{SKCE}}_k^{1/2}$ is the plug-in
estimator of $\mathrm{KCE}_k$.

Since
$(P_X, Y), (P_{X'}, Y'), (P_{X_1}, Y_1), \ldots, (P_{X_n}, Y_n)$ are
identically distributed and pairwise independent, we obtain
\begin{equation}\label{eq:expected_skceb}
    \begin{split}
        n^2 \Exp \widehat{\mathrm{SKCE}}_k &= \sum_{\substack{i,j=1,\\i \neq j}}^n \Exp_{P_{X_i}, Y_i, P_{X_j}, Y_j} h\big((P_{X_i}, Y_i), (P_{X_j}, Y_j)\big) \\
        &\qquad +
        \sum_{i=1}^n \Exp_{P_{X_i}, Y_i} h\big((P_{X_i}, Y_i), (P_{X_i}, Y_i)\big) \\
        &= n(n-1) \Exp_{P_X,Y,P_{X'},Y'} h\big((P_X, Y), (P_{X'}, Y')\big) + n \Exp_{P_X, Y} h\big((P_X, Y), (P_X, Y)\big) \\
        &= n(n-1) \mathrm{SKCE}_k + n \Exp_{P_X, Y} h\big((P_X, Y), (P_X, Y)\big).
    \end{split}
\end{equation}

With the same reasoning as above, there exist $\rho, \rho' \in \mathcal{H}$ such
that for all $f \in \mathcal{H}$ $\Exp_{Z_X} f(P_X, Z_X) = \langle f, \rho \rangle_{\mathcal{H}}$
and $\mathbb{E}_{Z_{X'}} f(P_{X'}, Z_{X'}) = \langle f, \rho' \rangle_{\mathcal{H}}$
almost surely. Thus we obtain
\begin{equation*}
    h\big((P_X, Y), (P_{X'}, Y')\big)
    = \langle k\big((P_X, Y), \cdot\big) - \rho, k\big((P_{X'}, Y'), \cdot\big) - \rho' \rangle_{\mathcal{H}}
\end{equation*}
almost surely, and therefore by \cref{lemma:skce} and the
Cauchy-Schwarz inequality
\begin{equation*}
    \begin{split}
        \mathrm{SKCE}_k &= \Exp_{P_X,Y,P_{X'},Y'} h\big((P_X, Y), (P_{X'}, Y')\big) \\
        &= \Exp_{P_X,Y,P_{X'},Y'} \langle k\big((P_X, Y), \cdot\big) - \rho, k\big((G', Y'), \cdot\big) - \rho' \rangle_{\mathcal{H}} \\
        &\leq \Exp_{P_X,Y,P_{X'},Y'} \big|\langle k\big((P_X, Y), \cdot\big) - \rho, k\big((P_{X'}, Y'), \cdot\big) - \rho' \rangle_{\mathcal{H}}\big| \\
        &\leq \Exp_{P_X,Y,P_{X'},Y'} \big\|k\big((P_X, Y), \cdot\big) - \rho\big\|_{\mathcal{H}} \big\|k\big((P_{X'}, Y'), \cdot\big) - \rho' \big\|_{\mathcal{H}} \\
        &\leq \Exp_{P_X,Y}^{1/2} \big\|k\big((P_X, Y), \cdot\big) - \rho\big\|^2_{\mathcal{H}} \Exp_{P_{X'},Y'}^{1/2}\big\|k\big((P_{X'}, Y'), \cdot\big) - \rho' \big\|^2_{\mathcal{H}}.
\end{split}
\end{equation*}
Since $(P_X,Y)$ and $(P_{X'},Y')$ are identically distributed, we obtain
\begin{equation*}
    \begin{split}
        \mathrm{SKCE}_k &\leq \Exp_{P_X,Y} \big\|k\big((P_X, Y), \cdot\big) - \rho\big\|^2_{\mathcal{H}} = \Exp_{P_X,Y} h\big((P_X,Y),(P_X,Y)\big).
    \end{split}
\end{equation*}
Thus together with \cref{eq:expected_skceb} we get
\begin{equation*}
    n^2 \Exp \widehat{\mathrm{SKCE}}_k \geq n (n - 1) \mathrm{SKCE}_k + n \mathrm{SKCE}_k = n^2 \mathrm{SKCE}_k,
\end{equation*}
and hence $\widehat{\mathrm{SKCE}}_k$ has a non-negative bias.
\end{proof}

\lemmaskceblock*

\begin{proof}
From \cref{lemma:skce} we know that $\mathrm{SKCE}_k < \infty$,
and \cref{lemma:hbound} implies that
$\widehat{\mathrm{SKCE}}_{k,B} < \infty$ almost surely.

For $b \in \{1, \ldots, \lfloor n / B \rfloor\}$, let
\begin{equation}\label{eq:etab}
    \widehat{\eta}_b \coloneqq \binom{B}{2}^{-1} \sum_{(b - 1) B < i < j \leq bB} h\big((P_{X_{i}}, Y_i), (P_{X_j}, Y_j)\big)
\end{equation}
be the estimator of the $b$th block. From \cref{lemma:hbound} it follows
that $\widehat{\eta}_b < \infty$ almost surely for all $b$.
Moreover, for all $b$, $\widehat{\eta}_b$ is a so-called U-statistic of
$\mathrm{SKCE}_k$ and hence satisfies
$\Exp \widehat{\eta}_b = \mathrm{SKCE}_k$~\citep[see, e.g.,][]{Vaart1998}.
Since $(P_{X_1}, Y_1), \ldots, (P_{X_n}, Y_n)$ are pairwise independent,
this implies that $\widehat{\mathrm{SKCE}}_{k,B}$ is an unbiased
estimator of $\mathrm{SKCE}_k$.
\end{proof}

\subsection{Calibration tests}

\begin{lemma}\label{lemma:eta_variance}
    Let $B \in \{2, \ldots, n\}$.
    If $\Var_{P_X,Y,P_{X'},Y'} h\big((P_X, Y), (P_{X'}, Y')\big) < \infty$,
    then for all $b \in \{1, \ldots, \lfloor n / B \rfloor\}$
    \begin{equation*}
        \Var \widehat{\eta}_b = \sigma^2_B \coloneqq \binom{B}{2}^{-1} \Big(2(B - 2) \zeta_1
        + \Var_{P_X,Y,P_{X'},Y'} h\big((P_X, Y), (P_{X'}, Y')\big)\Big),
    \end{equation*}
    where $\widehat{\eta}_b$ is defined according to \cref{eq:etab} and
    \begin{equation}\label{eq:def_zeta1}
        \zeta_1 \coloneqq \Exp_{P_X,Y} \Exp^2_{P_{X'},Y'} h\big((P_X, Y), (P_{X'}, Y')\big) - \mathrm{SKCE}^2_k.
    \end{equation}

    If model $P$ is calibrated, it simplifies to
    \begin{equation*}
        \sigma^2_B = \binom{B}{2}^{-1} \Exp_{P_X,Y,P_{X'},Y'} h^2\big((P_X, Y), (P_{X'}, Y')\big).
    \end{equation*}
\end{lemma}

\begin{proof}
Let $b \in \{1, \ldots, \lfloor n / B \rfloor \}$. Since
$\Var_{P_X,Y,P_{X'},Y'} h\big((P_X, Y), (P_{X'}, Y')\big) < \infty$,
the Cauchy-Schwarz inequality implies $\Var \widehat{\eta}_b < \infty$
as well.

As mentioned in the proof of \cref{lemma:skceblock}
above, $\widehat{\eta}_b$ is a U-statistic of $\mathrm{SKCE}_k$.
From the general formula of the variance of a
U-statistic~\citep[see, e.g.,][p.~298--299]{Hoeffding1948}
we obtain
\begin{equation*}
    \begin{split}
        \Var \widehat{\eta}_b &= \binom{B}{2}^{-1} \bigg(\binom{2}{1} \binom{B-2}{2-1} \zeta_1
        + \binom{2}{2} \binom{B-2}{2-2} \Var_{P_X,Y,P_{X'},Y'} h\big((P_X, Y), (P_{X'}, Y')\big)\bigg) \\
        &= \binom{B}{2}^{-1} \Big(2(B-2) \zeta_1 + \Var_{P_X,Y,P_{X'},Y'} h\big((P_X, Y), (P_{X'}, Y')\big)\Big),
    \end{split}
\end{equation*}
where
\begin{equation*}
    \zeta_1 = \Exp_{P_X,Y} \Exp^2_{P_{X'},Y'} h\big((P_X, Y), (P_{X'}, Y')\big) - \mathrm{SKCE}^2_k.
\end{equation*}

If model $P$ is calibrated, then $(P_X, Y) \stackrel{d}{=} (P_X, Z)$,
and hence for all $(p, y) \in \mathcal{P} \times \mathcal{Y}$
\begin{equation*}
    \begin{split}
        \Exp_{P_X,Y} h\big((p, y), (P_X, Y)\big) ={}& \Exp_{P_X,Y} k\big((p, y), (P_X, Y)\big) - \Exp_{Z' \sim p} \Exp_{P_X,Y} k\big((p, Z'), (P_X, Y)\big) \\
        &- \Exp_{P_X,Z} k\big((p, y), (P_X, Z)\big) + \Exp_{Z' \sim p} \Exp_{P_X,Z} k\big((p, Z'), (P_X, Y)\big) \\
        ={}& 0.
    \end{split}
\end{equation*}
This implies $\zeta_1 = \Exp_{P_X,Y} \Exp^2_{P_{X'},Y'} h\big((P_X, Y), (P_{X'},Y')\big) = 0$
and $\mathrm{SKCE}^2_k = 0$ due to \cref{lemma:skce}. Thus
\begin{equation*}
    \sigma^2_B = \binom{B}{2}^{-1} \Exp_{P_X,Y,P_{X'},Y'} h^2\big((P_X, Y), (P_{X'}, Y')\big),
\end{equation*}
as stated above.
\end{proof}

\begin{corollary}\label{corr:skceblock_variance}
    Let $B \in \{2, \ldots, n\}$.
    If $\Var_{P_X,Y,P_{X'},Y'} h\big((P_X, Y), (P_{X'}, Y')\big) < \infty$,
    then
    \begin{equation*}
        \Var \widehat{SKCE}_{k,B} = \lfloor n / B \rfloor^{-1} \sigma^2_B.
    \end{equation*}
    where $\sigma^2_B$ is defined according to \cref{lemma:eta_variance}.
\end{corollary}

\begin{proof}
Since the estimators $\widehat{\eta}_1, \ldots, \widehat{\eta}_{\lfloor n/B \rfloor}$
in each block are pairwise independent, this is an immediate consequence of
\cref{lemma:eta_variance}.
\end{proof}

\begin{corollary}\label{corr:skceblock_clt}
    Let $B \in \{2, \ldots, n\}$.
    If $\Var_{P_X,Y,P_{X'},Y'} h\big((P_X, Y), (P_{X'}, Y')\big) < \infty$,
    then
    \begin{equation*}
        \sqrt{\lfloor n / B \rfloor} \big(\widehat{\mathrm{SKCE}}_{k,B} - \mathrm{SKCE}_k\big) \xrightarrow{d} \mathcal{N}\big(0, \sigma^2_B\big) \qquad \text{as } n \to \infty,
    \end{equation*}
    where block size $B$ is fixed and $\sigma^2_B$ is
    defined according to \cref{lemma:eta_variance}.
\end{corollary}

\begin{proof}
    The result follows from \cref{lemma:skceblock}, \cref{lemma:eta_variance},
    and the central limit theorem~\citep[see, e.g.,][Theorem~A in Section~1.9]{Serfling1980}.
\end{proof}

\begin{remark}\label{remark:skceb_fixed}
    \Cref{corr:skceblock_clt} shows that $\widehat{\mathrm{SKCE}}_{k,B}$ is a consistent
    estimator of $\mathrm{SKCE}_k$ in the large sample limit as $n \to \infty$ with
    fixed number $B$ of samples per block. In particular, for the linear estimator with
    $B = 2$ we obtain
    \begin{equation*}
        \sqrt{\lfloor n / 2 \rfloor} \big(\widehat{\mathrm{SKCE}}_{k,2} - \mathrm{SKCE}_k\big) \xrightarrow{d} \mathcal{N}\big(0, \sigma^2_2\big) \qquad \text{as } n \to \infty.
    \end{equation*}

    Moreover, \cref{lemma:eta_variance,corr:skceblock_clt}
    show that the $p$-value of the null hypothesis that model $P$ is calibrated can be
    estimated by
    \begin{equation*}
        \Phi\bigg(-\frac{ \sqrt{\lfloor n / B \rfloor} \widehat{\mathrm{SKCE}}_{k,B}}{\widehat{\sigma}_B} \bigg),
    \end{equation*}
    where $\Phi$ is the cumulative distribution function of the standard normal
    distribution and $\widehat{\sigma}_B$ is the empirical standard deviation of the
    block estimates $\widehat{\eta}_1,\ldots, \widehat{\eta}_{\lfloor n / B \rfloor}$,
    and
    \begin{equation*}
        \Phi\bigg(- \frac{\sqrt{\lfloor n / B \rfloor B (B-1)}\widehat{\mathrm{SKCE}}_{k,B}}{\sqrt{2}\widehat{\sigma}}\bigg),
    \end{equation*}
    where $\widehat{\sigma}^2$ is an estimate of $\Exp_{P_X,Y,P_{X'},Y'} h^2\big((P_X,Y), (P_{X'},Y')\big)$.
    Similar $p$-value approximations for the two-sample test with blocks of fixed
    size were used by \citet{Chwialkowski2015}.
\end{remark}

\begin{corollary}\label{corr:eta_limit}
Assume $\Var_{P_X,Y,P_{X'},Y'} h\big((P_X, Y), (P_{X'}, Y')\big) < \infty$.
Let $s \in \{1, \ldots, \lfloor n / 2 \rfloor \}$. Then for all
$b \in \{1, \ldots, s\}$
\begin{equation}\label{eq:etab_asymptotics_general}
    \sqrt{B} \big(\widehat{\eta}_{b} - \mathrm{SKCE}_k\big)
    \xrightarrow{d} \mathcal{N}(0, 4 \zeta_1) \qquad \text{as } B \to \infty,
\end{equation}
where $\widehat{\eta}_b$ is defined according to \cref{eq:etab} with
$n = B s$, the number $s$ of equally-sized blocks is fixed,
and $\zeta_1$ is defined according to \cref{eq:def_zeta1}.

If model $P$ is calibrated, then
$\sqrt{B}\big(\widehat{\eta}_{b} - \mathrm{SKCE}_k\big)
= \sqrt{B} \widehat{\eta}_{b}$
is
asymptotically tight since $\zeta_1 = 0$, and
\begin{equation}\label{eq:etab_asymptotics_calibrated}
    B \widehat{\eta}_{b} \xrightarrow{d} \sum_{i=1}^\infty \lambda_i (Z_i - 1) \qquad \text{as } B \to \infty,
\end{equation}
where $Z_i$ are independent $\chi^2_1$ distributed random variables
and $\lambda_i \in \mathbb{R}$ are eigenvalues of the Hilbert-Schmidt integral
operator
\begin{equation*}
    K f(p, y) \coloneqq \Exp_{P_X,Y}\big(h((p, y), (P_X, Y)) f(P_X, Y)\big)
\end{equation*}
for Borel-measurable functions
$f \colon \mathcal{P} \times \mathcal{Y} \to \mathbb{R}$
with $\Exp_{P_X,Y} f^2(P_X, Y) < \infty$.
\end{corollary}

\begin{proof}
Let $s \in \{1, \ldots, \lfloor n / 2 \rfloor \}$ and $b \in \{1, \ldots, s\}$.
As mentioned above in the proof of \cref{lemma:skceblock}, the estimator
$\widehat{\eta}_b$, defined according to \cref{eq:etab}, is a so-called
U-statistic of $\mathrm{SKCE}_k$~\citep[see, e.g.,][]{Vaart1998}. Thus
\cref{eq:etab_asymptotics_general} follows from the asymptotic behaviour
of U-statistics~\citep[see, e.g.,][Theorem~12.3]{Vaart1998}.

If $P$ is calibrated, then we know from the proof of \cref{lemma:eta_variance}
that $\zeta_1 = 0$, and hence $\widehat{\eta}_b$ is a so-called degenerate-
U-statistic~\citep[see, e.g.,][Section~12.3]{Vaart1998}. From the theory
of degenerate U-statistics it follows that the sequence
$B \widehat{\eta}_b$ converges in distribution
to the limit distribution in \cref{eq:etab_asymptotics_calibrated},
which is known as Gaussian chaos.
\end{proof}

\begin{corollary}\label{corr:skceb_limit}
Assume $\Var_{P_X,Y,P_{X'},Y'} h\big((P_X, Y), (P_{X'}, Y')\big) < \infty$.
Let $s \in \{1, \ldots, \lfloor n / 2 \rfloor \}$. Then
\begin{equation*}
    \sqrt{B} \big(\widehat{\mathrm{SKCE}}_{k,B} - \mathrm{SKCE}_k\big)
    \xrightarrow{d} \mathcal{N}(0, 4 s^{-1} \zeta_1) \qquad \text{as } B \to \infty,
\end{equation*}
where the number $s$ of equally-sized blocks is fixed, $n = Bs$, 
and $\zeta_1$ is defined according to \cref{eq:def_zeta1}.

If model $P$ is calibrated, then
$\sqrt{B}\big(\widehat{\mathrm{SKCE}}_{k,B} - \mathrm{SKCE}_k\big)
= \sqrt{B} \widehat{\mathrm{SKCE}}_{k,B}$ is
asymptotically tight since $\zeta_1 = 0$, and
\begin{equation*}
    B \widehat{\mathrm{SKCE}}_{k,B} \xrightarrow{d} s^{-1} \sum_{i=1}^\infty \lambda_i (Z_{i} - s) \qquad \text{as } B \to \infty,
\end{equation*}
where $Z_{i} $ are independent $\chi^2_s$ distributed random variables
and $\lambda_i \in \mathbb{R}$ are eigenvalues of the Hilbert-Schmidt integral
operator
\begin{equation*}
    K f(p, y) \coloneqq \Exp_{P_X,Y}\big(h((p, y), (P_X, Y)) f(P_X, Y)\big)
\end{equation*}
for Borel-measurable functions
$f \colon \mathcal{P} \times \mathcal{Y} \to \mathbb{R}$
with $\Exp_{P_X,Y} f^2(P_X, Y) < \infty$.
\end{corollary}

\begin{proof}
Since the estimators $\widehat{\eta}_1, \ldots, \widehat{\eta}_{s}$ in each block are
pairwise independent, this is an immediate consequence of \cref{corr:eta_limit}.
\end{proof}

\begin{remark}\label{remark:skceb_quadratic}
\Cref{corr:skceb_limit} shows that $\widehat{\mathrm{SKCE}}_{k,B}$ is a consistent
estimator of $\mathrm{SKCE}_k$ in the large sample limit as $B \to \infty$ with
fixed number $\lfloor n / B \rfloor$ of blocks. Moreover, for the minimum variance
unbiased estimator with $B = n$, \cref{corr:skceb_limit}
shows that under the null hypothesis that model $P$ is calibrated
\begin{equation*}
    n \widehat{\mathrm{SKCE}}_{k,n} \xrightarrow{d} \sum_{i=1}^\infty \lambda_i (Z_{i} - 1) \qquad \text{as } n \to \infty,
\end{equation*}
where $Z_i$ are independent $\chi^2_1$ distributed random variables. Unfortunately
quantiles of the limit distribution of $\sum_{i=1}^\infty \lambda_i (Z_{i} - 1)$
(and hence the $p$-value of the null hypothesis that model $P$ is calibrated) can not
be computed analytically but have to be estimated by, e.g.,
bootstrapping~\citep{Arcones1992}, using a Gram matrix spectrum~\citep{Gretton2009},
fitting Pearson curves~\citep{Gretton2007}, or using a Gamma
approximation~\citep[p.~343, p.~359]{Johnson1994}.

\end{remark}

\begin{corollary}\label{corr:skceb_clt}
Assume $\Var_{P_X,Y,P_{X'},Y'} h\big((P_X, Y), (P_{X'}, Y')\big) < \infty$.
Then
\begin{equation}\label{eq:skceb_asymptotics_general}
    \sqrt{\lfloor n / B \rfloor B} \big(\widehat{\mathrm{SKCE}}_{k,B} - \mathrm{SKCE}_k\big)
    \xrightarrow{d} \mathcal{N}(0, 4 \zeta_1) \qquad \text{as } B \to \infty \text{ and } \lfloor n / B \rfloor \to \infty,
\end{equation}
where $B$ is the block size and $s$ is the number of equally-sized blocks, $n = Bs$,
and $\zeta_1$ is defined according to \cref{eq:def_zeta1}.

If model $P$ is calibrated, then
$\sqrt{\lfloor n / B \rfloor B}\big(\widehat{\mathrm{SKCE}}_{k,B} - \mathrm{SKCE}_k\big)
= \sqrt{\lfloor n / B \rfloor B} \widehat{\mathrm{SKCE}}_{k,B}$
is
asymptotically tight since $\zeta_1 = 0$, and
\begin{equation*}
    \sqrt{\lfloor n / B \rfloor} B\widehat{\mathrm{SKCE}}_{k,B} \xrightarrow{d} \mathcal{N}\bigg(0, \sum_{i=1}^\infty \lambda_i^2\bigg) \qquad \text{as } B \to \infty \text{ and } \lfloor n / B \rfloor \to \infty,
\end{equation*}
where $\lambda_i \in \mathbb{R}$ are eigenvalues of the Hilbert-Schmidt integral
operator
\begin{equation*}
    K f(p, y) \coloneqq \Exp_{P_X,Y}\big(h((p, y), (P_X, Y)) f(P_X, Y)\big)
\end{equation*}
for Borel-measurable functions
$f \colon \mathcal{P} \times \mathcal{Y} \to \mathbb{R}$
with $\Exp_{P_X,Y} f^2(P_X, Y) < \infty$.
\end{corollary}

\begin{proof}
    The result follows from \cref{corr:eta_limit}
    and the central limit theorem~\citep[see, e.g.,][Theorem~A in Section~1.9]{Serfling1980}.    
\end{proof}

\begin{remark}\label{remark:skceb_increasing}
    \Cref{corr:skceb_clt} shows that $\widehat{\mathrm{SKCE}}_{k,B}$ is a consistent estimator
    of $\mathrm{SKCE}_k$ in the large sample limit as $B \to \infty$ and
    $\lfloor n / B \rfloor \to \infty$, i.e., as both the number of samples per block and the
    number of blocks go to infinity. Moreover, \cref{corr:eta_limit,corr:skceb_clt} show that
    the $p$-value of the null hypothesis that $P$ is calibrated can be estimated by
    \begin{equation*}
        \Phi\bigg(- \frac{\sqrt{\lfloor n / B \rfloor}\widehat{\mathrm{SKCE}}_{k,B}}{\widehat{\sigma}_B} \bigg),
    \end{equation*}
    where $\widehat{\sigma}_B$ is the empirical standard deviation of the
    block estimates $\widehat{\eta}_1,\ldots, \widehat{\eta}_{\lfloor n / B \rfloor}$.
    Similar $p$-value approximations for the two-sample problem with blocks of increasing
    size were proposed and applied by \citet{Zaremba2013}.
\end{remark}

\section{Calibration mean embedding}
\label{app:ucme}

\subsection{Definition}

Similar to the unnormalized mean embedding~(UME) proposed by
\citet{Chwialkowski2015} in the standard MMD setting, instead of the
calibration error
$\mathrm{CE}_{\mathcal{F}_k} = \|\mu_{P_X Y} - \mu_{P_X Z_X}\|_{\mathcal{H}}$
we can consider the unnormalized calibration mean embedding~(UCME).

\begin{definition}\label{def:ucme}
Let $J \in \mathbb{N}$. The unnormalized calibration mean embedding~(UCME)
for kernel $k$ with $J$ test locations is defined
as the random variable
\begin{equation*}
    \begin{split}
        \mathrm{UCME}^2_{k,J} &= J^{-1} \sum_{j = 1}^J \big(\mu_{P_X Y}(T_j) - \mu_{P_X Z_X}(T_j)\big)^2 \\
        &= J^{-1} \sum_{j=1}^J \big(\Exp_{P_X,Y} k(T_j, (P_X, Y)) - \Exp_{P_X,Z_X} k(T_j, (P_X, Z_X))\big)^2,
    \end{split}
\end{equation*}
where $T_1, \ldots, T_J$ are i.i.d.\ random variables (so-called
test locations) whose distribution is absolutely continuous with
respect to the Lebesgue measure on $\mathcal{P} \times \mathcal{Y}$.
\end{definition}

As mentioned above, in many machine learning applications we actually
have $\mathcal{P} \times \mathcal{Y} \subset \mathbb{R}^d$ (up to
some isomorphism). In such a case, if $k$ is an analytic, integrable,
characteristic kernel, then for each $J \in \mathbb{N}$
$\mathrm{UCME}_{k,J}$ is a random metric between the distributions of
$(P_X, Y)$ and $(P_X, Z_X)$, as shown by
\citet[Theorem~2]{Chwialkowski2015}. In particular, this implies that
$\mathrm{UCME}_{k,J} = 0$ almost surely if and only if the two
distributions are equal.

\subsection{Estimation}

Again we assume $(P_{X_1}, Y_1), \ldots, (P_{X_n}, Y_n)$ is a 
validation data set of predictions and targets, which are i.i.d.\
according to the law of $(P_X, Y)$. The consistent, but biased,
plug-in estimator of $\mathrm{UCME}^2_{k,J}$ is given by
\begin{equation*}
     \widehat{\mathrm{UCME}}^2_{k,J} = {J}^{-1} \sum_{j=1}^J {\Bigg(n^{-1} \sum_{i=1}^n \Big( k\big(T_j, (P_{X_i}, Y_i)\big) - \Exp_{Z_{X_i}} k\big(T_j, (P_{X_i}, Z_{X_i})\big)\Big) \Bigg)}^2.
\end{equation*}

\subsection{Calibration mean embedding test}

As \citet{Chwialkowski2015} note, if model $P$ is calibrated,
for every fixed sequence of unique test locations
$\sqrt{n} \widehat{\mathrm{UCME}}^2_{k,J}$ converges in distribution
to a sum of correlated $\chi^2$ random variables, as $n \to \infty$.
The estimation of this asymptotic distribution, and its quantiles
required for hypothesis testing, requires a bootstrap or permutation
procedure, which is computationally expensive. Hence
\citet{Chwialkowski2015} proposed the following test based on
Hotelling's $T^2$-statistic~\citep{Hotelling1931}.

For $i = 1, \ldots, n$, let
\begin{equation*}
    Z_i \coloneqq \begin{pmatrix} k\big(T_1, (P_{X_i}, Y_i)\big) - \Exp_{Z_{X_i}} k\big(T_1, (P_{X_i}, Z_{X_i})\big) \\ \vdots \\ k\big(T_J, (P_{X_i}, Y_i)\big) - \Exp_{Z_{X_i}} k\big(T_J, (P_{X_i}, Z_{X_i})\big) \end{pmatrix} \in \mathbb{R}^J,
\end{equation*}
and denote the empirical mean and covariance matrix of
$Z_1, \ldots, Z_n$ by $\overline{Z}$ and $S$, respectively.
If $\mathrm{UCME}_{k,J}$ is a random metric between the distributions
of $(P_X, Y)$ and $(P_X, Z_X)$, then the test statistic
\begin{equation*}
    Q_n \coloneqq n \overline{Z}^T S^{-1} \overline{Z}
\end{equation*}
is almost surely asymptotically $\chi^2$ distributed with
$J$ degrees of freedom if model $P$ is calibrated, as
$n \to \infty$ with $J$
fixed; moreover, if model $P$ is uncalibrated, then for
any fixed $r \in \mathbb{R}$ almost surely $\Prob(Q_n > r) \to 1$ as $n \to \infty$~\citep[Proposition~2]{Chwialkowski2015}.
We call the resulting calibration test calibration mean
embedding (CME) test.

\section{Kernel choice}
\label{app:kernel}

A natural choice for the kernel
$k \colon (\mathcal{P} \times \mathcal{Y}) \times (\mathcal{P} \times \mathcal{Y}) \to \mathbb{R}$
on the product space of predicted distributions $\mathcal{P}$ and targets
$\mathcal{Y}$ is a tensor product kernel of the form $k = k_{\mathcal{P}} \otimes k_{\mathcal{Y}}$, i.e.,
a kernel of the form
\begin{equation*}
    k\big((p, y), (p', y')\big) = k_{\mathcal{P}}(p, p') k_{\mathcal{Y}}(y, y'),
\end{equation*}
where $k_{\mathcal{P}}\colon \mathcal{P} \times \mathcal{P} \to \mathbb{R}$ and
$k_{\mathcal{Y}} \colon \mathcal{Y} \times \mathcal{Y} \to \mathbb{R}$ are kernels
on the spaces of predicted distributions and targets, respectively.

As discussed in \cref{sec:kernel_choice}, if kernel $k$ is characteristic,
then the kernel calibration error $\mathrm{KCE}_k$ of model $P$ is zero
if and only if $P$ is calibrated.
Unfortunately, as shown by~\citet[Example~1]{Szabo2018}, for kernels $k_{\mathcal{P}}$ and $k_{\mathcal{Y}}$ their characteristic property is necessary but generally not sufficient for the tensor product kernel $k = k_{\mathcal{P}} \otimes k_{\mathcal{Y}}$ to be characteristic.
Hence in general stronger requirements on $k_{\mathcal{P}}$ and $k_{\mathcal{Y}}$ are needed.
For instance, if $k_{\mathcal{P}}$ and $k_{\mathcal{Y}}$ are characteristic, continuous, bounded, and translation-invariant kernels on $\mathbb{R}^{d_i}$ ($i = 1,2$), then $k_{\mathcal{P}} \otimes k_{\mathcal{Y}}$ is characteristic~\citep[Theorem~4]{Szabo2018}.
More generally, if $k_{\mathcal{P}}$ and $k_{\mathcal{Y}}$ are universal kernels on locally compact Polish spaces, then $k_{\mathcal{P}} \otimes k_{\mathcal{Y}}$ is characteristic~\citep[Theorem~5]{Szabo2018}.
In classification even the reverse implication holds, and $k_{\mathcal{P}} \otimes k_{\mathcal{Y}}$ is characteristic if and only if $k_{\mathcal{P}}$ and $k_{\mathcal{Y}}$ are universal~\citep[Corollary~3.15]{Steinwart2021}.

Many common kernels such as the Gaussian and Laplacian kernel on $\mathbb{R}^d$ are universal (and characteristic), and thus are an appropriate choice for kernel $k_{\mathcal{Y}}$ for real-valued target spaces.
In classification, kernel $k_{\mathcal{Y}}$ is universal if and only if it is strictly positive definite~\citep[Section~3.3]{Sriperumbudur2011} (note that target space $\mathcal{Y}$ is finite).

The choice of $k_{\mathcal{P}}$ might be less obvious
since $\mathcal{P}$ is a space of probability distributions. Intuitively
one might want to generalize Gaussian and Laplacian kernels and use kernels of the form
\begin{equation}\label{eq:kernelform}
    k_{\mathcal{P}}\big(p, p'\big) = \exp{\big(- \lambda d^\nu_{\mathcal{P}}(p, p')\big)},
\end{equation}
where $d_{\mathcal{P}} \colon \mathcal{P} \times \mathcal{P} \to \mathbb{R}$ is a
metric on $\mathcal{P}$ and $\nu, \lambda > 0$ are kernel hyperparameters.

Unfortunately, this construction does not necessarily yield valid kernels.
Most prominently, the Wasserstein distance does not lead to valid kernels $k_{\mathcal{P}}$ in general~\citep[Chapter~8.3]{Peyre2018}. However, if $d_{\mathcal{P}}(\cdot, \cdot)$
is a Hilbertian metric, i.e., a metric of the form
\begin{equation*}
    d_{\mathcal{P}}(p, p') = \big\|\phi(p) - \phi(p') \big\|_{H}
\end{equation*}
for some Hilbert space $H$ and mapping $\phi \colon \mathcal{P} \to H$, then
$k_{\mathcal{P}}$ in \cref{eq:kernelform} is a valid kernel for all $\lambda > 0$ and
$\nu \in (0, 2]$~\citep[Corollary~3.3.3, Proposition~3.2.7]{Berg1984}.

Non-constant radial kernels on $\mathbb{R}^d$ are universal~\citep[Theorem~17]{Micchelli2006}, and hence, with a similar reasoning as in the proof of \citet[Theorem~2.2]{Christmann2010}, $k_{\mathcal{P}}$ is universal on the metric space $(\mathcal{P}, d_{\mathcal{P}})$ for all $\lambda > 0$ and $\nu \in (0, 2]$ if $\phi$ is an injective embedding and $H = \mathbb{R}^d$ for some $d$.
More generally, if $\phi$ is injective and $H$ is a separable Hilbert space, $k_{\mathcal{P}}$ is known to be universal on $(\mathcal{P}, d_{\mathcal{P}})$ for $\lambda > 0$ and $\nu = 2$~\citep[Theorem~2.2]{Christmann2010}.
Note that in both cases it is required that $\phi$ is continuous.
Continuity of $\phi$ is implied by the choice of space $(\mathcal{P}, d_{\mathcal{P}})$ but has to be shown separately if $\mathcal{P}$ is equipped with a different metric.

In many machine learning applications such a mapping $\phi$ arises naturally from the parameterization of the predicted distributions.
For instance, if a regression model predicts normal distributions of real targets, i.e., if $\mathcal{Y} = \mathbb{R}$ and $\mathcal{P} = \{ \mathcal{N}(\mu, \sigma^2) \colon \mu \in \mathbb{R}, \sigma \in \mathbb{R}_{\geq 0} \}$, one may use $H = \mathbb{R}^2$ and define
\begin{equation*}
    \phi\big(\mathcal{N}(\mu, \sigma^2)\big) = [\mu, \sigma]^\mathsf{T}.
\end{equation*}

\subsection{Normal distributions}
\label{app:normal}

Assume that $\mathcal{Y} = \mathbb{R}^d$ and
$\mathcal{P} = \{\mathcal{N}(\mu, \Sigma) \colon \mu \in \mathbb{R}^d, \Sigma \in \mathbb{R}^{d \times d} \text{ psd}\}$,
i.e., the model outputs normal distributions
$P_X = \mathcal{N}(\mu_X, \Sigma_X)$.
The distribution of these outputs is defined by the distribution
of their mean $\mu_X$ and covariance matrix $\Sigma_X$.

Let $P_x = \mathcal{N}(\mu_x, \Sigma_x) \in \mathcal{P}$,
$y \in \mathcal{Y} = \mathbb{R}^d$, and $\gamma > 0$. We obtain
\begin{multline*}
    \Exp_{Z_x \sim P_x} \exp{\Big(- \gamma {\|Z_x - y\|}^2_2\Big)} \\
    = {\big|\mathbf{I}_{d} + 2\gamma \Sigma_x \big|}^{-1/2} \exp{\Big(-\gamma {(\mu_x - y)}^{\mathsf{T}} {\big(\mathbf{I}_d + 2\gamma \Sigma_x
    \big)}^{-1} {(\mu_x - y)}\Big)}
\end{multline*}
from \citet[Theorem~3.2.a.3]{mathai1992}. In particular, if
$\Sigma_x = \mathrm{diag}{(\Sigma_{x,1}, \ldots, \Sigma_{x,d})}$, then
\begin{multline*}
    \Exp_{Z_x \sim P_x} \exp{\Big(- \gamma {\|Z_x - y\|}^2_2\Big)} \\
    = \prod_{i=1}^{d} \bigg[{\big(1 + 2\gamma \Sigma_{x,i}\big)}^{-1/2} \exp{\Big(-\gamma {\big(1 + 2\gamma \Sigma_{x,i}\big)}^{-1} {\big(\mu_{x,i} - y_i\big)}^2 \Big)}\bigg].
\end{multline*}

Let $P_{x'} = \mathcal{N}(\mu_{x'}, \Sigma_{x'})$ be another
normal distribution. Then we have
\begin{equation*}
    \begin{split}
        &\Exp_{Z_x \sim P_x, Z_{x'} \sim P_{x'}} \exp{\Big(- \gamma {\|Z_x - Z_{x'}\|}^2_2\Big)} \\
        &\, = \big|\mathbf{I}_{d} + 2\gamma \Sigma_x\big|^{-1/2} \Exp_{Z_{x'} \sim P_{x'}} \exp{\Big(-\gamma {\big(\mu_x - Z_{x'}\big)}^\mathsf{T} {\big(\mathbf{I}_{d} + 2\gamma \Sigma_x\big)}^{-1} {\big(\mu_x - Z_{x'}\big)}\Big)} \\
        &\, = \big|\mathbf{I}_{d} + 2\gamma (\Sigma_x + \Sigma_{x'})\big|^{-1/2} \exp{\Big(-\gamma {\big(\mu_x - \mu_{x'}\big)}^\mathsf{T} {\big(\mathbf{I}_{d} + 2\gamma(\Sigma_x + \Sigma_{x'})\big)}^{-1} {\big(\mu_x - \mu_{x'}\big)}\Big)}.
    \end{split}
\end{equation*}
Thus if $\Sigma_x = \mathrm{diag}{(\Sigma_{x,1}, \ldots, \Sigma_{x,d})}$
and $\Sigma_{x'} = \mathrm{diag}{\big(\Sigma_{x',1}, \ldots, \Sigma_{x',d}\big)}$, then
\begin{multline*}
    \Exp_{Z_x \sim P_x, Z_{x'} \sim P_{x'}} \exp{\Big(- \gamma {\|Z_x - Z_{x'}\|}^2_2\Big)} \\
    = \prod_{i=1}^{d} \bigg[{\big(1 + 2\gamma (\Sigma_{x,i} + \Sigma_{x',i})\big)}^{-1/2} \exp{\Big(-\gamma{\big(1 + 2\gamma(\Sigma_{x,i} + \Sigma_{x',i})\big)}^{-1} {\big(\mu_{x,i} - \mu_{x',i}\big)}^2 \Big)} \bigg].
\end{multline*}

Hence we see that a Gaussian kernel
\begin{equation*}
    k_{\mathcal{Y}}(y, y') = \exp{\big(-\gamma {\|y - y'\|}_2^2\big)}
\end{equation*}
with inverse length scale $\gamma > 0$ on the space of
targets $\mathcal{Y} = \mathbb{R}^d$
allows us to compute $\Exp_{Z_x \sim P_x} k_{\mathcal{Y}}(Z_x, y)$ and
$\Exp_{Z_x \sim P_x, Z_{x'} \sim P_{x'}} k_{\mathcal{Y}}(Z_x, Z_{x'})$
analytically. Moreover, the Gaussian kernel is universal and characteristic on $\mathbb{R}^d$~\citep{Micchelli2006,Fukumizu2008}. Hence, as discussed above,
by choosing an universal kernel $k_{\mathcal{P}}$ we can
guarantee that $\mathrm{KCE}_k = 0$ if and only if model $P$ is
calibrated.

On the space of normal distributions, the 2-Wasserstein distance with
respect to the Euclidean distance between
$P_x = \mathcal{N}(\mu_x, \Sigma_x)$ and
$P_{x'} = \mathcal{N}(\mu_{x'}, \Sigma_{x'})$ is given by
\begin{equation*}
    W^2_2\big(P_x, P_{x'}\big)
    = {\|\mu_x - \mu_{x'}\|}_2^2 + \mathrm{Tr}\bigg(\Sigma_x + \Sigma_{x'} - 2{\Big({\Sigma_{x'}}^{1/2} \Sigma_x {\Sigma_{x'}}^{1/2}\Big)}^{1/2}\bigg).
\end{equation*}
If $\Sigma_x \Sigma_{x'} = \Sigma_{x'} \Sigma_x$, such as for univariate normal distributions (\cref{app:ols,app:friedman}) or normal distributions with diagonal covariance matrices (\cref{app:efficiency}), this can be simplified to
\begin{equation*}
    W^2_2\big(P_x, P_{x'}\big)
    = {\big\|\mu_x - \mu_{x'}\big\|}_2^2 + {\Big\|\Sigma_{x}^{1/2} - \Sigma_{x'}^{1/2}\Big\|}_{\mathrm{Frob}}^2.
\end{equation*}
In this case the 2-Wasserstein distance is a Hilbertian metric on the space of normal distributions via the injective embedding
\begin{equation*}
    \phi\big(\mathcal{N}(\mu, \Sigma)\big) \coloneqq \big[\mu_1, \ldots, \mu_d, (\Sigma^{1/2})_{1,1}, (\Sigma^{1/2})_{1,2}, \ldots, (\Sigma^{1/2})_{d,d}\big]^\mathsf{T} \in \mathbb{R}^{d(d + 1)}.
\end{equation*}
Hence as discussed above, if the covariance matrices commute, the choice
\begin{equation*}
    k_{\mathcal{P}}\big(P_x, P_{x'}\big) = \exp{\big( - \lambda W_2^\nu(P_x, P_{x'})\big)}
\end{equation*}
yields a valid kernel for all $\lambda > 0$ and $\nu \in (0, 2]$, and is universal on $(\mathcal{P}, W_2)$ for all such parameter choices.

Thus for all $\lambda, \gamma > 0$ and $\nu \in (0, 2]$
\begin{equation*}
    k\big((p, y), (p', y')\big) = \exp{\big(-\lambda W^\nu_2(p, p')\big)} \exp{\Big(-\gamma {\|y-y'\|}^2_2\Big)}
\end{equation*}
is a valid and characteristic kernel on the product space $(\mathcal{P}, W_2) \times (\mathcal{Y}, \|.\|_2)$ where $\mathcal{Y} \subset \mathbb{R}^d$ and $\mathcal{P}$ is a space of normal distributions on $\mathbb{R}^d$ with commuting covariance matrices.
For these kernels $k$, $h\big((p, y), (p', y')\big)$ can be evaluated analytically and it is guaranteed that $\mathrm{KCE}_k = 0$ if and only if model $P$ is calibrated.

\subsection{Laplace distributions}
\label{app:laplace}

Assume that $\mathcal{Y} = \mathbb{R}$ and
$\mathcal{P} = \{\mathcal{L}(\mu, \beta) \colon \mu \in \mathbb{R}, \beta > 0\}$,
i.e., the model outputs Laplace distributions $P_X = \mathcal{L}(\mu_X, \beta_X)$ with
probability density function
\begin{equation*}
    p_X(y) = \frac{1}{2 \beta_X} \exp{\big(- \beta_X^{-1} |y - \mu_X| \big)}
\end{equation*}
for $y \in \mathcal{Y} = \mathbb{R}$. The distribution of these outputs is
defined by the distribution of their mean $\mu_X$ and scale parameter $\beta_X$.

Let $P_x = \mathcal{L}(\mu_x, \beta_x) \in \mathcal{P}$, $y \in \mathcal{Y} = \mathbb{R}$,
and $\gamma > 0$. If $\beta_x \neq \gamma^{-1}$, we have
\begin{multline*}
    \Exp_{Z_x \sim P_x} \exp{\big(- \gamma |Z_x - y|\big)}
    \\
    = {\big(\beta_x^2 \gamma^2 - 1\big)}^{-1} \Big(\beta_x \gamma
    \exp{\big(-\beta_x^{-1} |\mu_x - y| \big)} - \exp{\big(-\gamma |\mu_x-y| \big)}\Big).
\end{multline*}
Additionally, if $\beta_x = \gamma^{-1}$, the dominated convergence theorem implies
\begin{equation*}
    \begin{split}
        &\Exp_{Z_x \sim P_x} \exp{\big(-\gamma |Z_x - y| \big)} \\
        &\qquad = \lim_{\gamma \to \beta_x^{-1}} {\big(\beta_x^2 \gamma^2 - 1\big)}^{-1} \Big(\beta_x \gamma
        \exp{\big(-\beta_x^{-1}|\mu_x-y|\big)} - \exp{\big(-\gamma|\mu_x - y|\big)}\Big) \\
        &\qquad = \frac{1}{2}\big(1 + \gamma |\mu_x - y|\big)\exp{\big(-\gamma |\mu_x-y|\big)}.
    \end{split}
\end{equation*}

Let $P_{x'} = \mathcal{L}(\mu_{x'}, \beta_{x'})$ be
another Laplace distribution. If $\beta_x \neq \gamma^{-1}$,
$\beta_{x'} \neq \gamma^{-1}$, and $\beta_x \neq \beta_{x'}$,
we obtain
\begin{equation*}
    \begin{split}
        \Exp_{Z_x \sim P_x, Z_{x'} \sim P_{x'}} \exp{\big(-\gamma |Z_x-Z_{x'}|\big)} ={}& \frac{\gamma \beta_x^3 }{(\beta_x^2 \gamma^2 - 1)(\beta_x^2-{\beta_{x'}}^2)} \exp{\big(-\beta_x^{-1} |\mu_x-\mu_{x'}|\big)} \\
        &+ \frac{\gamma \beta_{x'}^3}{(\beta_{x'}^2 \gamma^2 - 1)({\beta_{x'}}^2-\beta_x^2)} \exp{\big(-{\beta_{x'}}^{-1} |\mu_x-\mu_{x'}|\big)} \\
        &+ \frac{1}{(\beta_x^2 \gamma^2 - 1)({\beta_{x'}}^2 \gamma^2 - 1)} \exp{\big(-\gamma |\mu_x-\mu_{x'}| \big)}.
    \end{split}
\end{equation*}
As above, all other possible cases can be deduced by applying
the dominated convergence theorem. More concretely,
\begin{itemize}
    \item if $\beta_x = \beta_{x'} = \gamma^{-1}$, then
    \begin{multline*}
        \Exp_{Z_x \sim P_x, Z_{x'} \sim P_{x'}} \exp{\big(-\gamma |Z_x-Z_{x'}|\big)} \\
        = \frac{1}{8}\Big(3 + 3 \gamma |\mu_x-\mu_{x'}|+ \gamma^2 {|\mu_x - \mu_{x'}|}^2\Big) \exp{\big(-\gamma |\mu_x-\mu_{x'}|\big)},
    \end{multline*}
    \item if $\beta_x = \beta_{x'}$ and $\beta_x \neq \gamma^{-1}$, then
    \begin{multline*}
        \Exp_{Z_x \sim P_x, Z_{x'} \sim P_{x'}} \exp{\big(-\gamma |Z_x-Z_{x'}|\big)}
        = \frac{1}{{(\beta_x^2 \gamma^2 - 1)}^2} \exp{\big(-\gamma |\mu_x-\mu_{x'}|\big)} \\
        + \bigg(\frac{\gamma \big(\beta_x + |\mu_x - \mu_{x'}|\big)}{2(\beta_x^2 \gamma^2 - 1)} - \frac{\beta_x \gamma}{{(\beta_x^2 \gamma^2 - 1)}^2} \bigg) \exp{\Big(-\beta_x^{-1} |\mu_x-\mu_{x'}|\Big)},
    \end{multline*}
    \item if $\beta_x \neq \beta_{x'}$ and $\beta_x = \gamma^{-1}$, then
    \begin{multline*}
        \Exp_{Z_x \sim P_x, Z_{x'} \sim P_{x'}} \exp{\big(-\gamma |Z_x-Z_{x'}| \big)}
        = \frac{{\beta_{x'}}^3 \gamma^3}{{({\beta_{x'}}^2 \gamma^2 - 1)}^2} \exp{\Big(- {\beta_{x'}}^{-1} |\mu_x-\mu_{x'}| \Big)} \\
        - \bigg(\frac{1 + \gamma |\mu_x - \mu_{x'}|}{2({\beta_{x'}}^2 \gamma^2 - 1)} + \frac{{\beta_{x'}}^2\gamma^2}{{({\beta_{x'}}^2 \gamma^2 - 1)}^2} \bigg) \exp{\big(-\gamma |\mu_x-\mu_{x'}| \big)},
    \end{multline*}
    \item and if $\beta_x \neq \beta_{x'}$ and $\beta_{x'} = \gamma^{-1}$, then
    \begin{multline*}
        \Exp_{Z_x \sim P_x, Z_{x'} \sim P_{x'}} \exp{\Big(-\gamma |Z_x - Z_{x'}| \Big)}
        = \frac{\beta_x^3 \gamma^3}{{(\beta_x^2 \gamma^2 - 1)}^2} \exp{\Big(-\beta_x^{-1} |\mu_x-\mu_{x'}| \Big)} \\
        - \bigg(\frac{1 + \gamma|\mu_x - \mu_{x'}|}{2(\beta_x^2 \gamma^2 - 1)} + \frac{\beta_x^2 \gamma^2}{{(\beta_x^2 \gamma^2 - 1)}^2} \bigg) \exp{\big(-\gamma |\mu_x-\mu_{x'}| \big)}.
    \end{multline*}
\end{itemize}

The calculations above show that by choosing a Laplacian
kernel
\begin{equation*}
    k_{\mathcal{Y}}\big(y, y'\big) = \exp{\big(-\gamma |y-y'|\big)}
\end{equation*}
with inverse length scale $\gamma > 0$ on the space of targets
$\mathcal{Y} = \mathbb{R}$, we can compute
$\Exp_{Z_x \sim P_x} k_{\mathcal{Y}}(Z_x, y)$ and
$\Exp_{Z_x \sim P_x, Z_{x'} \sim P_{x'}} k_{\mathcal{Y}}(Z_x, Z_{x'})$
analytically. Additionally, the Laplacian kernel is characteristic
on $\mathbb{R}$~\citep{Fukumizu2008}.

Since the Laplace distribution is an elliptically contoured
distribution, we know from \citet[Corollary~2]{Gelbrich1990} that
the 2-Wasserstein distance with respect to the Euclidean distance between
$P_x = \mathcal{L}(\mu_x, \beta_x)$ and
$P_{x'} = \mathcal{L}(\mu_{x'}, \beta_{x'})$ can be
computed in closed form and is given by
\begin{equation*}
    W^2_2\big(P_x, P_{x'}\big) = {(\mu_x - \mu_{x'})}^2 + 2{(\beta_x - \beta_{x'})}^2.
\end{equation*}
Thus we see that the 2-Wasserstein distance is a Hilbertian metric on the space of Laplace distributions via the embedding
\begin{equation*}
    \phi\big(\mathcal{L}(\mu, \beta)\big) \coloneqq [\mu, \sqrt{2}\beta]^\mathsf{T}.
\end{equation*}
Hence, as discussed above,
\begin{equation*}
    k_{\mathcal{P}}\big(P_x, P_{x'}\big) = \exp{\big(- \lambda W_2^\nu(P_x, P_{x'})\big)}
\end{equation*}
is a valid and universal kernel on $(\mathcal{P}, W_2)$ for $0 < \nu \leq 2$ and all $\lambda > 0$.

Therefore for all $\lambda, \gamma > 0$ and $\nu \in (0, 2]$
\begin{equation*}
    k\big((p, y), (p', y')\big) = \exp{\big(-\lambda W^\nu_2(p, p')\big)} \exp{\big(-\gamma |y-y'|\big)}
\end{equation*}
is a valid and characteristic kernel on the product space $(\mathcal{P}, W_2) \times (\mathcal{Y}, |.|)$ where $\mathcal{Y} \subset \mathbb{R}$ and $\mathcal{P}$ is a space of Laplace distributions.
For these kernels $k$, $h\big((p, y), (p', y')\big)$ can be evaluated analytically and guarantees that $\mathrm{KCE}_k = 0$ if and only if model $P$ is calibrated.

\subsection{Predicting mixtures of distributions}
\label{app:mixture}

Assume that the model predicts mixture distributions, possibly with different
numbers of components. A special case of this setting are ensembles of models,
in which each ensemble member predicts a component of the mixture model.

Let $p, p' \in \mathcal{P}$ with $p = \sum_i \pi_i p_i$ and
$p' = \sum_j \pi'_j p'_j$, where $\pi, \pi'$ are histograms and $p_i, p'_j$
are the mixture components. For kernel $k_{\mathcal{Y}}$ and $y \in \mathcal{Y}$ we obtain
\begin{equation*}
    \Exp_{Z \sim p} k_{\mathcal{Y}}(Z, y) = \sum_i \pi_i \Exp_{W \sim p_i} k_{\mathcal{Y}}(Z, y)
\end{equation*}
and
\begin{equation*}
    \Exp_{Z \sim p, Z' \sim p'} k_{\mathcal{Y}}(Z, Z') = \sum_{i,j} \pi_i \pi'_j \Exp_{Z \sim p_i, Z' \sim p'_j} k_{\mathcal{Y}}(Z, Z').
\end{equation*}
Of course, for these derivations to be meaningful, we require that they do not depend on the
choice of histograms $\pi, \pi'$ and mixture components $p_i, p'_j$.

\begin{definition}[{{see~\citet{Yakowitz1968}}}]\label{def:identifiable}
    A family $\mathcal{P}$ of finite mixture models is called identifiable if two mixtures
    $p = \sum_{i=1}^{K} \pi_i p_i \in \mathcal{P}$ and $p' = \sum_{j=1}^{K'} \pi'_j p'_j \in \mathcal{P}$,
    written such that all $p_i$ and all $p'_j$ are pairwise distinct, are equal if and only if $K = K'$ and
    the indices can be reordered such that for all $k \in \{1, \ldots, K\}$ there exists
    some $k' \in \{1, \ldots, K\}$ with $\pi_k = \pi'_{k'}$ and $p_k = p'_{k'}$.
\end{definition}

Clearly, if $\mathcal{P}$ is identifiable, then the derivations above do not depend on the choice
of histograms and mixture components. Prominent examples of identifiable mixture models are
Gaussian mixture models and mixture models of families of products of exponential
distributions~\citep{Yakowitz1968}.

Moreover, similar to optimal transport for Gaussian mixture models
by \citet{Delon2019,Chen2019,Chen2020}, we can consider metrics of the form
\begin{equation*}
    \inf_{w \in \Pi(\pi, \pi')} {\bigg(\sum_{i,j} w_{i,j} c^s(p_i, p'_j)\bigg)}^{1/s},
\end{equation*}
where
\begin{equation*}
    \Pi(\pi, \pi') = \bigg\{ w \colon \sum_i w_{i,j} = \pi'_j \land \sum_j w_{i,j} = \pi_i \land \forall i, j \colon w_{i,j} \geq 0\bigg\}
\end{equation*}
are the couplings of $\pi$ and $\pi'$, and $c(\cdot, \cdot)$ is a cost function
between the components of the mixture model.

\begin{theorem}
    Let $\mathcal{P}$ be a family of finite mixture models that is
    identifiable in the sense of \cref{def:identifiable}, and let $s \in [1, \infty)$.

    If $d(\cdot, \cdot)$ is a (Hilbertian) metric on the space of mixture components, then
    the Mixture Wasserstein distance of order $s$ defined by
    \begin{equation}\label{eq:mixture_distance}
        \mathrm{MW}_s(p, p') \coloneqq \inf_{w \in \Pi(\pi, \pi')} {\bigg(\sum_{i,j} w_{i,j} d^s(p_i, p'_j)\bigg)}^{1/s},
    \end{equation}
    is a (Hilbertian) metric on $\mathcal{P}$.
\end{theorem}

\begin{proof}
    First of all, note that for all $p, p' \in \mathcal{P}$ an optimal
    coupling $\hat{w}$ exists~\citep[Theorem~4.1]{Villani2009}. Moreover,
    $\sum_{i,j} \hat{w}_{i,j} d^s(p_i, p'_j) \geq 0$, and hence $\mathrm{MW}_s(p, p')$ exists.
    Moreover, since $\mathcal{P}$ is identifiable, we see that $\mathrm{MW}_s(p, p')$ does not
    depend on the choice of histograms and mixture components. Thus $\mathrm{MW}_s$ is well-defined.

    Clearly, for all $p, p' \in \mathcal{P}$ we have $\mathrm{MW}_s(p, p') \geq 0$ and
    $\mathrm{MW}_s(p, p') = \mathrm{MW}_s(p', p)$. Moreover,
    \begin{equation*}
        \begin{split}
            \mathrm{MW}_s^s(p, p) &= \min_{w \in \Pi(\pi, \pi)} \sum_{i,j} w_{i,j} d^s(p_i, p_j)
            \leq \sum_{i,j} \pi_i \delta_{i,j} d^s(p_i, p_j) \\
            &= \sum_i \pi_i d^s(p_i, p_i) = \sum_i \pi_i 0^2 = 0,
        \end{split}
    \end{equation*}
    and hence $\mathrm{MW}_s(p, p) = 0$. On the other hand, let $p, p' \in \mathcal{P}$ with
    optimal coupling $\hat{w}$ with respect to $\pi$ and $\pi'$, and assume that
    $\mathrm{MW}_s(p, p') = 0$. We have
    \begin{equation*}
        p = \sum_i \pi_i p_i = \sum_{i,j} \hat{w}_{i,j} p_i = \sum_{i,j \colon \hat{w}_{i,j} > 0} \hat{w}_{i,j} p_i.
    \end{equation*}
    Since $\mathrm{MW}_s(p, p') = 0$, we have $\hat{w}_{i,j} d^s(p_i, p'_j) = 0$ for all $i,j$, 
    and hence $d^s(p_i, p'_j) = 0$ if $\hat{w}_{i,j} > 0$. Since $d$ is a metric,
    this implies $p_i = p'_j$ if $\hat{w}_{i,j} > 0$. Thus we get
    \begin{equation*}
        p = \sum_{i,j \colon \hat{w}_{i,j} > 0} \hat{w}_{i,j} p_i = \sum_{i,j \colon \hat{w}_{i,j} > 0} \hat{w}_{i,j} p'_j
        = \sum_{i,j} \hat{w}_{i,j} p'_j = \sum_j \pi'_j p'_j = p'.
    \end{equation*}

    Function $\mathrm{MW}_s$ also satisfies the triangle inequality, following
    a similar argument as \citet{Chen2019}. Let
    $p^{(1)}, p^{(2)}, p^{(3)} \in \mathcal{P}$ and denote the optimal coupling with respect
    to $\pi^{(1)}$ and $\pi^{(2)}$ by $\hat{w}^{(12)}$, and the optimal coupling with respect to
    $\pi^{(2)}$ and $\pi^{(3)}$ by $\hat{w}^{(23)}$. Define $w^{(13)}$ by
    \begin{equation*}
        w^{(13)}_{i,k} \coloneqq \sum_{j \colon \pi^{(2)}_j \neq 0} \frac{\hat{w}^{(12)}_{i,j} \hat{w}^{(23)}_{j,k}}{\pi^{(2)}_j}.
    \end{equation*}
    Clearly $w^{(13)}_{i,k} \geq 0$ for all $i,k$, and we see that
    \begin{equation*}
        \begin{split}
            \sum_i w^{(13)}_{i,k} &= \sum_i \sum_{j \colon \pi^{(2)}_j \neq 0} \frac{\hat{w}^{(12)}_{i,j} \hat{w}^{(23)}_{j,k}}{\pi^{(2)}_j}
            = \sum_{j \colon \pi^{(2)}_j \neq 0}  \sum_i \frac{\hat{w}^{(12)}_{i,j} \hat{w}^{(23)}_{j,k}}{\pi^{(2)}_j} \\
            &= \sum_{j \colon \pi^{(2)}_j \neq 0} \frac{\pi^{(2)}_j \hat{w}^{(23)}_{j,k}}{\pi^{(2)}_j}
            = \sum_{j \colon \pi^{(2)}_j \neq 0} \hat{w}^{(23)}_{j,k} = \pi^{(3)} - \sum_{j \colon \pi^{(2)}_j = 0} \hat{w}^{(23)}_{j,k}
        \end{split}
    \end{equation*}
    for all $k$. Since for all $j,k$, $\pi_j^{(2)} \geq \hat{w}^{(23)}_{j,k}$,
    we know that $\pi^{(2)}_j = 0$ implies $\hat{w}^{(23)}_{j,k} = 0$ for all $k$.
    Thus for all $k$
    \begin{equation*}
        \sum_i w^{(13)}_{i,k} = \pi^{(3)}.
    \end{equation*}
    Similarly we obtain for all $i$
    \begin{equation*}
        \sum_k w^{(13)}_{i,k} = \pi^{(1)}.
    \end{equation*}
    Thus $w^{(13)} \in \Pi(\pi^{(1)}, \pi^{(3)})$, and therefore by exploiting
    the triangle inequality for metric $d$ and the Minkowski inequality we get
    \begin{equation*}
        \begin{split}
        \mathrm{MW}_s\big(p^{(1)}, p^{(3)}\big) &\leq {\bigg(\sum_{i,k} w^{(13)}_{i,k} d^s\big(p^{(1)}_i, p^{(3)}_k\big)\bigg)}^{1/s}
        = {\bigg(\sum_{i,k} \sum_{j \colon \pi^{(2)}_j \neq 0} \frac{\hat{w}^{(12)}_{i,j} \hat{w}^{(23)}_{j,k}}{\pi^{(2)}_j} d^s\big(p^{(1)}_i, p^{(3)}_k\big)\bigg)}^{1/s} \\
        &\leq {\bigg(\sum_{i,k} \sum_{j \colon \pi^{(2)}_j \neq 0} \frac{\hat{w}^{(12)}_{i,j} \hat{w}^{(23)}_{j,k}}{\pi^{(2)}_j} \big(d(p^{(1)}_i, p^{(2)}_j) + d(p^{(2)}_j, p^{(3)}_k)\big)^s\bigg)}^{1/s} \\
        &\leq {\bigg(\sum_{i,k} \sum_{j \colon \pi^{(2)}_j \neq 0} \frac{\hat{w}^{(12)}_{i,j} \hat{w}^{(23)}_{j,k}}{\pi^{(2)}_j} d^s\Big(p^{(1)}_i, p^{(2)}_j\Big)\bigg)}^{1/s} \\
        &\qquad + {\bigg(\sum_{i,k} \sum_{j \colon \pi^{(2)}_j \neq 0} \frac{\hat{w}^{(12)}_{i,j} \hat{w}^{(23)}_{j,k}}{\pi^{(2)}_j} d^s\Big(p^{(2)}_j, p^{(3)}_k\Big)\bigg)}^{1/s} \\
        &= {\bigg(\sum_{i} \sum_{j \colon \pi^{(2)}_j \neq 0} \hat{w}^{(12)}_{i,j} d^s\Big(p^{(1)}_i, p^{(2)}_j\Big)\bigg)}^{1/s} \\
        &\qquad + {\bigg(\sum_{k} \sum_{j \colon \pi^{(2)}_j \neq 0} \hat{w}^{(23)}_{i,k} d^s\Big(p^{(2)}_j, p^{(3)}_k\Big)\bigg)}^{1/s} \\
        &\leq {\bigg(\sum_{i,j} \hat{w}^{(12)}_{i,j} d^s\Big(p^{(1)}_i, p^{(2)}_j\Big)\bigg)}^{1/s} + {\bigg(\sum_{j,k} \hat{w}^{(23)}_{i,k} d^s\Big(p^{(2)}_j, p^{(3)}_k\Big)\bigg)}^{1/s} \\
        &= \mathrm{MW}_s\big(p^{(1)}, p^{(2)}\big) + \mathrm{MW}_s\big(p^{(2)}, p^{(3)}\big).
        \end{split}
    \end{equation*}
    
    Thus $\mathrm{MW}_s$ is a metric, and it is just left to show that it is Hilbertian
    if $d$ is Hilbertian.
    Since $d$ is a Hilbertian metric, there exists a Hilbert
    space $\mathcal{H}$ and a mapping $\phi$ such that
    \begin{equation*}
        d(x, y) = \|\phi(x) - \phi(y)\|_{\mathcal{H}}.
    \end{equation*}
    Let $r_1, \ldots, r_n \in \mathbb{R}$ with $\sum_i r_i = 0$
    and $p^{(1)}, \ldots, p^{(n)} \in \mathcal{P}$. Denote
    the optimal coupling with respect to $\pi^{(i)}$ and
    $\pi^{(j)}$ by $\hat{w}^{(i,j)}$. Then we have
    \begin{equation}\label{eq:hilbert_zero_1}
        \begin{split}
            \sum_{i,j} r_i r_j \sum_{k,l} \hat{w}^{(i,j)}_{k,l} \|\phi(p^{(i)}_k)\|_\mathcal{H}^2
            &= \sum_{i,k} r_i \|\phi(p^{(i)}_k)\|_\mathcal{H}^2 \sum_j r_j \sum_l \hat{w}^{(i,j)}_{k,l} \\
            &= \sum_{i,k} r_i \|\phi(p^{(i)}_k)\|_\mathcal{H}^2 \sum_j r_j \pi^{(i)}_k \\
            &= \sum_{i,k} r_i \pi^{(i)}_k \|\phi(p^{(i)}_k)\|_\mathcal{H}^2 \sum_j r_j
            = 0,
        \end{split}
    \end{equation}
    and similarly
    \begin{equation}\label{eq:hilbert_zero_2}
        \sum_{i,j} r_i r_j \sum_{k,l} \hat{w}^{(i,j)}_{k,l} \|\phi(p^{(j)}_l)\|_\mathcal{H}^2 = 0.
    \end{equation}
    Moreover, for all $k,l$ we get
    \begin{equation*}
        \begin{split}
            \sum_{i,j} r_i r_j \hat{w}^{(i,j)}_{k,l} \Big\langle \phi\Big(p^{(i)}_k\Big), \phi\Big(p^{(j)}_l\Big) \Big\rangle_\mathcal{H}
            &= \Big\langle \sum_i r_i \sqrt{\hat{w}^{(i,j)}_{k,l}} \phi\Big(p^{(i)}_k\Big), \sum_j r_j \sqrt{\hat{w}^{(i,j)}_{k,l}}  \phi\Big(p^{(j)}_l\Big)\Big\rangle_\mathcal{H} \\
            &= \Big\| \sum_i r_i \sqrt{\hat{w}^{(i,j)}_{k,l}} \phi\Big(p^{(i)}_k\Big) \Big\|_{\mathcal{H}}^2 \geq 0,
        \end{split}
    \end{equation*}
    and hence
    \begin{equation}\label{eq:hilbert_nonneg_1}
        \sum_{i,j} r_i r_j \sum_{k,l} \hat{w}^{(i,j)}_{k,l} \Big\langle \phi(p^{(i)}_k), \phi(p^{(j)}_l) \Big\rangle_\mathcal{H} \geq 0,
    \end{equation}
    and similarly
    \begin{equation}\label{eq:hilbert_nonneg_2}
        \sum_{i,j} r_i r_j \sum_{k,l} \hat{w}^{(i,j)}_{k,l} \Big\langle \phi(p^{(j)}_l), \phi(p^{(i)}_k) \Big\rangle_\mathcal{H} \geq 0.
    \end{equation}
    Hence from \cref{eq:hilbert_zero_1,eq:hilbert_zero_2,eq:hilbert_nonneg_1,eq:hilbert_nonneg_2} we get
    \begin{equation*}
        \begin{split}
            \sum_{i,j} r_i r_j \mathrm{MW}_s^s(p^{(i)}, p^{(j)}) &= \sum_{i,j} r_i r_j \sum_{k,l} \hat{w}^{(i,j)}_{k,l} d^s\Big(p^{(i)}_k, p^{(j)}_l\Big) \\
            &= \sum_{i,j} r_i r_j \sum_{k,l} \hat{w}^{(i,j)}_{k,l} \Big\|\phi\Big(p^{(i)}_k\Big) - \phi\Big(p^{(j)}_l\Big)\Big\|_{\mathcal{H}}^2 \\
            &=  \sum_{i,j} r_i r_j \sum_{k,l} \hat{w}^{(i,j)}_{k,l} \Big\|\phi\Big(p^{(i)}_k\Big)\Big\|_\mathcal{H}^2 \\
            &\qquad -  \sum_{i,j} r_i r_j \sum_{k,l} \hat{w}^{(i,j)}_{k,l} \Big\langle \phi\Big(p^{(i)}_k\Big), \phi\Big(p^{(j)}_l\Big) \Big\rangle_\mathcal{H} \\
            &\qquad - \sum_{i,j} r_i r_j \sum_{k,l} \hat{w}^{(i,j)}_{k,l} \Big\langle \phi\Big(p^{(j)}_l\Big), \phi\Big(p^{(i)}_k\Big) \Big\rangle_\mathcal{H} \\
            &\qquad + \sum_{i,j} r_i r_j \sum_{k,l} \hat{w}^{(i,j)}_{k,l} \Big\|\phi\Big(p^{(j)}_l\Big)\Big\|_\mathcal{H}^2 \\
            &\leq 0,
        \end{split}
    \end{equation*}
    which shows that $\mathrm{MW}_s^s$ is a negative definite
    kernel~\citep[Definition~3.1.1]{Berg1984}. Since $0 < 1/s < \infty$,
    $\mathrm{MW}_s$ is a negative definite kernel as well~\citep[Corollary~3.2.10]{Berg1984},
    which implies that metric $\mathrm{MW}_s$ is Hilbertian~\citep[Proposition~3.3.2]{Berg1984}.
\end{proof}

Hence we can lift a Hilbertian metric for the mixture components to a Hilbertian
metric for the mixture models. For instance, if the mixture components are normal
distributions, then the 2-Wasserstein distance with respect to the Euclidean distance is
a Hilbertian metric for the mixture components. When we lift it to the space $\mathcal{P}$
of Gaussian mixture models we obtain the $\mathrm{MW}_2$ metric proposed by
\citet{Delon2019,Chen2019,Chen2020}. As shown by \citet{Delon2019},
the discrete formulation of $\mathrm{MW}_2$ obtained by our construction is
equivalent to the definition
\begin{equation}\label{eq:mw2}
    \mathrm{MW}^2_2(p, p') \coloneqq \inf_{\gamma \in \Pi(p, p') \cap \mathrm{GMM}_{2n}(\infty)}
    \int_{\mathbb{R}^n \times \mathbb{R}^n} d^2(y, y') \,\mathrm{d}\gamma(y,y')
\end{equation}
for two Gaussian mixtures $p, p'$ on $\mathbb{R}^n$, where $\Pi(p, p')$ are the
couplings of $p$ and $p'$ (not of the histograms!) and $\mathrm{GMM}_{2n}(\infty) =
\cup_{k \geq 0} \mathrm{GMM}_{2n}(k)$ is the set of all finite Gaussian mixture distributions
on $\mathbb{R}^{2n}$. The construction of the discrete formulation as a solution to
a constrained optimization problem similar to \cref{eq:mw2} can be generalized to
mixtures of $t$-distributions. However, it is not possible for arbitrary mixture models
such as mixtures of generalized Gaussian distributions, even though they are
elliptically contoured distributions~\citep{Deledalle2018,Delon2019}.

The optimal coupling of the discrete histograms
can be computed efficiently using techniques from linear programming and optimal
transport theory such as the network simplex algorithm and the Sinkhorn algorithm.
As discussed above, if metric $d_{\mathcal{P}}$ is of the form in
\cref{eq:mixture_distance}, functions of the form
\begin{equation*}
    k_{\mathcal{P}}(p, p') = \exp{\big(- \lambda d^\nu_{\mathcal{P}}(p, p')\big)}
\end{equation*}
are valid kernels on $\mathcal{P}$ for all $\lambda > 0$ and $\nu \in (0, 2]$.

Thus taken together, if $k_{\mathcal{Y}}$ is a kernel on the
target space $\mathcal{Y}$ and $d(\cdot, \cdot)$ is a Hilbertian
metric on the space of mixture components, then for all $s \in [1, \infty)$, $\lambda > 0$,
and $\nu \in (0, 2]$
\begin{equation*}
    k\big((p, y), (p', y')\big) = \exp{\big(-\lambda \mathrm{MW}_s^\nu(p, p')\big)} k_{\mathcal{Y}}(y, y')
\end{equation*}
is a valid kernel on the product space $\mathcal{P} \times \mathcal{Y}$
of mixture distributions and targets that allows to evaluate $h\big((p, y), (p', y')\big)$
analytically.
Moreover, if $k_\mathcal{Y}$ and $k_\mathcal{P}$ are universal kernels, then $k$ is characteristic and hence $\mathrm{KCE}_k = 0$ if and only if model $P$ is calibrated.

\section{Classification as a special case}
\label{app:classification}

We show that the calibration error introduced in
\cref{def:ce} is a generalization of the calibration error
for classification proposed by \citet{Widmann2019}.
Their formulation of the calibration error is based on
a weighted sum of class-wise discrepancies between the
left hand side and right hand side of \cref{def:calibration},
where the weights are output by a vector-valued function
of the predictions. Hence their framework can only be applied to
finite target spaces, i.e., if $|\mathcal{Y}| < \infty$.

Without loss of generality, we assume that
$\mathcal{Y} = \{1, \ldots, d\}$ for some
$d \in \mathbb{N} \setminus \{1\}$. In our notation, the
previously defined calibration error, denoted by
$\mathrm{CCE}$ (classification calibration error), with
respect to a function space
$\mathcal{G} \subset \{f \colon \mathcal{P} \to \mathbb{R}^d\}$
is given by
\begin{equation*}
    \mathrm{CCE}_{\mathcal{G}} \coloneqq \sup_{g \in \mathcal{G}} \bigg| \Exp_{P_X}\bigg(\sum_{y \in \mathcal{Y}} \big(\Prob(Y = y| P_X) - P_X(\{y\})\big) g_y\big(P_X\big) \bigg) \bigg|.
\end{equation*}
For the function class
\begin{equation*}
    \mathcal{F} \coloneqq \big\{f \colon \mathcal{P} \times \mathcal{Y} \to \mathbb{R}, (p, y) \mapsto g_y(p) \big| g \in \mathcal{G}\big\}
\end{equation*}
we get
\begin{equation*}
    \mathrm{CCE}_{\mathcal{G}}
    = \sup_{f \in \mathcal{F}} \big| \Exp_{P_X,Y} f(P_X, Y) - \Exp_{P_X,Z_X} f(P_X, Z_X)\big|
    = \mathrm{CE}_{\mathcal{F}}.
\end{equation*}
Similarly, for every function class
$\mathcal{F} \subset \{f \colon \mathcal{P} \times \mathcal{Y} \to \mathbb{R}\}$,
we can define the space
\begin{equation*}
    \mathcal{G} \coloneqq \Big\{g \colon \mathcal{P} \to \mathbb{R}^d,
    p \mapsto {\big(f(p, 1), \ldots, f(p, d)\big)^\mathsf{T}} \Big| f \in \mathcal{F} \Big\},
\end{equation*}
for which
\begin{equation*}
    \mathrm{CE}_{\mathcal{F}}
    = \sup_{g \in \mathcal{G}} \bigg| \Exp_{P_X}\bigg(\sum_{y \in \mathcal{Y}} \big(\Prob(Y = y| P_X) - P_X(\{y\})\big) g_y\big(P_X\big) \bigg) \bigg|
    = \mathrm{CCE}_{\mathcal{G}}.
\end{equation*}
Thus both definitions are equivalent for classification models but
the structure of the employed function classes differs. The
definition of $\mathrm{CCE}$ is based on vector-valued functions on
the probability simplex whereas the formulation presented in this paper uses
real-valued function on the product space of the probability simplex and
the targets.

An interesting theoretical aspect of this difference is that
in the case of $\mathrm{KCE}$ we consider real-valued kernels on
$\mathcal{P} \times \mathcal{Y}$ instead of
matrix-valued kernels on $\mathcal{P}$, as shown by the following
comparison. By $e_i \in \mathbb{R}^d$ we denote the $i$th unit vector,
and for a prediction $p \in \mathcal{P}$ its representation
$v_p \in \mathbb{R}^d$ in the probability simplex is defined as
\begin{equation*}
    {(v_p)}_y = p\big(\{y\}\big)
\end{equation*}
for all targets $y \in \mathcal{Y}$.

Let $k \colon (\mathcal{P} \times \mathcal{Y}) \times (\mathcal{P} \times \mathcal{Y}) \to \mathbb{R}$.
We define the matrix-valued function
$K \colon \mathcal{P} \times \mathcal{P} \to \mathbb{R}^{d \times d}$
by
\begin{equation*}
    \big[K(p, p')\big]_{y,y'} = k\big((p, y), (p', y')\big)
\end{equation*}
for all $y,y' \in \mathcal{Y}$ and $p, p' \in \mathcal{P}$.
From the positive definiteness of kernel $k$ it follows that
$K$ is a matrix-valued kernel~\citep[Definition~2]{Micchelli2005}.
We obtain
\begin{equation*}
    \begin{split}
        \mathrm{SKCE}_{k} ={}& \Exp_{P_X,Y,P_{X'},Y'} \big[K(P_X, P_{X'})\big]_{Y,Y'}
        - 2 \Exp_{P_X,Y,P_{X'},Z_{X'}} \big[K(P_X, P_{X'})\big]_{Y,Z_{X'}} \\
        &+ \Exp_{P_X,Z_X,P_{X'},Z_{X'}} \big[K(P_X, P_{X'})\big]_{Z_X,Z_{X'}} \\
        ={}& \Exp_{P_X,Y,P_{X'},Y'} e_Y^\mathsf{T} K(P_X, P_{X'}) e_{Y'}
        - 2 \Exp_{P_X,Y,P_{X'},Y'} e_Y^\mathsf{T} K(P_X, P_{X'}) v_{P_{X'}} \\
        & + \Exp_{P_X,Y,P_{X'},Y'} v_{P_X}^\mathsf{T} K(P_X, P_{X'}) v_{P_{X'}} \\
        ={}& \Exp_{P_X,Y,P_{X'},Y'}{(e_Y - v_{P_X})}^\mathsf{T} K(P_X, P_{X'}) (e_{Y'} - v_{P_{X'}}),
    \end{split}
\end{equation*}
which is exactly the result by \citet{Widmann2019}
for matrix-valued kernels.

As a concrete example, \citet{Widmann2019} used a
matrix-valued kernel of the form
$(p, p') \mapsto \exp{(- \gamma \|p - p'\|)} \mathbf{I}_d$
in their experiments. In our formulation this corresponds to the
real-valued tensor product kernel
$\big((p, y), (p', y')\big) \mapsto \exp{(- \gamma \|p - p'\|)} \delta_{y,y'}$.

\section{Temperature scaling}
\label{app:temperature}

Since many modern neural network models for classification have been
demonstrated to be uncalibrated~\citep{Guo2017}, it is
of high practical interest being able to improve calibration of
predictive models.
Generally, one distinguishes between calibration techniques that are
applied during training and post-hoc calibration methods that try to
calibrate an existing model after training.

Temperature scaling~\citep{Guo2017} is a simple calibration method 
for classification models with only one scalar parameter. Due to its
simplicity it can trade off calibration of different classes~\citep{Kull2019}, but
conveniently it does not change the most-confident prediction and
hence does not affect the accuracy of classification models with
respect to the 0-1 loss.

In regression, common post-hoc calibration methods are
based on quantile binning and hence insufficient for our framework.
\citet{Song2019} proposed a calibration method for regression models
with real-valued targets, based on a special case of \cref{def:calibration}.
This calibration method was shown to perform well empirically but is
computationally expensive and requires users to choose hyperparameters for
a Gaussian process model and its variational inference.
As a simpler alternative, we
generalize temperature scaling to arbitrary predictive models in the
following way.

\begin{definition}\label{def:temperature}
Let $P_x$ be the output of a probabilistic predictive model $P$ 
for feature $x$. If $P_x$ has probability density function $p_x$ with
respect to a reference measure $\mu$, then temperature scaling
with respect to $\mu$ with temperature $T > 0$ yields a new output
$Q_x$ whose probability density function $q_x$ with respect to $\mu$
satisfies
\begin{equation*}
    q_x \propto p_x^{1/T}.
\end{equation*}
\end{definition}

The notion for classification models given by \citet{Guo2017} can be
recovered by choosing the counting measure on the classes as reference
measure.

For some exponential families on $\mathbb{R}^d$ we obtain
particularly simple transformations with respect to the
Lebesgue measure $\lambda^d$ that keep the type of predicted
distribution and its mean invariant. Hence in contrast
to other calibration methods, for these models temperature
scaling yields analytically tractable distributions and does
not negatively impact the accuracy of the models with
respect to the mean squared error and the mean absolute
error.

For instance, temperature scaling of multivariate power exponential
distributions~\citep{Gomez1998} in $\mathbb{R}^d$,
of which multivariate normal distributions are a special case, with
respect to $\lambda^d$ corresponds to multiplication of their
scale parameter with $T^{1/\beta}$, where $\beta$ is the
so-called kurtosis parameter~\citep{GomezSanchezManzano2008}.
For normal distributions, this corresponds to multiplication
of the covariance matrix with $T$.

Similarly, temperature scaling of Beta and Dirichlet distributions
with respect to reference measure
\begin{equation*}
    \mu(\mathrm{d}x) \coloneqq x^{-1} (1-x)^{-1} \mathbbm{1}_{(0,1)}(x) \lambda^1(\mathrm{d}x)
\end{equation*}
and
\begin{equation*}
    \mu(\mathrm{d}x) \coloneqq \bigg(\prod_{i=1}^d x_i^{-1}\bigg) \mathbbm{1}_{(0,1)^d}(x) \lambda^d(\mathrm{d}x),
\end{equation*}
respectively, corresponds to division of the canonical parameters of these
distributions by $T$ without affecting the predicted mean value.

All in all, we see that temperature scaling for general
predictive models preserves some of the nice properties for classification
models. For some exponential families such as normal
distributions reference measure $\mu$ can be chosen such that temperature
scaling is a simple transformation of the parameters of the predicted
distributions (and hence leaves the considered model class invariant) that
does not affect accuracy of these models with respect to the mean squared error
and the mean absolute error.

\section{Expected calibration error for countably infinite discrete target spaces}
\label{app:ece_infinite}

In literature, $\mathrm{ECE}_d$ and $\mathrm{MCE}_d$ are defined for binary and multi-class
classification problems~\citep{Vaicenavicius2019,Guo2017,Naeini2015}. For common distance measures
on the probability simplex such as the total variation distance and the squared Euclidean distance,
$\mathrm{ECE}_d$ and $\mathrm{MCE}_d$ can be formulated as a calibration error in the framework of
\citet{Widmann2019}, which is a special case of the framework proposed in this paper for binary and
multi-class classification problems.

In contrast to previous approaches, our framework handles countably infinite discrete target spaces
as well. For every problem with countably infinitely many targets, such as, e.g., Poisson regression,
there exists an equivalent regression problem on the set of natural numbers. Hence without loss of generality
we assume $\mathcal{Y} = \mathbb{N}$. Denote the space of probability distributions on $\mathbb{N}$,
the infinite dimensional probability simplex, with $\Delta^\infty$. Clearly, $\Delta^\infty$ can be
viewed as a subspace of the sequence space $\ell^1$ that consists of all sequences $x = (x_n)_{n \in \mathbb{N}}$
with $x_n \geq 0$ for all $n \in \mathbb{N}$ and $\|x\|_{1} = 1$.

\begin{theorem}\label{thm:ece_infinite}
    Let $1 < p < \infty$ with Hölder conjugate $q$. If
    \begin{equation*}
        \mathcal{F} \coloneqq \{f \colon \Delta^{\infty} \times \mathbb{N} \to \mathbb{R} \mid
    \Exp_{P_X} \|\big(f(P_X, n)\big)_{n \in \mathbb{N}}\|_{p}^p \leq 1 \},
    \end{equation*}
    then
    \begin{equation*}
        \mathrm{CE}^q_{\mathcal{F}} = \Exp_{P_X} \|\Prob(Y|P_X) - P_X\|_{q}^q.
    \end{equation*}

    Let $\mu$ be the law of $P_X$. If $\mathcal{F} \coloneqq \{f \colon \Delta^{\infty} \times \mathbb{N} \to \mathbb{R} \mid \Exp_{P_X} \|(f(P_X,n))_{n \in \mathbb{N}}\|_1 \leq 1\}$,
    then
    \begin{equation*}
        \mathrm{CE}_{\mathcal{F}} = \mu\text{-}\esssup_{\xi \in \Delta^{\infty}} \sup_{y \in \mathbb{N}} |\Prob(Y=y|P_X=\xi) - \xi(\{y\})|.
    \end{equation*}
    Moreover, if
    $\mathcal{F} = \{f \colon \Delta^{\infty} \times \mathbb{N} \to \mathbb{R} \mid \mu\text{-}\esssup_{\xi \in \Delta^\infty} \sup_{y \in \mathbb{N}} |f(\xi, y)| \leq 1\}$,
    then
    \begin{equation*}
        \mathrm{CE}_{\mathcal{F}} = \Exp_{P_X} \|\Prob(Y|P_X) - P_X\|_1.
    \end{equation*}
\end{theorem}

\begin{proof}
    Let $1 \leq p \leq \infty$, and let $\mu$ be the law of $P_X$ and $\nu$ be the counting measure on $\mathbb{N}$.
    Since both $\mu$ and $\nu$ are $\sigma$-finite measures, the product measure $\mu \otimes \nu$ is uniquely
    determined and $\sigma$-finite as well. Using these definitions, we can reformulate $\mathcal{F}$ as
    \begin{equation*}
        \mathcal{F} =  \{f \in L^p(\Delta^\infty \times \mathbb{N}; \mu \otimes \nu) \mid \|f\|_{p; \mu \otimes \nu} \leq 1\}.
    \end{equation*}

    Define the function $\delta \colon \Delta^{\infty} \times \mathbb{N} \to \mathbb{R}$ $(\mu \otimes \nu)$-almost
    surely by
    \begin{equation*}
        \delta(\xi, y) \coloneqq \Prob(Y = y \,|\, P_X = \xi) - \xi(\{y\}).
    \end{equation*}
    Note that $\delta$ is well-defined since we assume that all singletons on $\Delta^{\infty}$ are
    $\mu$-measurable. Moreover, $\delta \in L^q(\Delta^\infty \times \mathbb{N}; \mu \otimes \nu)$,
    which follows from $(\xi, y) \mapsto \Prob(Y = y \,|\, P_X = \xi)$
    and $(\xi, y) \mapsto \xi(\{y\})$ being functions in $L^q(\Delta^\infty \times \mathbb{N}; \mu \otimes \nu)$.

    Since $\mu \otimes \nu$ is a $\sigma$-finite measure, the extremal equality of Hölder's inequality implies that
    \begin{equation*}
        \begin{split}
            \mathrm{CE}_{\mathcal{F}} &= \sup_{f \in \mathcal{F}} \Exp_{P_X,Y} f(P_X,Y) - \Exp_{P_X,Z_X} f(P_X, Z_X) \\
            &= \sup_{f \in \mathcal{F}} \bigg| \Exp_{P_X,Y} f(P_X,Y) - \Exp_{P_X,Z_X} f(P_X, Z_X) \bigg| \\
            &= \sup_{f \in \mathcal{F}} \bigg| \int_{\Delta^{\infty} \times \mathbb{N}} f(\xi, y) \delta(\xi, y) \,(\mu \otimes \nu)(\mathrm{d}(\xi, y)) \bigg| \\
            &= \|\delta\|_{q;\mu \otimes \nu}.
        \end{split}
    \end{equation*}
    Note that the second equality follows from the symmetry of the function spaces $\mathcal{F}$: for every $f \in \mathcal{F}$, also $-f \in \mathcal{F}$.

    Hence for $1 < p \leq \infty$, we obtain
    \begin{equation*}
        \begin{split}
            \mathrm{CE}_{\mathcal{F}}^q &= \int_{\Delta^\infty \times \mathbb{N}} |\delta(\xi, y)|^{q} \,(\mu \otimes \nu)(\mathrm{d}(\xi, y)) \\
            &= \Exp_{P_X} \|(\delta(P_X, y))_{y \in \mathbb{N}}\|_q^q = \Exp_{P_X} \|\Prob(Y |P_X) - P_X\|_q^q.
        \end{split}
    \end{equation*}
    For $p = 1$, we get
    \begin{equation*}
        \begin{split}
            \mathrm{CE}_{\mathcal{F}} = \mu\text{-}\esssup_{\xi \in \Delta^{\infty}} \sup_{y \in \mathbb{N}} |\delta(\xi, y)| = \mu\text{-}\esssup_{\xi \in \Delta^{\infty}} \sup_{y \in \mathbb{N}} |\Prob(Y=y|P_X=\xi) - \xi(\{y\})|,
        \end{split}
    \end{equation*}
    which concludes the proof.
\end{proof}

We see that our framework deals with countably infinite discrete target spaces seamlessly whereas the previously
proposed framework by \citet{Widmann2019} is not applicable to such spaces. It is mathematically pleasing to see
that for countably infinite discrete targets the calibration errors obtained in \cref{thm:ece_infinite} within
our framework coincide with the natural generalization of $\mathrm{ECE}_d$ and $\mathrm{MCE}_d$ given in
\cref{app:ece_mce}.

\end{document}